\let\oldvec\vec
\documentclass[preprint,11pt]{article}
\let\vec\oldvec

\usepackage{amsmath,amssymb}
\usepackage{amsthm}
\usepackage{mathtools}
\usepackage{url}
\usepackage{makeidx}
\usepackage{enumerate}
\usepackage{fullpage}
\usepackage{graphicx,float,psfrag,epsfig,caption}
\usepackage[usenames,dvipsnames,svgnames,table]{xcolor}
\definecolor{darkgreen}{rgb}{0.0,0,0.9}
\usepackage[pagebackref,letterpaper=true,colorlinks=true,pdfpagemode=none,citecolor=OliveGreen,linkcolor=BrickRed,urlcolor=BrickRed,pdfstartview=FitH]{hyperref}
\usepackage{color}
\usepackage{xr}
\usepackage{caption}
\usepackage{graphicx}
\usepackage[utf8]{inputenc}

\usepackage{nicefrac}       
\usepackage{microtype}      

\usepackage{enumitem}
\usepackage{caption}
\usepackage{subcaption}
\usepackage{wrapfig}
\usepackage{lipsum}

\usepackage{hyperref,graphicx,amsmath,amsfonts,amssymb,bm,url,breakurl,epsf,color,MnSymbol,mathbbol,fmtcount,semtrans,caption,subcaption,multirow,comment, boldline}
\usepackage[noend]{algpseudocode}
\usepackage{amssymb}

\usepackage[utf8]{inputenc} 
\usepackage[T1]{fontenc}    
\usepackage{url}            
\usepackage{booktabs}       
\usepackage{amsfonts}       
\usepackage{nicefrac}       
\usepackage{microtype}      

\usepackage{wrapfig}

\makeatletter
\providecommand*{\boxast}{%
  \mathbin{
    \mathpalette\@boxit{*}%
  }%
}
\newcommand*{\@boxit}[2]{%
  \sbox0{$\m@th#1\Box$}%
  \ifx#1\displaystyle \ht0=\dimexpr\ht0+.05ex\relax \fi
  \ifx#1\textstyle \ht0=\dimexpr\ht0+.05ex\relax \fi
  \ifx#1\scriptstyle \ht0=\dimexpr\ht0+.04ex\relax \fi
  \ifx#1\scriptscriptstyle \ht0=\dimexpr\ht0+.065ex\relax \fi
  \sbox2{$#1\vcenter{}$}
  \rlap{%
    \hbox to \wd0{%
      \hfill
      \raisebox{%
        \dimexpr.5\dimexpr\ht0+\dp0\relax-\ht2\relax
      }{$\m@th#1#2$}%
      \hfill
    }%
  }%
  \Box
}
\makeatother

  \makeatletter
\def\BState{\State\hskip-\ALG@thistlm}
\makeatother

  \usepackage{mathtools}

\usepackage{titlesec}

\usepackage{tikz}
\usepackage{pgfplots}
\usetikzlibrary{pgfplots.groupplots}

\titleformat{\paragraph}
{\normalfont\normalsize\bfseries}{\theparagraph}{1em}{}
\titlespacing*{\paragraph}
{0pt}{3.25ex plus 1ex minus .2ex}{1.5ex plus .2ex}

\usepackage{movie15}

\usepackage{caption}
\usepackage[bottom,hang,flushmargin]{footmisc} 

\newcommand{\tsn}[1]{{\left\vert\kern-0.25ex\left\vert\kern-0.25ex\left\vert #1 
    \right\vert\kern-0.25ex\right\vert\kern-0.25ex\right\vert}}

\definecolor{darkred}{RGB}{150,0,0}
\definecolor{darkgreen}{RGB}{0,150,0}
\definecolor{darkblue}{RGB}{0,0,200}
\hypersetup{colorlinks=true, linkcolor=darkred, citecolor=darkgreen, urlcolor=darkblue}

\newtheorem{theorem}{Theorem}[section]

\newtheorem{lemma}[theorem]{Lemma}
\newtheorem{corollary}[theorem]{Corollary}
\newtheorem{propo}[theorem]{Proposition}
\newtheorem{definition}[theorem]{Definition}

\newtheorem{remarks}[subsection]{Remarks}

\newcommand{\eps}{\varepsilon}

\newcommand{\beq}{\begin{equation}}

\newcommand{\eeq}{\end{equation}}

\newcommand{\nn}{\nonumber}

\newcommand{\z}{{\vct{z}}}

\newcommand{\tn}[1]{\|{#1}\|_{\ell_2}}

\newcommand{\w}{\vct{w}}

\newcommand{\opnorm}[1]{\left\|#1\right\|}

\newcommand{\onenorm}[1]{\left\|#1\right\|_{\ell_1}}
\newcommand{\twonorm}[1]{\left\|#1\right\|_{\ell_2}}

\newcommand{\abs}[1]{\left|#1\right|}

\newcommand{\x}{\vct{x}}

\newcommand{\y}{\vct{y}}

\definecolor{emmanuel}{RGB}{255,127,0}

\newcommand{\R}{\mathbb{R}}

\newcommand{\<}{\langle}
\renewcommand{\>}{\rangle}

\newcommand{\sgn}[1]{\textrm{sgn}(#1)}

\newcommand{\E}{\operatorname{\mathbb{E}}}

\newcommand{\vct}[1]{\bm{#1}}
\newcommand{\mtx}[1]{\bm{#1}}

\numberwithin{equation}{section} 

\def \endprf{\hfill {\vrule height6pt width6pt depth0pt}\medskip}

\newcommand{\hth}{{\widehat{\boldsymbol{\theta}}}}
\newcommand{\bth}{{\boldsymbol{\theta}}}

\newcommand\reals{\mathbb{R}}
\newcommand\bSigma{\boldsymbol{\Sigma}}
\newcommand\normal{{\sf N}}
\newcommand\bdelta{\boldsymbol{\delta}}
\newcommand\sign{{\rm sign}}
\newcommand\sT{{\sf T}}
\newcommand\I{\mtx{I}}
\newcommand\zero{\mtx{0}}

\newcommand\bv{\mtx{v}}

\newcommand\bz{\boldsymbol{z}}

\def\de{{\rm d}}

\def\by{\boldsymbol{y}}

\def\bX{\boldsymbol{X}}
\def\bx{\vct{x}}

\def\hz{\widehat{\vct{z}}}

\def\cS{\mathcal{S}}

\def\reals{\mathbb{R}}
\def\SR{{\sf SR}}
\def\AR{{\sf AR}}
\def\etest{\eps_{{\rm test}}}
\def\erf{{\rm erf}}
\def\erfc{{\rm erfc}}
\def\bq{\vct{q}}
\def\ST{{\sf ST}}

\title{Precise Tradeoffs in Adversarial Training for Linear Regression}
\author{Adel Javanmard\thanks{Data Science and Operations Department, Marshall School of Business, University of Southern California, Los Angeles, CA},  \quad Mahdi Soltanolkotabi,\thanks{Ming Hsieh Department of Electrical and Computer Engineering, University of Southern California, Los Angeles, CA}  \quad Hamed Hassani\thanks{Department of Electrical and Systems Engineering, University of Pennsylvania, Philadelphia, PA}}

\begin{document}

\maketitle

\begin{abstract}

Despite breakthrough performance, modern learning models are known to be highly vulnerable to small adversarial perturbations in their inputs. While a wide variety of recent \emph{adversarial training} methods have been effective at improving robustness to perturbed inputs (robust accuracy), often this benefit is accompanied by a decrease in accuracy on benign inputs (standard accuracy), leading to a tradeoff between
often competing objectives. Complicating matters further, recent empirical evidence suggest that a variety of other factors (size and quality of training data, model size, etc.) affect this tradeoff in somewhat surprising ways. In this paper we provide a precise and comprehensive understanding of the role of adversarial training in the context of linear regression with Gaussian features. In particular, we characterize the fundamental tradeoff between the accuracies achievable by any algorithm regardless of computational power or size of the training data. Furthermore, we precisely characterize the standard/robust accuracy and the corresponding tradeoff achieved by a contemporary mini-max adversarial training approach in a high-dimensional regime where the number of data points and the parameters of the model grow in proportion to each other. Our theory for adversarial training algorithms also facilitates the rigorous study of how a variety of factors (size and quality of training data, model overparametrization etc.) affect the tradeoff between these two competing accuracies. 

\end{abstract}

{\bf keywords. }%
  Tradeoffs in Adversarial Training, High-dimensional Statistics, Gaussian processes, Linear Regression

\section{Introduction}

Recent advances in machine learning and deep learning in particular, have led to trained models with breakthrough performance in a variety of applications spanning visual object classification to speech recognition and natural language processing. Despite wide empirical success, these modern learning models are known to be highly vulnerable to small adversarial perturbations to their inputs \cite{biggio2013evasion,szegedy2014intriguing}. For instance, in the context of image classification even small perturbations of the image, which are imperceptible to a human, can lead to incorrect classification by these models. As these modern inferential techniques begin to be deployed in applications such as autonomous or recognition systems in which safety, reliability, and security are crucial, it is increasingly important to ensure trained models are robust against abrupt or  adversarial perturbations to the input.

To mitigate the effect of adversarial perturbations, a wide variety of \emph{adversarial training} methods have been developed \cite{DBLP:journals/corr/GoodfellowSS14, kurakin2016adversarial, DBLP:conf/iclr/MadryMSTV18, DBLP:conf/iclr/RaghunathanSL18, DBLP:conf/icml/WongK18} which often involve augmenting the training loss so as to become more robust to input perturbations. While adversarial training methods have been rather successful at improving the accuracy of the trained model on adversarially perturbed inputs (\emph{robust accuracy}), often this benefit comes at the cost of decreasing accuracy on natural unperturbed inputs (\emph{standard accuracy}) \cite{DBLP:conf/iclr/MadryMSTV18}. Therefore, it is crucial to understand the tradeoff between robust and standard accuracy with adversarial training.  Complicating matters further, recent empirical evidence suggest that a variety of other factors affect this tradeoff in somewhat surprising ways. For instance, experiments in \cite{tsipras2018robustness} demonstrate that while adversarial training typically has a negative effect on standard accuracy, it outperforms non-adverserial training methods when there are only a few training samples. Perhaps surprisingly, the recent paper by \cite{raghunathan2019adversarial} suggests that in some cases the tradeoff between standard and robust accuracy can be mitigated with additional unlabeled data. Towards demystifying these empirical phenomena, in this paper we aim to precisely characterize the role of adversarial training by focusing on the following key questions:

\begin{quote}
\textit{What is the fundamental tradeoff between robust and standard accuracies in both finite and infinite data limits? How can we algorithmically achieve this tradeoff and what is the role of adversarial training? What is the effect of the size/quality of the data on this tradeoff? How does the model size (e.g. overparametrization) change this tradeoff?}
\end{quote}
A few recent papers have begun to answer some of these questions in specific settings \cite{tsipras2018robustness,DBLP:conf/icml/ZhangYJXGJ19, raghunathan2019adversarial}. See Section \ref{related} for a detailed discussion. Despite this interesting recent progress, a comprehensive understanding of the role of adversarial training and how it precisely affects the aforementioned tradeoffs remains largely mysterious. In this paper we aim to provide a precise characterization of the role of adversarial training by focusing on the simple yet foundational problem of linear regression. 
\medskip

\noindent{\textbf{Contributions.}} 
We formally introduce the linear regression problem with adversarially perturbed inputs in Section~\ref{secform} and address the questions above in this setting.
\begin{itemize}[leftmargin=*]
\item  We characterize the fundamental tradeoff between standard risk\footnote{Since we focus on a regression problem henceforth we focus on risk in lieu of accuracy.} ($\SR$) and adversarial risk ($\AR$) achievable by any algorithm regardless of the computational power and the size of the available training data (see Section \ref{sec31}). This is carried out by deriving the asymptotic expressions of standard and adversarial risks, and analysing the Pareto optimal points of a two dimensional region consisting of all the achievable $(\SR,\AR)$ pairs.  This analysis clearly demonstrates the existence of a non-trivial tradeoff between the two risks in linear regression as depicted in Figure~\ref{fig:curves}.  

\item In Section \ref{sec32}, we  turn our attention to modern adversarial training algorithms  and provide a precise characterizition of the standard and adversarial risks achieved by them. This is carried out in a high-dimensional regime where the size of the training data $n$ and the number of parameters $p$ grow proportional to each other with their ratio $n/p \to \delta$ for fixed $\delta \in(0,+\infty)$. A key ingredient of our analysis is a powerful extension of a classical Gaussian process inequality \cite{gordon1988milman} known as the Convex Gaussian Minimax Theorem developed in \cite{thrampoulidis2015regularized} and further extended in \cite{thrampoulidis2018precise,deng2019model}.

\item Our precise characterization of the standard and robust risks for adversarial training algorithms allows us to rigorously study a variety of phenomena.  First, we study the tradeoffs between standard and adversarial risks for a contemporary adversarial training algorithm and show that as the limiting ratio $n/p\to\delta$ between the number of training data $n$ and number of parameters $p$ grows, the algorithmic tradeoff curve approaches the fundamental (Pareto-optimal) tradeoff curve. These findings are manifested empirically in Figure~\ref{fig:curves}.   We also characterize the effect of the size of the training data and model overparametrization (see Section~\ref{sec33}). We argue analytically and empirically that  in the overparametrized regime (i.e. when $\delta < 1$) adversarial training helps improve standard risk (compared to normal training). However, as the size of training data grows (i.e. $\delta$ becomes large) adversarial training effectively hurts standard risk. In short, adversarial training improves generalization in the overparametrized regime, but effectively hurts generalization in the sufficiently underparametrized regime. Finally, in Section \ref{sec34} we demonstrate and prove the emergence of a phenomenon in adversarial training which is similar to the so-called double-descent phenomenon. When traditional training is used, the double-descent phenomena demonstrates that increasing the model complexity beyond a certain interpolation threshold always improves generalization. We show that the double-descent behavior continues to hold with adversarial training. However, for linear regression model considered in this paper, the global minimum of the risk is achieved under the interpolation threshold \emph{whose value changes with $\eps$}. Our theory also allows us to study how the adversarial training affects the interpolation threshold.
\end{itemize}

\vspace{-0.5cm}

\section{Problem formulation}
\label{secform}
In a typical supervised learning problem, we wish to fit a function $f_{\vct{\theta}}$, parameterized by $\vct{\theta}\in\R^p$ to a training data set of $n$ input-output pairs $\{(\vct{x}_i,y_i)\}_{i=1}^n$ drawn i.i.d.~from some common law $\mathcal{P}$. The fitting problem often consists of finding a parameter $\widehat{\vct{\theta}}$ that minimizes the empirical risk 
\begin{align}
\label{empobj}
\widehat{\vct{\theta}} \in \arg\min_{\bth \in \reals^p}\quad\frac{1}{n}\sum_{i=1}^n \ell\left(\vct{x}_i,y_i;\vct{\theta}\right):=\frac{1}{n}\sum_{i=1}^n \widetilde{\ell}\left(f_{\vct{\theta}}(\vct{x}_i),y_i\right), 
\end{align}
over the space of all parameters $\vct{\theta}$. The loss $\widetilde{\ell}(f_{\vct{\theta}}(\vct{x}),y)$ measures discrepancy between the output (or label) $y$ and the prediction $f_{\vct{\theta}}(\vct{x})$. The goal is of course to learn models that perform well on the yet unseen test data that is also generated from the same distribution $\mathcal{P}$. In particular, the empirical risk above serves as a surrogate for the population risk (loss) $\E_{(\vct{x},y)\sim\mathcal{P}}[\ell(\vct{x},y;\vct{\theta})]$. 
%
%

In practice, many models trained by following this paradigm are often highly vulnerable to adversarial perturbations with many well documented examples in deep learning. This observation has given rise to a surge of interest in both, finding such perturbations (a.k.a adversarial attacks) and also learning models that are robust against such perturbations (a.k.a. adversarial training). A line of recent work \cite{tsipras2018robustness,madry2017towards} propose training approaches that demonstrate promising empirical performance against adversarial perturbations. Motivated by applications in image processing, these papers consider an adversarial attack model where for a predefined perturbation set $\cS$, the adversary has the power of perturbing   
 each data point $\x$ by adding an element of $\cS$. Then an estimator $\hth^{\cS}$ is constructed by solving a saddle point problem that takes into account such manipulative power for the adversary:
 \begin{align}\label{eq:hth-cS}
\hth^{\cS} \in \arg\min_{\bth\in \reals^p}\text{ }\max_{\bdelta_i\in \cS}\quad \frac{1}{n}\sum_{i=1}^n  \ell\left(\vct{x}_i+\bdelta_i,y_i;\vct{\theta}\right)\,.
\end{align}  
To evaluate the performance of such an estimator, in this paper we consider two metrics of particular interest, \emph{standard risk} and \emph{adversarial risk}.
 \medskip
 
 \noindent{\bf Standard risk.} This is the expected prediction loss of an estimator $\hth$ on an uncorrupted test data point that is generated from the same distribution as the training data. Namely,
 \begin{align}\label{SR}
\SR(\hth):= \frac{1}{p} \E\left[\ell\left(\x,y;\hth\right)\right]\,\quad\text{where}\quad (\x,y)\sim \mathcal{P}\,.
 \end{align}
 \medskip
 
 \noindent{\bf Adversarial risk.}  This is the expected prediction loss of an estimator $\hth$ on an adversarially corrupted test data point according to the  attack model \eqref{eq:hth-cS}. Namely,
  \begin{align}\label{AR}
\AR(\hth):= \frac{1}{p} \E\Big[\max_{\bdelta\in \cS} \ell\left(\x+\bdelta,y;\hth\right) \Big]\,\quad\text{where}\quad (\x,y)\sim \mathcal{P}\,.
 \end{align}
Stated differently, the adversarial risk measures how well the estimator $\hth$ performs in predicting the true label when it is fed with an adversarially corrupted test data point. We note that the factor $1/p$ is the proper scaling so the risk has a finite limit under our asymptotic regime.
 
Focusing on linear regression, in this paper we aim to derive asymptotically exact characterizations of these two metrics and study the tradeoff achieved by the class of estimators $\hth^{\cS}$ of the form~\eqref{eq:hth-cS}. These characterizations will also enable us to study the effect of various quantities (e.g.~size and quality of the training data, model size, etc.) on the trade-off between statistical and adversarial risk. Specifically, we consider the linear regression model below.
\begin{definition}[Linear Regression Setting]\label{linregmod} We consider
standard Gaussian linear regression model with the training data consisting of $n$ i.i.d pairs $(\x_i,y_i)$, with $\x_i\sim \normal(0,\mtx{I}_{p})$ representing the features and $y_i\in \reals$ the corresponding label given by \footnote{We note that our analysis in Section~\ref{sec:proofs} can be extended to general Gaussian linear regression where $\vct{x}_i\sim \normal(0,\bSigma)$. This however requires more involved derivations that are not included in this version.}
\begin{align}\label{eq:linear}
y_i = \<\vct{x}_i,\bth_0\>+ w_i\, \quad\text{where}\quad w_i\sim \normal(0,\sigma_0^2)\,.
\end{align}  
We also focus on training linear models of the form $f_{\vct{\theta}}(\x)=\langle \vct{x},\vct{\theta}\rangle$ via a quadratic loss $\ell(\x,y;\bth) = \frac{1}{2}(y-\<\x,\bth\>)^2$ and consider perturbation sets of the form $\cS:= \{\bdelta\in \reals^p:\, \twonorm{\bdelta}\le \epsilon\}$ where $\eps$ is a measure of the adversary's  power. To make the dependence on $\eps$ explicit in our notation, we replace $\hth^{\cS}$ for this choice of $\cS$ by $\hth^\eps$. In this case~\eqref{eq:hth-cS} takes the form
 \begin{align}\label{eq:htheps}
\hth^{\eps} \in \arg\min_{\bth\in \reals^p}\, \max_{\twonorm{\bdelta_i}\le \eps}\, \frac{1}{2n} \sum_{i=1}^n \left(y_i-\<\x_i+\bdelta_i,\bth\>\right)^2\,.
\end{align}  
\end{definition}
Next we formally introduce the asymptotic regime of interest in this paper. 
\smallskip

\noindent{\bf Asymptotic regime.} For a given sample size $n$, we define an \emph{instance} of the standard Gaussian model by a tuple $(\bth_0,p,\sigma_0)$, with $\bth_0\in \reals^p$, $p\in \mathbb{N}$ and $\sigma_0\in \reals_{\ge0}$. We consider sequence of instances of the Gaussian model indexed by the sample size $n$.

\begin{definition}\label{def:converging}
The sequence of instances $\{\bth_0(n), p(n), \sigma_0(n)\}_{n\in \mathbb{N}}$ indexed by $n$ is called a converging sequence if:
\begin{itemize}
\item We have $\frac{n}{p}\to \delta \in (0,\infty)$ and $\frac{\sigma_0^2(n)}{p} \to \sigma^2$ as $n\to \infty$.
\item Empirical second moment of the signal converges, i.e., $\frac{1}{p} \sum_{i=1}^p \theta_{0,i}(n)^2 \to V^2 <\infty$, as $n\to \infty$.
\end{itemize}
\end{definition}

In summary, we have introduced the following notations and terms which will be used throughout the paper: the dimension $p$, number of training data points $n$, overparametrization parameter $\delta = n/p$, normalized noise power $\sigma^2$, normalized norm of the true model $V^2$, and the adversary's power $\eps$.  
\vspace{-.3cm}
\section{Main Results}\label{sec:main}
In this paper we wish to understand fundamental tradeoffs between standard and adversarial risks as well as what can be achieved by modern adversarial training approaches. In Section \ref{sec31} we characterize the fundamental tradeoff between standard and adversarial risk achievable by any algorithm regardless of the computational power and the size of the available training data. Then in Section \ref{sec32} we turn our attention to precisely characterizing the standard and adversarial accuracy tradeoffs achieved by modern adversarial training algorithms of the form \eqref{eq:hth-cS}. This is carried out in a high-dimensional regime where the size of the training data $n$ and the number of parameters $p$ grow proportional to each other with their ratio $n/p \to \delta$ for fixed $\delta \in(0,+\infty)$. Next, in Section \ref{sec33} we focus on studying the role of that the size of the training data plays and how it affects the standard accuracy. Finally, in Section \ref{sec34} we prove the emergence of a phenomena  in adversarial training similar to the so-called double-descent phenomena without adversarial training. 

\subsection{Fundamental tradeoffs between standard and adversarial risk}\label{sec:pareto}
\label{sec31}
Motivated by the conflict observed between standard and adversarial risk in modern adversarial training \cite{DBLP:conf/iclr/MadryMSTV18}, we first wish to understand the fundamental tradeoffs that can be achieved between the two objectives. That is, the optimal tradeoff that can be achieved between standard and adversarial risk objectives for \emph{any estimator} $\hth$ even with access to infinite computational power and infinite training data.  We discuss the tradeoffs achievable by specific algorithms with finite training data in the next section. 
\medskip

\noindent\textbf{$(\SR, \AR)$ Region and its Pareto Optimal Curve:} As discussed previously in Section \ref{secform} for an estimator $\hth$ we use $\SR(\hth)$ and $\AR(\hth)$ to denote the standard and adversarial risks achieved by $\hth$.  Thus, for any estimator $\hth$ we obtain a point $(\SR(\hth), \AR(\hth))$ in the 2-d plane.  We refer to the set of all such points, for all $\hth \in \R^p$, as the $(\SR, \AR)$ region. To obtain the optimal tradeoff between standard and adversarial risks we need to characterize the Pareto-optimal points of this region.\footnote{Given a region $\mathcal{C} \in \R^2$, a point $(x,y) \in \mathcal{C}$ is Pareto optimal if there exists no other point $(x', y') \in \mathcal{C}$ s.t. $x' \leq x$ and $y' \leq y$. }

In the linear regression setting of this paper the expressions of standard accuracy \eqref{SR} and adversarial accuracy \eqref{AR} are convex functions of $\bth$. Therefore, using standard results in multi-objective optimization we can derive all the Pareto optimal points of the $(\SR, \AR)$ region, by minimizing a weighted combination of these two accuracies for different weights $\lambda$.
\begin{align}\label{eq:pareto}
\bth^\lambda =\arg\min_{\bth}  \,\,{\lambda} \, \overbrace{\E\left\{(y-\<\x,\bth\>)^2\right\}}^{\text{standard risk}} \, +  \,\overbrace{\E\Big\{\max_{\twonorm{\bdelta}\le \etest} (y-\<\x+\bdelta,\bth\>)^2 \Big\}}^{\text{adversarial risk}} \,.
\end{align} 
The Pareto-optimal curve is then given by $\{(\SR(\bth^\lambda), \AR(\bth^\lambda): \; \lambda\ge0\}$. 
\medskip

\noindent\textbf{Analytical Expression of the Optimal Tradeoffs:} Before we proceed to calculate $\bth^\lambda$, we derive the standard and adversarial risks ($\SR(\hth)$ and $\AR(\hth)$) as a functions of $\bth_0$ and $\sigma_0^2$ in the Gaussian linear regression model. We defer the proof of this Lemma to Section \ref{lem:SR-ARpf}.
\vspace{-.2cm}
\begin{lemma}\label{lem:SR-AR} Consider the linear regression setting of Definition \ref{linregmod}. For a given estimator $\hth$ the standard risk \eqref{SR} is equal to
\vspace{-.1cm}
\begin{align*}
\SR(\hth):=\frac{1}{p}\E\left[(y-\<\x,\hth\>)^2\right]= \frac{\sigma_0^2}{p} + \frac{1}{p} \twonorm{\hth-\bth_0}^2,
\end{align*}
\vspace{-.2cm}
Furthermore, the adversarial risk \eqref{AR} with a corruption level of $\etest$ is equal to
\begin{align*}
\AR(\hth):=& \frac{1}{p}\E\bigg[\max_{\twonorm{\bdelta}\le \etest} (y-\<\x+\bdelta,\hth\>)^2 \bigg]\\
=&\frac{1}{p} \left(\sigma_0^2+ \twonorm{\hth-\bth_0}^2+ \etest^2\twonorm{\hth}^2\right)+2\sqrt{\frac{2}{\pi}} \frac{\etest}{\sqrt{p}} \twonorm{\hth} \left(\frac{\sigma_0^2}{p}+ \frac{1}{p}\twonorm{\hth-\bth_0}^2\right)^{1/2}.
\end{align*}
\end{lemma}
With a precise expression of the standard and adversarial risk in hand our next theorem characterizes the solution $\bth^\lambda$ of the optimization problem~\eqref{eq:pareto} which in conjunction with Lemma \ref{lem:SR-AR} determines the Pareto-optimal tradeoff curve. We defer the proof of this result to Section \ref{pro:hth-Lampf}.
\begin{propo}\label{pro:hth-Lam}
Under the linear regression setting of Definition \ref{linregmod}, the solution $\bth^\lambda$ of the
optimization problem~\eqref{eq:pareto} is given by
\vspace{-0.65cm}
\begin{align*}
\bth^\lambda = (1+\gamma_0^{\lambda})^{-1} \bth_0\,,
\end{align*}
\vspace{-0.25cm}
with $\gamma_0^{\lambda}$ the fixed point of the following two equations:
\begin{align*}
\gamma_0^{\lambda} = \frac{\etest^2+\sqrt{\frac{2}{\pi}} \etest A^\lambda}{1+\lambda + \sqrt{\frac{2}{\pi}} \frac{\etest}{A^\lambda}}\,\quad\text{and}\quad
A^\lambda = \frac{1}{\twonorm{\bth_0}} \left((1+\gamma_0^\lambda)^2\sigma_0^2+ (\gamma_0^\lambda)^2\twonorm{\bth_0}^2\right)^{1/2}\,.
\end{align*}
\end{propo}
In Figure~\ref{fig:curves} we plot the Pareto optimal curve in the $(\SR,\AR)$ plane in black for an instance where $\etest=0.5$ and the normalized norm of the true model and the noise power are both equal to one ($\sigma=V=1$). This curve serves as a fundamental limit on the performance of any algorithm even with access to infinite data and computational power. This figure also contains algorithmic tradeoffs which we discuss in further detail in the next section. In particular, in the next section we precisely characterize the SR-AR tradeoff achieved by a specific adversarial training algorithm. 
\begin{figure}
\centering
  \includegraphics[scale=1.2]{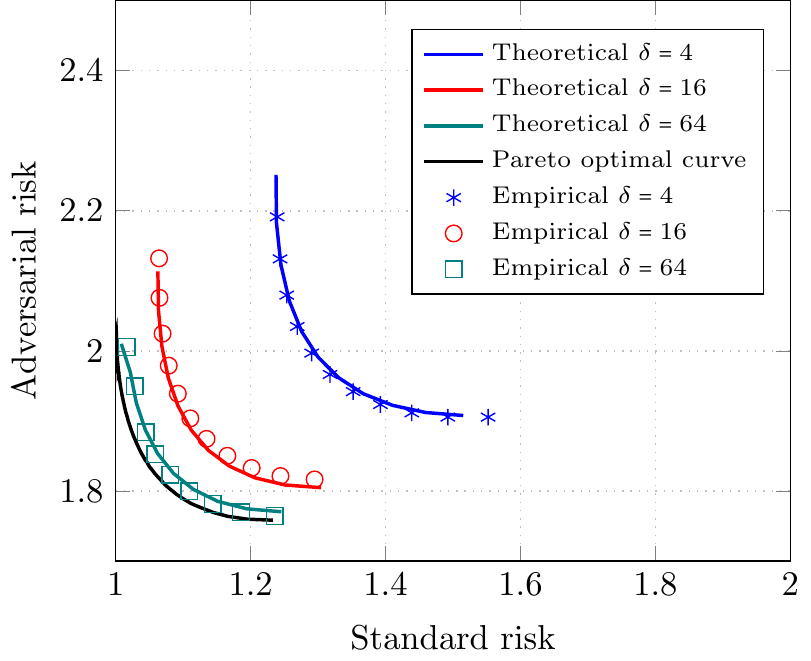}
  \caption{Pareto optimal curve along with algorithmic curves for several values of $\delta$. As $\delta$ grows the algorithmic tradeoff curves approach the fundamental Pareto optimal curve. The dots correspond to the empirical data obtained by solving for the optimal solution $\hth^\eps$ of \eqref{eq:htheps} using gradient descent and then computing $(\SR(\hth^\eps), \AR(\hth^\eps))$ from Lemma~\ref{lem:SR-AR} with different values of $\eps$. Here, $\sigma = 1$, $V=1$, $p = 1000$, and $\etest=0.5$.}
  \label{fig:curves}
\end{figure}

\subsection{Algorithmic tradeoffs between standard and adversarial risks}
\label{sec32}
Given the fundamental tradeoff of the previous section,  the natural question that arises is whether it is possible to achieve this tradeoff algorithmically with only finite data and computational power? Specifically, what is the tradeoff achieved by common adversarial training algorithms? In this section we consider the class of estimators $\hth^\eps$ constructed through the saddle point problem~\eqref{eq:htheps} for various values $\eps$ at training i.e.~$\{\hth^\eps:\, \eps\ge 0\}$. We wish to precisely derive the tradeoff curve between the standard and the adversarial risks achieved by this class of estimators. We refer to such curve as \emph{algorithmic} tradeoff curve since it corresponds to the specific class of saddle point estimators as opposed to the Pareto optimal trade off curves studied in Section~\ref{sec:pareto} which serve as lowerbound for any estimator. To avoid any confusion about the tradeoffs discussed we would like to emphasize that: 
\begin{itemize}
\item[(i)] In the \emph{training} phase, we are \emph{varying} the adversarial power $\eps$, and accordingly, obtain a  range of estimators $\hth^\eps$ by solving \eqref{eq:htheps}. 
\item[(ii)] At \emph{test} time, the adversarial power is fixed to a given value $\etest$ and we will measure the (expected) standard and adversarial risks of the trained estimators $\hth^\eps$ with respect to the true adversarial power $\etest$. By varying $\eps$ at training time, we expect to sweep a tradeoff between standard and adversarial risks, i.e. estimators $\hth^\eps$ with large $\eps$ should have a smaller adversarial risk but higher standard risk, and estimators with smaller $\eps$ should behave the opposite. 
\end{itemize}

\vspace{.2cm}
\noindent\textbf{Analytical Expression of the Algorithmic Tradeoffs.} Our goal for the rest of this section is to analytically derive the algorithmic tradeoffs in terms of the overparametrization parameter $n/p \to \delta \in (0, \infty)$ which represents the number of training data points per dimension.  We focus on converging sequences of Gaussian model instances as described in Definition~\ref{def:converging}.  Recall that By virtue of Lemma~\ref{lem:SR-AR}, in order to derive the asymptotic standard and adversarial risk of $\hth^\eps$, it suffices to obtain an exact characterization of the asymptotic error $\lim_{n\to\infty}\frac{1}{p}\twonorm{\hth^\eps-\bth_0}^2$ and the asymptotic estimator norm $\lim_{n\to\infty}\frac{1}{p}\twonorm{\hth^\eps}^2$. This is the subject of the next theorem formally proven in Section \ref{Thm5pf}.
\begin{theorem}\label{thm:main}
Let $\{(\bth_0(n),p(n),\sigma_0(n))\}_{n\in \mathbb{N}}$ be a converging sequence of instances of the standard Gaussian design model. Consider the linear regression model~\eqref{eq:linear} and let $\hth^\eps$ be a solution of \eqref{eq:htheps}. If $\eps,\delta>0$ or $\eps=0$, $\delta>1$, then 
\begin{itemize}[leftmargin=*]
\item[(a)] The following convex-concave minimax scalar optimization has a unique solution $(\alpha_*,\beta_*,\gamma_*,\tau_{h*}, \tau_{g*})$:
\begin{align}
\label{eq019}
 \max_{0\le\beta\le K_\beta}\sup_{\gamma,\tau_h\ge0}\;\;\min_{0\le \alpha\le K_\alpha}\;\;\min_{\tau_g\ge 0}\;\;  \quad D(\alpha,\beta, \gamma,\tau_h,\tau_g)\,,\quad\text{where}
 \end{align}
\vspace{-0.5cm}
 \begin{align}\label{eq019-2}
 &D(\alpha,\beta, \gamma,\tau_h,\tau_g):=\frac{\delta\beta}{2(\tau_g+\beta)} \left(\alpha^2+\sigma^2\right)\nn\\
 &+\delta \mathbb{1}_{\big\{\frac{\gamma(\tau_g+\beta)}{\delta\eps\beta\sqrt{\alpha^2+\sigma^2}}>\sqrt{\frac{2}{\pi}}\big\}}\frac{\beta^2(\alpha^2+\sigma^2)}{2\tau_g(\tau_g+\beta)}\left(\erf\left(\frac{\tau_*}{\sqrt{2}}\right)-\frac{\gamma(\tau_g+\beta)}{\delta\eps\beta\sqrt{\alpha^2+\sigma^2}}\;\tau_*\right)\nn\\
 &-\frac{\alpha}{2\tau_h}(\gamma^2 +\beta^2)+ \gamma\sqrt{\frac{\alpha^2\beta^2}{\tau_h^2}+V^2} -\frac{\alpha\tau_h}{2}+\frac{\beta\tau_g}{2}\,,
\end{align}
and $\tau_*$ is the unique solution to 
\begin{align*}
\frac{\gamma(\tau_g+\beta)}{\delta\eps\beta\sqrt{\alpha^2+\sigma^2}}-\frac{\beta}{\tau_g}\tau-\tau\cdot \erf\left(\frac{\tau}{\sqrt{2}}\right)-\sqrt{\frac{2}{\pi}} e^{-\frac{\tau^2}{2}}=0
\end{align*}
\item[(b)] It holds in probability that $\lim_{n\to\infty}\frac{1}{p} \twonorm{\hth^\eps-\bth_0}^2 = \alpha_*^2$.
\item[(c)] It holds in probability that 
\vspace{-0.6cm}
\begin{align}
\lim_{n\to\infty}\frac{1}{\sqrt{p}}\twonorm{\hth^\eps}= \frac{\beta_* \tau_*\sqrt{\alpha_*^2+\sigma^2}}{\eps \tau_{g*}}\,.
\end{align}
\end{itemize}
\end{theorem}
We note that the loss \eqref{eq:htheps} and its optimal solution are a rather complicated and high-dimensional function of the features/label pairs $\{(\vct{x}_i,y_i)\}_{i=1}^n$. Nevertheless the Theorem above provides a precise characterization of its properties using a 5 dimensional convex-concave mini-max optimization problem! Such a precise characterization allows us to provide a precise understanding of the standard and adversarial accuracies. In particular, combining Theorem~\ref{thm:main} (parts (b)-(c)) with Lemma~\ref{lem:SR-AR} we can obtain the asymptotic values of $\SR(\hth^\eps)$ and $\AR(\hth^\eps)$, and derive the algorithmic tradeoff curve achieved by the class $\{\hth^\eps:\, \eps\ge 0\}$ as $\eps$ varies (discussed in the next corollary proven in Section \ref{cor6pf}).
\begin{corollary}\label{cor6}
Let $\{(\bth_0(n),p(n),\sigma_0(n))\}_{n\in \mathbb{N}}$ be a converging sequence of instances of the standard Gaussian design model. Consider the linear regression model~\eqref{eq:linear} and let $\hth^\eps$ be a solution of \eqref{eq:htheps}. Further assume that $\eps,\delta>0$ or $\eps=0$, $\delta>1$.  Also denote $(\alpha_*,\beta_*,\gamma_*,\tau_{h*}, \tau_{g*})$ as the optimal solutions of the minimax optimization~\eqref{eq019}. Then, the following identities hold in probability:
\begin{align}
\lim_{n\to\infty} \SR(\hth^\eps) &= \sigma^2 + \alpha_*^2\,,\\
\lim_{n\to\infty} \AR(\hth^\eps) &=  \left(\sigma^2+\alpha_*^2 + \etest^2 (\alpha_*^2+\sigma^2) \left(\frac{\beta_*\tau_*}{\eps \tau_{g*}}\right)^2\right)
+2\sqrt{\frac{2}{\pi}} \frac{\etest\beta_*\tau_*}{\eps \tau_{g*}} (\sigma^2+ \alpha_*^2)\,.
\end{align}
\end{corollary}
 The corollary above provides a precise characterization of the standard and adversarial accuracy achieved by the adversarial training algorithm consisting of running gradient descent on the saddle point problem~\eqref{eq:htheps}. In Figure~\ref{fig:curves}, we plot the algorithmic tradeoff curve for several values of $\delta$ as well as the empirical values obtained by running gradient descent. As we observe, our theoretical prediction and the empirical values are rather close match even for moderately large parameter values ($p=1000$). Such a precise characterization allows us to rigorously study a variety of phenomena. We mention one such phenomena below and discuss others in the coming sections. The plots in Figure~\ref{fig:curves} clearly show that when $\delta$ grows the algorithmic tradeoff curve approaches the Pareto-optimal tradeoff curve. In other words, one can achieve optimal tradeoff of standard and adversarial risks by the specific class of estimators $\hth^\eps$ constructed by the saddle point problem~\eqref{eq:htheps}. This observation is formally stated in the next theorem with the proof deferred to Section \ref{delta_limitpf}. 
\begin{theorem} \label{delta_limit}
Let $\{(\bth_0(n),p(n),\sigma_0(n))\}_{n\in \mathbb{N}}$ be a converging sequence of instances of the standard Gaussian design model. Consider the linear regression model~\eqref{eq:linear}, and let $\hth^\eps$ be a solution of \eqref{eq:htheps} and $\bth^\lambda$ the solution of \eqref{eq:pareto}. Then for any $\lambda\ge 0$ there exists $\eps = \eps(\sigma, V, \etest, \lambda)$, such that
\begin{eqnarray}
\lim_{\delta \to\infty}\lim_{n \to\infty} \SR(\hth^\eps) = \lim_{p\to\infty} \SR(\bth^\lambda)\,,\quad\quad
\lim_{\delta \to\infty} \lim_{n \to\infty} \AR(\hth^\eps) = \lim_{p\to\infty} \AR(\bth^\lambda)\,. 
\end{eqnarray}
\end{theorem}
The theorem above formally proves that in the infinite data limit ($\delta\rightarrow +\infty$) one of the commonly used adversarial training algorithms achieves the optimal tradeoff between standard and robust accuracies.

\subsection{The role of the size of the training data and overparameterization}
\label{sec33}
\begin{figure}
\centering
\begin{minipage}{.485\textwidth}
  \centering
    \includegraphics[scale=1]{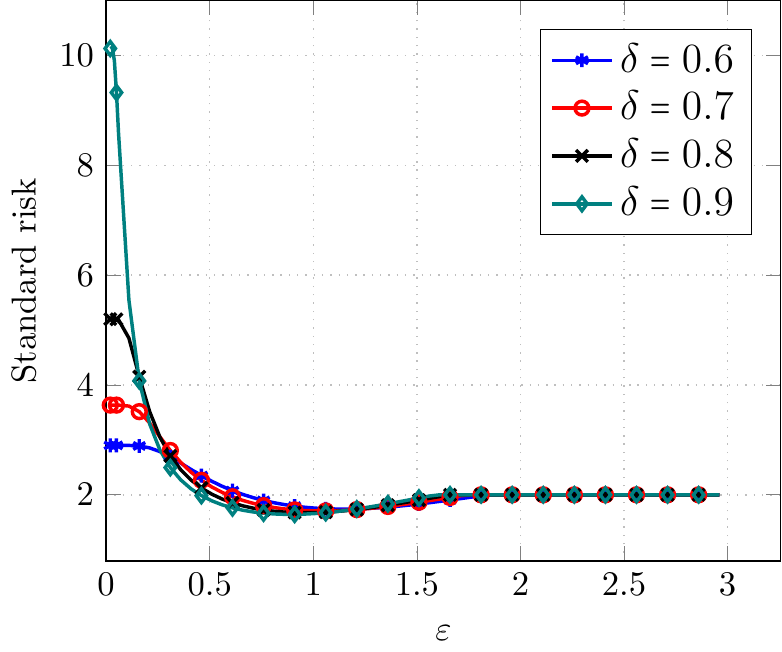}
   \caption*{(a) Theoretical curves}
\end{minipage}\hspace{0.3cm}
\begin{minipage}{.485\textwidth}
  \centering
  \includegraphics[scale=1]{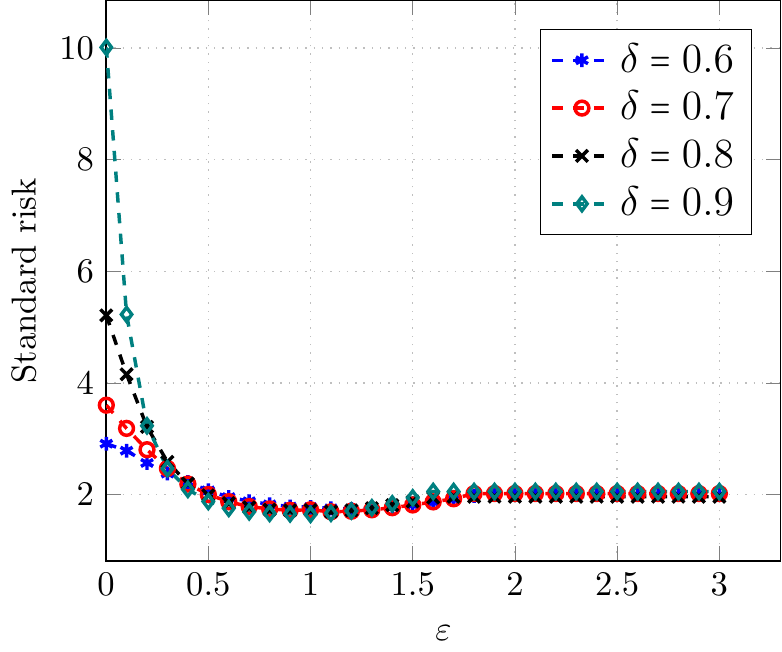}
   \caption*{(b) Empirical curves}
\end{minipage}
\caption{Standard risk ($\SR(\hth^\eps)$) versus $\eps$ for several values of $\delta<1$.  Left panel corresponds to the theoretical curve obtained by Theorem~\ref{thm:main} (with $\sigma =1$ and $V=1$), and the right panel corresponds to the empirical results (with $\sigma =1$ and $\theta_{0,i}\sim \normal(0,1)$). The empirical results are averaged over 100 different realizations of noise and features.   As $\delta$ grows to one, we observe a faster decay in the standard risk with respect to the adversarial power $\varepsilon$.  
\label{fig:SR_deltaLess1}}
\end{figure}
As discussed, our precise understanding of the optimal solution of adversarial training allows us to precisely characterize the effect of various phenomena. In particular in this section we focus on the role of the size of the training data. We begin by considering the common scenario in modern learning where trained models often consist of more parameters than the training data set. In Figure~\ref{fig:SR_deltaLess1}-(a) we plot the standard risk, using Theorem~\ref{thm:main} Part (b), versus $\eps$ for different values of $\delta<1$. As we observe for small to moderate values of $\eps$, this curve is decreasing in $\eps$, which implies that adversarial training helps with improving standard accuracy. The standard risk falls steeper as $\delta$ becomes closer to one. In Figure~\ref{fig:SR_deltaLarger1}-(a) we observe a similar trend for $\delta>1$. However, as $\delta$ grows larger than one, the positive effect of the adversarial training on the standard risk falters and we see a lower decline. When $\delta= 10$, the curve almost levels at $\eps=0$ and then starts to becomes increasing with $\eps$. In other words, for larger $\delta$ we start to see that adversarial training has a negative effect on standard risk starting from smaller values of $\eps$. Our theoretical prediction are in line with recent empirical observations of a similar flavor \cite{tsipras2018robustness} observed in neural networks. Therefore, our theoretical results formally proves the emergence of such a behavior. We provide further insight into the emergence of this phenomena a long with some more rigorous theoretical guarantees in Appendix \ref{insight}.

\begin{figure}
\centering
\begin{minipage}{.485\textwidth}
  \centering
    \includegraphics[scale=0.95]{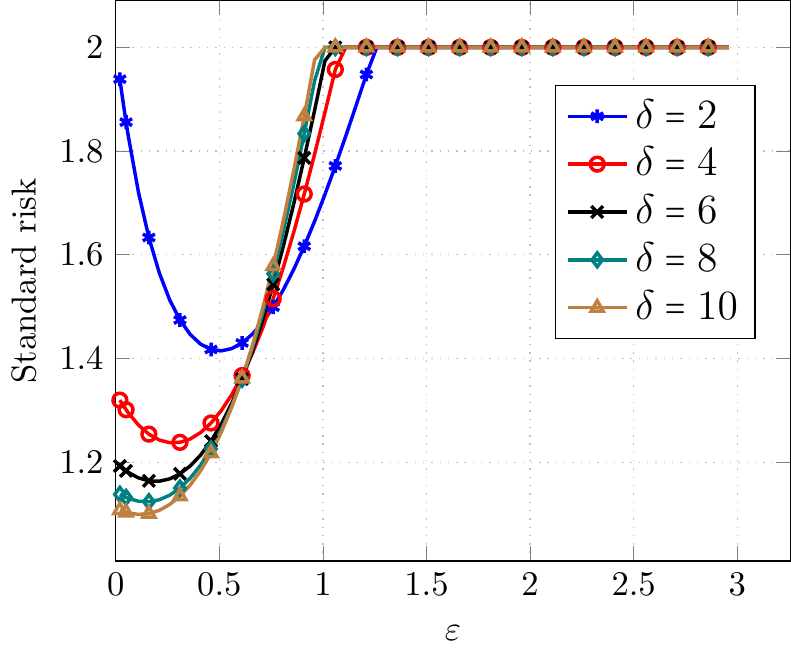}
   \caption*{(a) Theoretical curves}
\end{minipage}\hspace{0.3cm}
\begin{minipage}{.485\textwidth}
  \centering
  \includegraphics[scale=0.95]{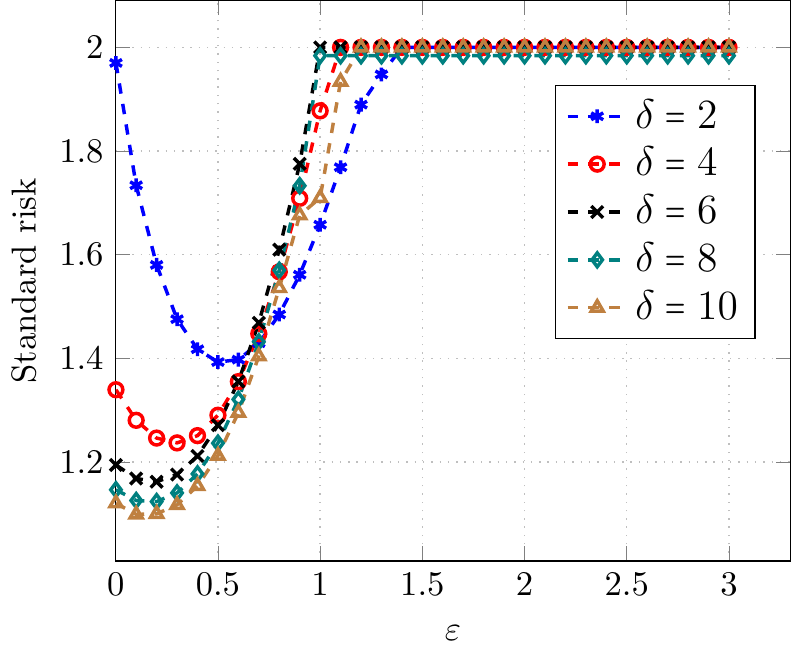}
   \caption*{(b) Empirical curves}
\end{minipage}
\caption{Standard risk ($\SR(\hth^\eps)$) versus $\eps$ for several values of $\delta>1$.  Left panel corresponds to the theoretical curve obtained by Theorem~\ref{thm:main} (with $\sigma =1$ and $V=1$), and the right panel corresponds to the empirical results (with $\sigma =1$ and $\theta_{0,i}\sim \normal(0,1)$). The empirical results are averaged over 100 different realizations of noise and features.  As $\delta$ grows, we observe a slower decay in the standard risk at small $\eps$ due to adversarial training. For $\delta = 10$, the standard risk has a small initial slope with respect to $\eps$ and then starts to increase rapidly. Put differently, with larger $\delta$, the negative effect of adversarial training on the standard risk starts at smaller $\eps$.  
\label{fig:SR_deltaLarger1}}
\end{figure}

%

\subsection{Double-descent in adversarial training}
\label{sec34}
When $\eps=0$, the estimator $\hth^\eps$ given by~\eqref{eq:htheps} reduces to the least-squares estimator. It is known that the plot of standard risk as a function of number of model complexity ($1/\delta = p/n$) exhibits a so-called `double-descent' behavior~\cite{belkin2018understand,belkin2018reconciling,hastie2019surprises}. Namely, (1) up to the interpolation threshold $\delta =1$ (beyond which the estimator achieves zero training error and the model interpolates the training data) the risk curve follows a U-shape; the risk first decreases as $p$ increases because the model becomes less biased but then starts to increase because of the inflated variance of the estimator. (2) After the peak at the interpolation threshold, the risk decreases and essentially attains its global minimum at `infinite' model complexity (extremely overparametrized regime). 
    
The double-descent phenomenon is not limited to neural networks and have been empirically observed in a variety of models including random features and random forest models. Recently, analytical derivation of this phenomenon has been developed for least square regression and random features model~\cite{tsipras2018robustness,mei2019generalization}. For least square regression with Gaussian covariates, it is shown that the global minimum of the risk is achieved in the underparametrized setting $\delta>1$ (unless miss-specified structures are assumed). Nonetheless, these work are focused on training with unperturbed features. 
    
In Figure~\ref{fig:DD} (a), we plot the standard risk (theoretical predictions from Theorem~\ref{thm:main}) versus $1/\delta = p/n$, for several values of adversarial power $\eps$. We also depict the empirical version of these curves in Figure Figure~\ref{fig:DD} (b). These plots demonstrate that the double-descent phenomena continues to hold even with adversarial training. Interestingly however the interpolation threshold changes with $\eps$. For small $\eps$, we observe double-descent behavior with the interpolation threshold $\delta\approx 1$. However, as $\eps$ increases the location of the peak shifts to higher values of $1/\delta$. 

\begin{figure}
\centering
\begin{minipage}{.485\textwidth}
  \centering
  \includegraphics[scale=1]{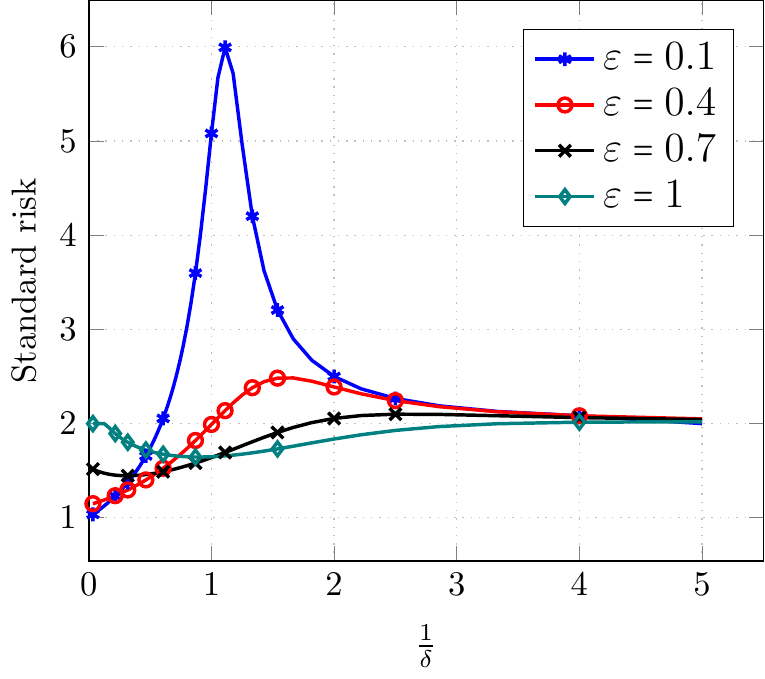}
  \caption*{(a) Theoretical curves}
  \end{minipage}\hspace{0.3cm}
  \begin{minipage}{.485\textwidth}
  \centering
  \includegraphics[scale=1]{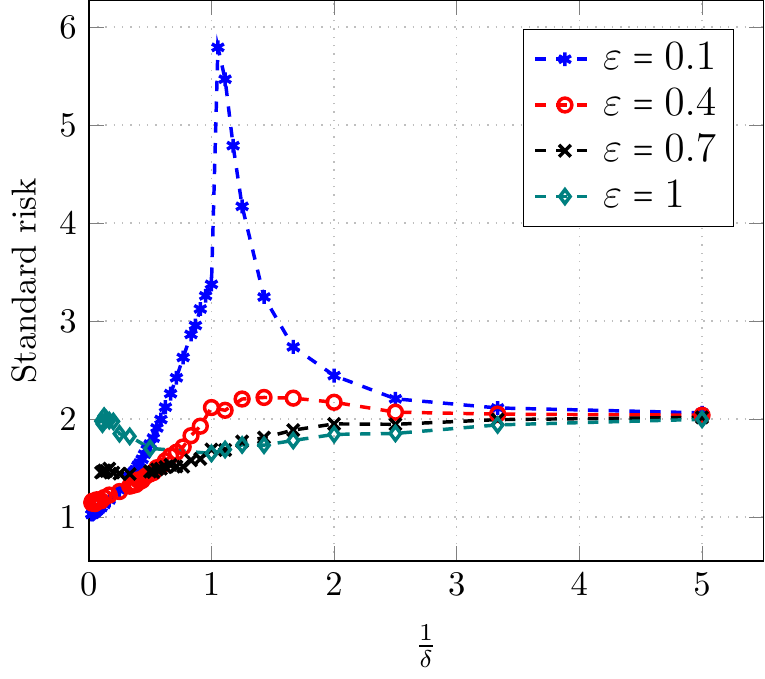}
  \caption*{(b) Empirical curves}
  \end{minipage}
  \caption{Standard risk versus model complexity $1/\delta = p/n$. Left panel corresponds to the theoretical curve obtained by Theorem~\ref{thm:main} (with $\sigma =1$ and $V=1$), and the right panel corresponds to the empirical results (with $\sigma =1$ and $\theta_{0,i}\sim \normal(0,1)$). Here, we recover the double-descent behavior where the interpolation threshold shifts with $\eps$.}\vspace{-0.3cm}
  \label{fig:DD}
\end{figure}

\vspace{-.3cm}
\section{Further Related Work}
\label{related}

The trade-off between standard and adversarial accuracy has been studied recently in \cite{DBLP:conf/iclr/MadryMSTV18, DBLP:conf/nips/SchmidtSTTM18,tsipras2018robustness,raghunathan2019adversarial, DBLP:conf/icml/ZhangYJXGJ19, pydi2019adversarial}.  An central question is whether standard and robust objectives are fundamentally at conflict? In other words, is there a predictor that can achieve both optimal standard accuracy and robust accuracy when the number of training data samples is sufficiently large? In this regard, \cite{tsipras2018robustness,DBLP:conf/icml/ZhangYJXGJ19} construct learning problems  where the optimal robust accuracy is fundamentally at conflict with the standard accuracy, i.e.  no predictor can achieve both optimal standard accuracy and robust accuracy even in the infinite data limit. However, there are clearly many natural  learning problems in which a predictor with optimal standard and high robust accuracy exists (hence the two objectives are not at conflict).  An instance of such cases has been studied in \cite{raghunathan2019adversarial} suggesting that the inconsistency between adversarial accuracy and standard accuracy may be due to insufficient number of training samples. In contrast, in this paper we have shown that a fundamental tradeoff exists between the two accuracies in linear regression even with limited~samples.



Another line of work considers the tradeoff between standard and robust accuracy when the capacity of the learning model varies \cite{nakkiran2019adversarial,DBLP:conf/nips/GaoCLHWL19}. In particular, \cite{nakkiran2019adversarial} provides classification problems where simple classifiers with high standard accuracy exist; but having high robust accuracy is possible through more complex classifiers.
The notions of capacity and complexity in the presence of adversarially perturbed inputs (a.k.a. adversarially robust learnability) have also been studied in a series of interesting papers \cite{DBLP:conf/icml/BubeckLPR19, DBLP:conf/nips/CullinaBM18, DBLP:journals/corr/abs-1810-09519, DBLP:conf/icml/YinRB19, DBLP:conf/colt/MontasserHS19}.  In particular \cite{DBLP:conf/colt/MontasserHS19} show that any hypothesis class with finite VC dimension is adversarially-robust PAC learnable in the $\ell_\infty$ metric using modified (improper) learning rules. Finally, let us point out that under specific high-dimensional data distributions (e.g. isotropic Gaussian), any classifier becomes highly vulnerable to adversarial $\ell_2$ perturbations  \cite{gilmer2018adversarial,DBLP:conf/iclr/ShafahiHSFG19}. Thus the adversarial error approaches $1$ as the dimension grows. This phenomenon does not occur in our regression setting as the regression loss is smoothly varying as opposed to the classification error.  


\vspace{-.3cm}
\section{Sketch and roadmap of the proof}
To be able to provide a precise characterization of the various tradeoffs we need to develop a precise understanding of the adversarial training objective 
\vspace{-0.3cm}
\begin{align*}
\underset{\vct{\theta}\in\R^p}{\min}\text{ }\mathcal{L}(\vct{\theta}):=\underset{\vct{\theta}\in\R^p}{\min}\text{ }\underset{\twonorm{\vct{\delta}_i}\le \eps}{{\rm \max}}\;\;  \frac{1}{2n} \sum_{i=1}^n \left(y_i - \<\x_i +\bdelta_i,\bth\>\right)^2,
\end{align*}
and its optimal solution $\hth^{\eps} \in \arg\min_{\bth\in \reals^p }\mathcal{L}(\vct{\theta})$. To achieve this we carry out the following steps.
\medskip

\noindent\textbf{Step I: Simplification of the loss (Section \ref{step1}).} The maximization objective is equal to the optimal value of a maximization problem and hence characterizing its properties directly is challenging. In the first step of our proof we show that one can in-fact solve this maximization problem and derive an expression for the loss in closed form. Specifically, we show
\begin{align}
\label{simpobj}
\mathcal{L}(\vct{\theta})=\frac{1}{2n} \sum_{i=1}^n \left(|y_i-\<\x_i,\bth\>|+ \eps \twonorm{\bth}\right)^2=\frac{1}{2n}\twonorm{\abs{\vct{y}-\mtx{X}\vct{\theta}}+\eps\twonorm{\vct{\theta}}}^2.
\end{align}
The main intuition behind this derivation is that one can think of the min-max optimization problem above as a game between a learner and an adversary where the learner first chooses a parameter $\vct{\theta}$ and then the adversary changes each feature $\vct{x}_i$ given the label $y_i$ and the learner's choice of $\vct{\theta}$. We show that the best choice for the adversary to maximize the error is to pick $\vct{\delta}_i$ in the direction of $\vct{\theta}$ with a magnitude of $\eps$ (maximum power of the adversary) and with the sign  of the misfit on the $i$ the training data point ($\sgn{\langle \vct{x}_i,\vct{\theta}\rangle-y_i}$). We formally prove this result by connecting it to the well-known trust region subproblem in optimization.
\medskip

\noindent\textbf{Step II: Reduction to an Auxiliary Optimization (AO) problem (Section \ref{step2}).}\\
The loss \eqref{simpobj}, while significantly simplified, is still rather complicated and it is completely unclear how to precisely characterize its behavior and the quality of its optimal solution. In particular, the dependence on the random data matrix $\mtx{X}$ is still rather complex hindering statistical analysis even in an asymptotic setting. To bring the optimization problem into a form more amenable to precise asymptotic analysis we carry out a series of reformulations of the optimization problem. First, we rescale the loss. Next we consider a change of variable of the form $\vct{z}=\frac{1}{\sqrt{p}}(\vct{\theta}-\vct{\theta}_0)$ and add new variables by adding equality constraints. Finally, we use duality to cast the problem into a mini-max form. Combining these steps we arrive at the following equivalent Primal Optimization (PO) problem
 \begin{align}
 \label{equi1}
\min_{\z\in\R^p,\vct{v}\in\R^n}  \max_{\vct{u}\in\R^n}\;\;\; \frac{1}{\sqrt{p}}\vct{u}^T\bX \vct{z} - \frac{1}{\sqrt{p}}\vct{u}^T \vct{\omega} +\frac{1}{\sqrt{p}}\vct{u}^T \vct{v}+\ell(\vct{v};\vct{z}),
 \end{align}
 where $\vct{\omega}=\frac{\w}{\sqrt{p}}\in\R^n$ is a Gaussian vector with i.i.d.~$\mathcal{N}(0,\sigma^2)$ entires and 
 \vspace{-0.4cm}

  \begin{align*}
\ell(\vct{v};\vct{z}):= \frac{1}{2p}\left(\twonorm{\vct{v}}^2+2\frac{\eps}{\sqrt{p}}\onenorm{\vct{v}}\twonorm{\vct{\theta}_0+\sqrt{p}\vct{z}}+\frac{\eps^2}{p}\twonorm{\vct{\theta}_0+\sqrt{p}\vct{z}}^2\right).
 \end{align*}
This equivalent form may be counter-intuitive as we started by simplifying a different mini-max optimization problem and we have now again introduced a new maximization! The main advantage of this new form is that it is in fact affine in the data matrix $\bX$. This particular form allows us to use a powerful extension of a classical Gaussian process inequality due to \cite{gordon1988milman} known as Convex Gaussian Minimax Theorem (CGMT) \cite{thrampoulidis2015regularized} which focuses on characterizing the asymptotic behavior of mini-max optimization problems that are affine in a Gaussian matrix $\bX$. This result enables us to characterize the properties of \eqref{equi1} by studying the asymptotic behavior of the following, arguable simpler, \emph{Auxiliary Optimization (AO)} problem instead
\begin{align}
\label{equiAO}
\min_{\vct{z}\in\R^p, \vct{v}} \max_{\vct{u}\in\R^n}\; \frac{1}{\sqrt{p}}\left(\twonorm{\vct{z}} \vct{g}^T\vct{u} + \twonorm{\vct{u}} \vct{h}^T\vct{z}- \vct{u}^T \vct{\omega} +\vct{u}^T \vct{v}\right)+\ell(\vct{v};\vct{z}).
\end{align}
We emphasize that the relationship between the above AO problem \eqref{equiAO} and how it is exactly related to the PO problem \eqref{equi1} is much more intricate and technical. See Section \ref{step2} for details. 

The CGMT framework has been recently used to derive precise characterization of the generalization error of the max-margin linear classifiers in overparametrized regime with separable data~\cite{deng2019model,montanari2019generalization}. Also \cite{liang2020precise} uses the CGMT framework to analyze max-$\ell_1$-margin classifiers. 
\smallskip

\noindent\textbf{Step III: Scalarization of the Auxiliary Optimization (AO) problem (Section \ref{step3}).}\\
In this step we further simplify the AO problem in \eqref{equiAO}. In particular we show the asymptotic behavior of the AO can be characterized rather precisely via the following scalar optimization problem involving five variables:
\begin{align}
\label{eq0192}
 \max_{0\le\beta\le K_\beta}\sup_{\gamma,\tau_h\ge0}\;\;\min_{0\le \alpha\le K_\alpha}\;\;\min_{\tau_g\ge 0}\;\;  \quad D(\alpha,\beta, \gamma,\tau_h,\tau_g)\,\quad\text{where}
 \end{align}
 \begin{align}\label{eq019-22}
 D(\alpha,\beta, \gamma,\tau_h,\tau_g):=&\frac{\delta\beta}{2(\tau_g+\beta)} \left(\alpha^2+\sigma^2\right)\nn\\
 &+\delta \mathbb{1}_{\big\{\gamma(\tau_g+\beta)>\sqrt{\frac{2}{\pi}}\delta\eps\beta\sqrt{\alpha^2+\sigma^2}\big\}}\frac{\beta^2(\alpha^2+\sigma^2)}{2\tau_g(\tau_g+\beta)}\left(\erf\left(\frac{\tau_*}{\sqrt{2}}\right)-\frac{\gamma(\tau_g+\beta)}{\delta\eps\beta\sqrt{\alpha^2+\sigma^2}}\tau_*\right)\nn\\
 &-\frac{\alpha}{2\tau_h}(\gamma^2 +\beta^2)+ \gamma\sqrt{\frac{\alpha^2\beta^2}{\tau_h^2}+V^2} -\frac{\alpha\tau_h}{2}+\frac{\beta\tau_g}{2}
\end{align}
In particular a variety of conclusions can be derived based on the optimal solutions of the above optimization problem as we discuss in the next step. We note that while the expressions may look complicated we prove that this optimization problem is in fact convex in the minimization parameters $(\alpha,\tau_g)$ and concave in the maximization parameters $(\beta,\gamma,\tau_h)$ so that its optimal solutions can be easily derived via a simple low-dimensional gradient descent rather quickly and accurately. We also note that this proof is quite intricate and involved, so it is not possible to give an intuitive sketch of the arguments here. We refer to Section \ref{step3} for details.
\medskip

\noindent\textbf{Step IV: Completing the proof of the theorems (Sections \ref{funpf} and \ref{algpf}).}\\
Finally, we utilize the above scalar form to derive all of the different theorems and results stated in Section \ref{sec:main}. This is done by relating the quantities of interest in each theorem to the optimal solutions of \eqref{eq0192}. For instance, we show that $\lim_{n\to\infty}\frac{1}{p} \twonorm{\hth^\eps-\bth_0}^2 = \alpha_*^2$ with $\alpha_*$ the optimal solution over $\alpha$. These calculations/proofs are carried out in detail in Sections  \ref{funpf} and \ref{algpf}. Since each argument is different we do not provide a summary here and refer to the corresponding sections.
\section{Proofs}
\label{sec:proofs}
\subsection{Notations}
We define the data matrix $\mtx{X}\in\R^{n\times p}$ with the rows consisting of the training data features $\vct{x}_1, \vct{x}_2, \ldots, \vct{x}_n$. For a convex function $f:\reals^m\to \reals$, we denote its its Fenchel conjugate by $f^*(\y) = \sup_{\x} \y^T\x - f(\x)$. We also define the Moreau envelope function of $f$ at $\x$ with parameter $\tau$ as
\[
e_f(\x;\tau) \equiv \min_{\vct{v}} \frac{1}{2\tau} \twonorm{\x-\vct{v}}^2+f(\vct{v})\,.
\]

\subsection{Simplification of the loss}
\label{step1}
As discussed earlier in this section we wish to derive a closed form for the loss 
\begin{align}
\label{firstoptstep1}
\mathcal{L}(\vct{\theta}):=\underset{\twonorm{\vct{\delta}_i}\le \eps}{{\rm max}}\;\;  \frac{1}{2n} \sum_{i=1}^n \left(y_i - \<\x_i +\bdelta_i,\bth\>\right)^2
\end{align}
and in particular show that 
\begin{align}
\label{concstep1}
\mathcal{L}(\vct{\theta})=\frac{1}{2n} \sum_{i=1}^n \left(|y_i-\<\x_i,\bth\>|+ \eps \twonorm{\bth}\right)^2=\frac{1}{2n}\twonorm{\abs{\vct{y}-\mtx{X}\vct{\theta}}+\eps\twonorm{\vct{\theta}}}^2
\end{align}
To this aim first note that the maximization in \eqref{firstoptstep1} decouples over $i$ so that we can write
\begin{align*}
\mathcal{L}(\vct{\theta}):=\;\;  \frac{1}{2n} \sum_{i=1}^n \underset{\twonorm{\vct{\delta}_i}\le \eps}{{\rm max}}\left(y_i - \<\x_i +\bdelta_i,\bth\>\right)^2
\end{align*}
To continue further define $\widetilde{y}_i:= y_i-\<\x_i,\bth\>$. By expanding the square the optimization over $\bdelta_i$ can be rewritten in the form
\[
\underset{\twonorm{\vct{\delta}_i}\le \eps}{\min}\;\;  -\frac{1}{2}\widetilde{y}_i^2 + \widetilde{y}_i \<\bth,\bdelta_i\>- \frac{1}{2} \<\bth,\bdelta_i\>^2\,.
\]
Note that this is trust-region subproblem and $\bdelta_i$ is a solution if and only if $\twonorm{\bdelta_i}\le \eps$ and there exists $\lambda_i \ge0$ such that
\begin{enumerate}
\item $(-\bth\bth^\sT + \lambda_i \I)\bdelta_i = -\tilde{y}_i\bth$\,.
\item $-\bth\bth^\sT + \lambda_i \I \succeq \zero$ (or equivalently $\lambda_i \ge \twonorm{\bth}^2$)
\item $\lambda_i(\eps-\twonorm{\bdelta_i}) = 0$.
\end{enumerate}
Since by (2), $\lambda_i>0$, condition (3) reduces to $\twonorm{\bdelta_i} = \eps$. Also from (1), we have
\begin{align}
\bdelta_i &= -\widetilde{y}_i (-\bth\bth^\sT + \lambda_i \I)^{-1}\bth\nonumber\\
&= -\lambda_i^{-1}\widetilde{y}_i \left(\I + \frac{\bth\bth^\sT}{\lambda_i - \twonorm{\bth}^2}\right)\bth\nonumber\\
&= -\lambda_i^{-1}\widetilde{y}_i \bth\frac{\lambda_i}{\lambda_i - \twonorm{\bth}^2}\nonumber\\
&= -\widetilde{y}_i \bth\frac{1}{\lambda_i - \twonorm{\bth}^2}\,.\label{eq:delta-i}
\end{align}
Using the fact that $\twonorm{\bdelta_i} = \eps$ in the latter identity we thus conclude that $\lambda_i = (1/\eps) \twonorm{\bth}\abs{\widetilde{y}_i} + \twonorm{\bth}^2$. Substituting for $\lambda_i$ in~\eqref{eq:delta-i} we obtain
\[\bdelta_i =  - \frac{\widetilde{y}_i}{\abs{\widetilde{y}_i}} \frac{\bth\eps}{\twonorm{\bth} }=-\eps\sgn{ y_i-\<\x_i,\bth\>}\frac{\vct{\theta}}{\twonorm{\vct{\theta}}}\,.\]
Substituting the latter into \eqref{firstoptstep1} we arrive at \eqref{concstep1} to complete our simplification of the loss.

\subsection{Reduction to an auxiliary optimization problem via CGMT}
\label{step2}
We are interested in characterizing the properties of the optimal paramter $\hth^{\eps}$ and thus it shall be convenient to work with a scaled version of the loss \eqref{concstep1}. This scaling of course does not affect the optimal solution $\hth^{\eps}$. Thus hence forth we focus on the following objective
\begin{align}\label{eq:hth-eps2}
\hth^{\eps} = \arg\min_{\bth\in \reals^p}   \frac{1}{2p^2} \sum_{i=1}^n \left(\abs{y_i - \<\x_i,\bth\>} + \eps\twonorm{\bth}\right)^2\,.
\end{align}
To continue further it is convenient to consider a change of variable of the form $\vct{z}=\frac{1}{\sqrt{p}}(\vct{\theta}-\vct{\theta}_0)$ and note that 
\begin{align*}
y_i-\langle\vct{x}_i,\vct{\theta}\rangle=w_i+\langle\vct{x}_i,\vct{\theta}_0-\vct{\theta}\rangle=w_i-\sqrt{p}\langle\vct{x}_i,\vct{z}\rangle.
\end{align*}
Define $\ell(v;\vct{\theta}) := \frac{1}{2}\left(\abs{v}+\eps\twonorm{\vct{\theta}}\right)^2$ and note that with this change of variable we have that $\widehat{\vct{z}}^{\eps}=\frac{1}{\sqrt{p}}(\hth^{\eps}-\vct{\theta}_0)$ is given by
 \begin{align*}
 \hz^{\eps} = \arg\min_{\z} \frac{1}{p^2}\sum_{i=1}^n \ell \left(w_i - \sqrt{p}\<\vct{x}_i,\z\>;\vct{\theta}_0+\sqrt{p}\vct{z}\right).  
 \end{align*}
 Equivalently we can rewrite this optimization problem in the form
 \begin{align}
 \label{tmp11}
 \min_{\z\in\R^p,\vct{v}\in\R^n} \frac{1}{p^2}\sum_{i=1}^n \ell\left(\sqrt{p}v_i;\vct{\theta}_0+\sqrt{p}\vct{z}\right)\quad\text{subject to}\quad\sqrt{p}\vct{v} = \w- \sqrt{p}\bX \z.
 \end{align}
 We note that the scaling of $\vct{v}$ is arbitrary but serves the purpose of simplifying the exposition later on. The loss above is still rather complicated and it is unclear how to study and characterize the properties of its optimal solution in an asymptotic regime where the size of the training data and the number of parameters grow in proportion with each other. To study this loss in an asymptotic fashion we first cast it as a different mini-max optimization using duality. In particular by associating a dual variable $\frac{\vct{u}}{p}$ with the equality constraint, we obtain
 \begin{align}
 \label{lin}
 &\min_{\z\in\R^p,\vct{v}\in\R^n}  \max_{\vct{u}\in\R^n}\; \frac{1}{p}\Big\{\vct{u}^T(\sqrt{p}\bX) \vct{z} - \vct{u}^T \w +\sqrt{p}\vct{u}^T \vct{v}\Big\}+\frac{1}{p^2}\sum_{i=1}^n \ell\left(\sqrt{p}v_i;\vct{\theta}_0+\sqrt{p}\vct{z}\right)\nn\\
 &\quad\quad\quad\quad\quad\quad\quad\quad\quad\quad\quad\quad\quad\quad\quad= \min_{\z\in\R^p,\vct{v}\in\R^n}  \max_{\vct{u}\in\R^n}\; \frac{1}{p}\Big\{\vct{u}^T(\sqrt{p}\bX) \vct{z} - \vct{u}^T \w +\sqrt{p}\vct{u}^T \vct{v}\Big\}+\ell(\vct{v};\vct{z})
 \end{align}
 where
 \begin{align*}
\ell(\vct{v};\vct{z}):=\frac{1}{p^2}\sum_{i=1}^n \ell\left(\sqrt{p}v_i;\vct{\theta}_0+\sqrt{p}\vct{z}\right)= \frac{1}{2p}\left(\twonorm{\vct{v}}^2+2\frac{\eps}{\sqrt{p}}\onenorm{\vct{v}}\twonorm{\vct{\theta}_0+\sqrt{p}\vct{z}}+\frac{\eps^2}{p}\twonorm{\vct{\theta}_0+\sqrt{p}\vct{z}}^2\right)
 \end{align*}
 At first this may be counter-intuitive as we started by simplifying a different mini-max optimization problem and now we are again introducing a new maximization! The main advantage of this new form is that \eqref{lin} is in fact affine in the matrix. This particular form allows us to use a powerful extension of a classical Gaussian process inequality due to Gordon \cite{gordon1988milman} known as Convex Gaussian Minimax Theorem (CGMT) \cite{thrampoulidis2015regularized} which focuses on characterizing the asymptotic behavior of mini-max optimization problems that are affine in a Gaussian matrix $\bX$. Formally, the CGMT framework shows that a problem of the form
 \begin{align}
 \label{generalPO}
 \min_{\vct{z}\in\mathcal{S}_{\vct{z}}}\text{ }\max_{\vct{u}\in\mathcal{S}_{\vct{u}}}\quad \vct{u}^T\mtx{X}\vct{z}+\psi(\vct{z},\vct{u})
 \end{align}
 with $\mtx{X}$ a matrix with $\mathcal{N}(0,1)$ entries can be replaced asymptotically with
 \begin{align}
 \label{generalAO}
 \min_{\vct{z}\in\mathcal{S}_{\vct{z}}}\text{ }\max_{\vct{u}\in\mathcal{S}_{\vct{u}}}\quad\twonorm{\vct{z}}\vct{g}^T\vct{u}+\twonorm{\vct{u}}\vct{h}^T\vct{z}+\psi(\vct{z},\vct{u})
 \end{align}
 where $\vct{g}$ and $\vct{h}$ are independent Gaussian vectors with i.i.d.~$\mathcal{N}(0,1)$ entries and $\psi(\vct{z},\vct{u})$ is convex in $\vct{z}$ and concave in $\vct{u}$. In the above $\mathcal{S}_{\vct{z}}$ and $\mathcal{S}_{\vct{u}}$ are compact sets. We refer to \cite[Theorem 3]{thrampoulidis2015regularized} for precise statements. Following \cite{thrampoulidis2015regularized} we shall refer to problems of the form \eqref{generalPO} and \eqref{generalAO} as the Primal Problem (PO) and the Auxiliary Problem (AO).

As evident from the above to be able to apply CGMT, requires the minimization/maximization to be over compact sets. To avoid this technical issue one can introduce ``artificial" boundedness constraint so that they do not change the optimal solution. More specifically, we can add constraints of the form $\cS_{\z} = \{\z|\;\; \twonorm{\vct{z}}\le K_\alpha\}$ and $\cS_{\vct{u}} = \{\vct{u}:\,\twonorm{\vct{u}}\le K_\beta \}$ for sufficiently large constants $K_\alpha$ and $K_\beta$ without changing the optimal solution of \eqref{lin} in a precise asymptotic sense. See Appendix \ref{setres} for precise statements and proofs. This allows us to replace \eqref{lin} with
 \begin{align}
 \label{linmod}
\min_{\z\in\mathcal{S}_{\vct{z}},\vct{v}\in\R^n}  \max_{\vct{u}\in\mathcal{S}_{\vct{u}}}\;\;\; \frac{1}{\sqrt{p}}\vct{u}^T\bX \vct{z} - \frac{1}{\sqrt{p}}\vct{u}^T \vct{\omega} +\frac{1}{\sqrt{p}}\vct{u}^T \vct{v}+\ell(\vct{v};\vct{z}),
 \end{align}
 where $\vct{\omega}=\frac{\w}{\sqrt{p}}\in\R^n$ is a Gaussian vector with i.i.d.~$\mathcal{N}(0,\sigma^2)$ entires.
 
 With these compact constraints in place we can now apply the CGMT result. To this aim note that this optimization is in the desired form of a Primary Optimization (PO): it has a bilinear term $\vct{u}^T\bX \z$ plus a function 
 \begin{align*}
 \psi(\z,\vct{u}) =\min_{\vct{v}\in\R^n}\frac{1}{\sqrt{p}} \left(-\vct{u}^T\vct{\omega}+\vct{u}^T\vct{v}\right) + \ell(\vct{v};\vct{z})
 \end{align*}
 which is convex in $\vct{z}$\footnote{Note that the prior to the minimization over $\vct{v}$ the problem is trivially jointly convex in $(\vct{z},\vct{v})$ and partial minimization preserves convexity.} and concave in $\vct{u}$. The corresponding Auxiliary Optimization (AO) thus takes the form
  \begin{align}
  \label{finalAO}
 &\min_{\vct{z}\in\cS_{\z}} \max_{\vct{u}\in\cS_{\vct{u}}}\; \frac{1}{\sqrt{p}}\left(\twonorm{\vct{z}} \vct{g}^T\vct{u} + \twonorm{\vct{u}} \vct{h}^T\vct{z}\right)+\min_{\vct{v}\in\R^n}\frac{1}{\sqrt{p}}\left( - \vct{u}^T \vct{\omega} +\vct{u}^T \vct{v}\right)+\ell(\vct{v};\vct{z})\nonumber\\
 &\quad\quad\quad\quad\quad\quad\quad\quad\quad\quad\quad\quad=\min_{\vct{z}\in\cS_{\z}, \vct{v}} \max_{\vct{u}\in\cS_{\vct{u}}}\; \frac{1}{\sqrt{p}}\left(\twonorm{\vct{z}} \vct{g}^T\vct{u} + \twonorm{\vct{u}} \vct{h}^T\vct{z}- \vct{u}^T \vct{\omega} +\vct{u}^T \vct{v}\right)+\ell(\vct{v};\vct{z}).
 \end{align}
 This completes the derivation of the AO. 
 \subsection{Scalarization of the auxilary optimization problem}
  \label{step3}
 In this section we continue our proof by significantly simplifying the AO problem. In particular we show that the behavior of the AO and hence the PO can be completely characterized by \eqref{eq019-2}. This is arguably the most intricate part of our proofs.
 
 We begin simplifying the AO by maximizing over $\vct{u}$. To this aim we decompose the optimization problem over $\mathcal{S}_{\vct{u}}$ in terms of its direction and radius. Specifically, $\vct{u}=\beta\widetilde{\vct{u}}$ with $\widetilde{\vct{u}}\in\mathbb{S}^{n-1}$ and $0\le \beta\le K_\beta$. Using this decomposition we have
 \begin{align*}
 &\max_{\vct{u}\in\cS_{\vct{u}}}\; \frac{1}{\sqrt{p}}\left(\twonorm{\vct{z}} \vct{g}^T\vct{u} + \twonorm{\vct{u}} \vct{h}^T\vct{z}- \vct{u}^T \vct{\omega} +\vct{u}^T \vct{v}\right)\\
 &\quad\quad\quad\quad\quad\quad\quad\quad\quad\quad\quad\quad\quad\quad= \max_{0\le \beta\le K_\beta}\text{ }\max_{\vct{u}\in\mathbb{S}^{n-1}}\; \frac{1}{\sqrt{p}}\left(\twonorm{\vct{z}} \vct{g}^T\vct{u} + \twonorm{\vct{u}} \vct{h}^T\vct{z}- \vct{u}^T \vct{\omega} +\vct{u}^T \vct{v}\right)\\
  &\quad\quad\quad\quad\quad\quad\quad\quad\quad\quad\quad\quad\quad\quad= \max_{0\le \beta\le K_\beta}\text{ }\max_{\vct{u}\in\mathbb{S}^{n-1}}\; \frac{1}{\sqrt{p}}\vct{u}^T\left(\twonorm{\vct{z}} \vct{g} - \vct{\omega} + \vct{v}\right)+\frac{\beta}{\sqrt{p}} \vct{h}^T\vct{z}\\
  &\quad\quad\quad\quad\quad\quad\quad\quad\quad\quad\quad\quad\quad\quad= \max_{0\le \beta\le K_\beta}\text{ } \frac{1}{\sqrt{p}}\twonorm{\twonorm{\vct{z}} \vct{g} - \vct{\omega} + \vct{v}}+\frac{\beta}{\sqrt{p}} \vct{h}^T\vct{z}.
 \end{align*}
 Plugging the latter into \eqref{finalAO} the AO reduces to 
 \begin{align*}
 \min_{\vct{z}\in\cS_{\z}, \vct{v}} \max_{0\le \beta\le K_\beta}\text{ } \frac{1}{\sqrt{p}}\twonorm{\twonorm{\vct{z}} \vct{g} - \vct{\omega} + \vct{v}}+\frac{\beta}{\sqrt{p}} \vct{h}^T\vct{z}+\ell(\vct{v};\vct{z}).
\end{align*}
 We hope to eventually simplify the minimization over $\vct{v}$ and $\vct{z}$ also. For this minimization to become easier in our later calculation we proceed by writing $ \ell(\vct{v};\vct{z})$ in terms of its conjugate with respect to $\vct{z}$. That is,
 \begin{align*}
 \ell(\vct{v};\vct{z})=\sup_{\vct{q}} \vct{q}^T\vct{z}-\widetilde{\ell}(\vct{v};\vct{q})
 \end{align*}
 where $ \widetilde{\ell}(\vct{v};\vct{q})$ is the conjugate of $\ell$ with respect to $\vct{z}$. The logic behind this is that AO with then simplify to
  \begin{align}
  \label{simpAO}
 \min_{\vct{z}\in\cS_{\vct{z}},\vct{v}}\max_{0\le\beta\le K_\beta, \vct{q}}\;\; \frac{\beta}{\sqrt{p}} \twonorm{\twonorm{\vct{z}}\vct{g}-\vct{\omega}+\vct{v}} +\frac{\beta}{\sqrt{p}} \vct{h}^T \vct{z} +\vct{q}^T\vct{z}-\widetilde{\ell}(\vct{v};\vct{q}).
\end{align}
To proceed it would be convenient to flip the order of minimum and maximum in the above. However, for this to be allowed the mini-max problem typically has to be convex/concave in the min/max parameters (e.g.~via the celebrated Sion's min-max Theorem \cite{sion1958general}). It is not clear that the above objective has this form so that the flipping of the order of the min and max is justified. However, since the original PO problem is convex/concave in the min/max parameters one can justify such a flipping of the min and max in the AO based on the PO. We note that this is justified for asymptotic calculations and refer to \cite[Appendix A.2.4]{thrampoulidis2015precise} for precise details on this derivation. Thus, we will instead consider the following problem as the (AO) which is asymptotically equivalent to \eqref{simpAO}
  \[
\max_{0\le\beta\le K_\beta, \vct{q}} \min_{\vct{z}\in\cS_{\vct{z}},\vct{v}}\;\; \frac{\beta}{\sqrt{p}} \twonorm{\twonorm{\vct{z}}\vct{g}-\vct{\omega}+\vct{v}} +\frac{\beta}{\sqrt{p}} \vct{h}^T \vct{z} +\vct{q}^T\vct{z}-\widetilde{\ell}(\vct{v};\vct{q}).
\]
To simplify further we now optimize over the direction and norm of $\vct{z}$ ($\twonorm{\vct{z}}=\alpha$) to arrive at
  \begin{align}
  \label{usethm4}
\max_{0\le\beta\le K_\beta, \vct{q}} \min_{0\le \alpha\le K_\alpha,\vct{v}}\;\; \frac{\beta}{\sqrt{p}} \twonorm{\alpha\vct{g}-\vct{\omega}+\vct{v}} -\alpha\twonorm{\frac{\beta}{\sqrt{p}} \vct{h} +\vct{q}}-\widetilde{\ell}(\vct{v};\vct{q}).
\end{align}
Next note that $-\widetilde{\ell}(\vct{v};\vct{q})$ is convex in $\vct{v}$. To see this first note that
\begin{align*}
\widetilde{\ell}(\vct{v};\vct{q})=\sup_{\vct{z}} \vct{q}^T\vct{z}-\ell(\vct{v};\vct{z}).
\end{align*}
Also since $\ell$ is jointly convex in $(\vct{v},\vct{z})$, then $-\ell(\vct{v};\vct{z})$ is jointly concave in $(\vct{v},\vct{z})$. Also $\vct{q}^T\vct{z}$ is jointly concave in $(\vct{v},\vct{z})$. Therefore,  $\vct{q}^T\vct{z}-\ell(\vct{v};\vct{z})$ is jointly concave in $(\vct{v},\vct{z})$ and based on the partial maximization rule we can conclude that $\widetilde{\ell}(\vct{v};\vct{q})$ should be concave in $\vct{v}$ which in turn implies $-\widetilde{\ell}(\vct{v};\vct{q})$ is convex in $\vct{v}$. The other terms are also trivially jointly convex in $\alpha, \vct{v}$ so that overall the objective is jointly convex in $\alpha, \vct{v}$. The objective above is also trivially jointly concave in $\beta, \vct{q}$. Thus based on Sion's min-max Theorem \cite{sion1958general}) we could change the order of the mins and maxs as we please. This allows us to reorder $\max_{\vct{q}}$ and $\min_{\alpha, \bv}$ to arrive at
\begin{align}
\max_{0\le\beta\le K_\beta} \min_{0\le \alpha\le K_\alpha,\vct{v}}\max_{\vct{q}} \;\; \frac{\beta}{\sqrt{p}} \twonorm{\alpha\vct{g}-\vct{\omega}+\vct{v}} -\alpha\twonorm{\frac{\beta}{\sqrt{p}} \vct{h} +\vct{q}}-\widetilde{\ell}(\vct{v};\vct{q})
\end{align}
To proceed, we first compute $\widetilde{\ell}(\vct{v};\vct{q})$ in the Lemma below with the proof deferred to Appendix \ref{conjlemmapf}.
\begin{lemma}\label{conjlemma} The conjugate of 
\begin{align*}
\ell(\vct{v};\vct{z}):=\frac{1}{2p}\left(\twonorm{\vct{v}}^2+2\frac{\eps}{\sqrt{p}}\onenorm{\vct{v}}\twonorm{\vct{\theta}_0+\sqrt{p}\vct{z}}+\frac{\eps^2}{p}\twonorm{\vct{\theta}_0+\sqrt{p}\vct{z}}^2\right)
 \end{align*}
 with respect to the variable $\vct{z}$ is given by
\begin{align*}
\widetilde{\ell}(\vct{v};\vct{q}):=\sup_{\vct{z}} \vct{q}^T\vct{z}-\ell(\vct{v};\vct{z})=-\frac{1}{\sqrt{p}}\vct{q}^T\vct{\theta}_0+\frac{1}{2\delta p^2}\left(\frac{p}{\eps}\tn{\vct{q}}-\onenorm{\vct{v}}\right)_+^2 -\frac{1}{2p}\twonorm{\bv}^2.
\end{align*}
\end{lemma}
Using this characterization of $\widetilde{\ell}(\bv;\vct{q})$ we arrive at the following representation of AO problem
\begin{align}
\label{conjtemp}
 \min_{0\le \alpha\le K_\alpha,\vct{v}}\max_{0\le\beta\le K_\beta}\max_{\vct{q}}\;\; &\frac{\beta}{\sqrt{p}} \twonorm{\alpha\vct{g}-\vct{\omega}+\vct{v}} -\alpha\twonorm{\frac{\beta}{\sqrt{p}} \vct{h} +\vct{q}}+ \frac{1}{\sqrt{p}}\vct{q}^T\vct{\theta}_0\nonumber\\
 &\quad\quad\quad\quad\quad\quad\quad\quad\quad\quad-\frac{1}{2\delta p^2}\left(\frac{p}{\eps}\tn{\vct{q}}-\onenorm{\vct{v}}\right)_+^2 + \frac{1}{2p}\twonorm{\bv}^2
\end{align}
To simplify further we next focus on the maximization over $\vct{q}$ or equivalently the following minimization problem
\begin{align*}
&\min_{\vct{q}}\quad \alpha\twonorm{\frac{\beta}{\sqrt{p}} \vct{h} +\vct{q}} +\frac{1}{2\delta p^2}\left(\frac{p}{\eps}\tn{\vct{q}}-\onenorm{\vct{v}}\right)_+^2 - \frac{1}{\sqrt{p}}\vct{q}^T\vct{\theta}_0\\
&\min_{\vct{q}}\text{ }\inf_{\tau_h\ge 0}\text{ } \frac{\alpha}{2\tau_h}\twonorm{\frac{\beta}{\sqrt{p}} \vct{h} +\vct{q}}^2 +\frac{\alpha\tau_h}{2}+
\frac{1}{2\delta p^2}\left(\frac{p}{\eps}\tn{\vct{q}}-\onenorm{\vct{v}}\right)_+^2 - \frac{1}{\sqrt{p}}\vct{q}^T\vct{\theta}_0\\
&\min_{\vct{q}}\text{ }\inf_{\tau_h\ge 0}\text{ } \frac{\alpha}{2\tau_h}\tn{\vct{q}}^2+\frac{\alpha\beta^2}{2p\tau_h}\tn{\vct{h}}^2+\frac{\alpha\beta}{\tau_h\sqrt{p}}\vct{h}^T\vct{q} +\frac{\alpha\tau_h}{2}
+\frac{1}{2\delta p^2}\left(\frac{p}{\eps}\tn{\vct{q}}-\onenorm{\vct{v}}\right)_+^2 - \frac{1}{\sqrt{p}}\vct{q}^T\vct{\theta}_0
\end{align*}
The above is a linear function of $\vct{q}$ plus a term depending on $\twonorm{\vct{q}}$. So fixing $\twonorm{\vct{q}} = \gamma\ge 0$ the optimal $\vct{q}$ is given by $\vct{q} =- \gamma \frac{\frac{\alpha\beta}{\tau_h}\vct{h}-\vct{\theta_0}}{\twonorm{\frac{\alpha\beta}{\tau_h}\vct{h}-\vct{\theta_0}}}$, which simplifies the above to
\begin{align*}
\inf_{\tau_h,\gamma\ge0}\quad  \frac{\alpha}{2\tau_h}\gamma^2+\frac{\alpha\beta^2}{2p\tau_h}\tn{\vct{h}}^2
- \frac{\gamma}{\sqrt{p}} \twonorm{\frac{\alpha\beta}{\tau_h}\vct{h}-\vct{\theta_0}}
 +\frac{\alpha\tau_h}{2}
+\frac{1}{2\delta p^2}\left(\frac{p}{\eps}\gamma-\onenorm{\vct{v}}\right)_+^2
\end{align*}
%
%
%
Plugging the latter into \eqref{conjtemp} the AO reduces to
\begin{align}\label{conj2}
 \min_{0\le \alpha\le K_\alpha,\vct{v}} \max_{0\le\beta\le K_\beta}\sup_{\gamma,\tau_h\ge0}\;\; &\frac{\beta}{\sqrt{p}} \twonorm{\alpha\vct{g}-\vct{\omega}+\vct{v}}+\frac{1}{2p}\twonorm{\vct{v}}^2\nn\\
 &-\frac{\alpha}{2\tau_h}\gamma^2-\frac{\alpha\beta^2}{2p\tau_h}\tn{\vct{h}}^2+ \frac{\gamma}{\sqrt{p}} \twonorm{\frac{\alpha\beta}{\tau_h}\vct{h}-\vct{\theta_0}} -\frac{\alpha\tau_h}{2}
 -\frac{1}{2\delta p^2}\left(\frac{p}{\eps}\gamma-\onenorm{\vct{v}}\right)_+^2
\end{align}
%
To continue we state a lemma with the proof deferred to Appendix \ref{cvxconcavelempf} 
\begin{lemma}
\label{cvxconcavelem}
The function
\begin{align*}
f(\gamma,\beta,\tau_h):=\gamma^2+\frac{\beta^2}{p}\tn{\vct{h}}^2- 2\frac{\gamma}{\sqrt{p}} \twonorm{\beta\vct{h}-\frac{\vct{\theta_0}}{\alpha}}
\end{align*}
is jointly convex in the parameters $(\gamma,\beta,\tau_h)$.
\end{lemma}
Using this lemma we can trivially conclude that the objective \eqref{conj2} is jointly concave in $(\gamma,\beta,\tau_h)$. Also note that $\widetilde{\ell}$ is concave in $\vct{v}$ and hence $-\widetilde{\ell}$ is convex in $\vct{v}$. This implies that the objective in \eqref{conjtemp} is jointly convex in $(\alpha,\vct{v})$. Since maximization (with respect to the direction of $\vct{q}$) preserves convexity therefore \eqref{conj2} is trivially jointly convex in $(\alpha,\vct{v})$.  Therefore, we can flip the order of min and max in \eqref{conj2} (again using Sion's min-max Theorem) to arrive at
\begin{align}
\label{eq09}
 \max_{0\le\beta\le K_\beta}\sup_{\gamma,\tau_h\ge0}\;\;\min_{0\le \alpha\le K_\alpha}\;\;\min_{\vct{v}}\;\;  &\frac{\beta}{\sqrt{p}} \twonorm{\alpha\vct{g}-\vct{\omega}+\vct{v}}+\frac{1}{2p}\twonorm{\vct{v}}^2\nn\\
 &-\frac{\alpha}{2\tau_h}\gamma^2-\frac{\alpha\beta^2}{2p\tau_h}\tn{\vct{h}}^2+ \frac{\gamma}{\sqrt{p}} \twonorm{\frac{\alpha\beta}{\tau_h}\vct{h}-\vct{\theta_0}} -\frac{\alpha\tau_h}{2}
 -\frac{1}{2\delta p^2}\left(\frac{p}{\eps}\gamma-\onenorm{\vct{v}}\right)_+^2
\end{align}
We now focus on minimization over $\bv$. To this aim note that
\begin{align}
\label{eq10}
\min_{\vct{v}}\;\;  &\frac{\beta}{\sqrt{p}} \twonorm{\alpha\vct{g}-\vct{\omega}+\vct{v}}+\frac{1}{2p}\twonorm{\vct{v}}^2 -\frac{1}{2\delta p^2}\left(\frac{p}{\eps}\gamma-\onenorm{\vct{v}}\right)_+^2\nn\\
\min_{\tau_g\ge0,\vct{v}}\;\;  &\frac{\beta}{2\tau_g p} \twonorm{\alpha\vct{g}-\vct{\omega}+\vct{v}}^2 +\frac{\beta\tau_g}{2} +\frac{1}{2p}\twonorm{\vct{v}}^2 -\frac{1}{2\delta p^2}\left(\frac{p}{\eps}\gamma-\onenorm{\vct{v}}\right)_+^2\nn\\
\min_{\tau_g\ge0,\vct{v}}\;\;  &\frac{\beta}{2\tau_g p} \twonorm{\alpha\vct{g}-\vct{\omega}+\vct{v}}^2 +\frac{\beta\tau_g}{2} +\frac{1}{2p}\twonorm{\vct{v}}^2 -\frac{1}{2\delta p^2}\left(\frac{p}{\eps}\gamma-\onenorm{\vct{v}}\right)_+^2
\end{align}
Recall the definition of the Moreau envelope function of a function $f$ at a point $\vct{x}$ with parameter $\mu$,
\[
e_f(\vct{x};\mu)\equiv \min_{\bv} \frac{1}{2\mu} \twonorm{\vct{x}-\bv}^2+ f(\bv)\,.
\]
and define
\begin{align}\label{eq:f}
f(\bv;\gamma)\equiv \frac{1}{2} \twonorm{\bv}^2 -\frac{1}{2\delta p} (\frac{p}{\eps}\gamma-\onenorm{\bv})_+^2\,,
\end{align}
Note that $f(\bv;\gamma)$ is convex in $\bv$ (since $-\widetilde{\ell}(\bv;\vct{q})$ was convex in $\bv$). Thus, \eqref{eq10} can be rewritten in the more compact form
\begin{align}\label{eq:dum0}
\min_{\tau_g\ge0}\;\;  &\frac{1}{p} e_f\left(\vct{\omega}-\alpha\vct{g};\frac{\tau_g}{\beta}\right) + \frac{\beta\tau_g}{2}
  \end{align}
In our next lemma we compute $e_f$. We defer the proof to Appendix \ref{meenvpf}.

\begin{lemma}\label{meenv}
Consider the function $f$ given by~\eqref{eq:f}. Then,
\begin{align*}
e_f(\vct{x};\mu) &= \frac{1}{2(\mu+1)}\twonorm{\bx}^2 +
\min_{\tau\ge 0} G(\bx;\mu,\tau)
\end{align*} 
where
\begin{align*}
&G(\bx;\mu,\tau) = \frac{1}{2\mu(\mu+1)}\twonorm{\vct{x}-\ST(\bx;\tau)}^2-\frac{1}{2n}\left(\frac{p}{\eps}\gamma-\frac{1}{1+\mu}\onenorm{\ST(\bx;\tau)}\right)_+^2.
\end{align*}
Furthermore, $e_f(\bx;\tau)$ is strictly convex in $\bx$.
\end{lemma}
Plugging Lemma \ref{meenv} into \eqref{eq:dum0} we have
\begin{align*}
&\frac{1}{p} e_f\left(\alpha\vct{g}-\vct{\omega};\frac{\tau_g}{\beta}\right)+ \frac{\beta\tau_g}{2}  
=\frac{\beta}{2(\tau_g+\beta)} \frac{1}{p}\twonorm{\alpha\vct{g}-\vct{\omega}}^2 + \min_{\tau\ge0} \frac{1}{p} G\left(\alpha\vct{g}-\vct{\omega};\frac{\tau_g}{\beta},\tau\right)+\frac{\beta\tau_g}{2}
\end{align*}
Plugging this in \eqref{eq09} the AO problem reduces to
\begin{align}
\label{6some}
 \max_{0\le\beta\le K_\beta}\sup_{\gamma,\tau_h\ge0}\;\;\min_{0\le \alpha\le K_\alpha}\;\;\min_{\tau_g\ge 0}\;\;\min_{\tau\ge 0}\;\;  \quad&\frac{\beta}{2(\tau_g+\beta)} \frac{1}{p}\twonorm{\alpha\vct{g}-\vct{\omega}}^2 + \frac{1}{p} G\left(\alpha\vct{g}-\vct{\omega};\frac{\tau_g}{\beta},\tau\right)+\frac{\beta\tau_g}{2}\nn\\
 &-\frac{\alpha}{2\tau_h}\gamma^2-\frac{\alpha\beta^2}{2p\tau_h}\tn{\vct{h}}^2+ \frac{\gamma}{\sqrt{p}} \twonorm{\frac{\alpha\beta}{\tau_h}\vct{h}-\vct{\theta_0}} -\frac{\alpha\tau_h}{2}
\end{align}
We note that since the problem \eqref{eq10} was jointly convex in $(\vct{v},\alpha,\tau_g)$ and \eqref{eq09} jointly concave in $(\beta,\gamma,\tau_h)$ and partial minimization preserves convexity we thus conclude that the objective is jointly convex in $(\alpha,\tau_g)$ and jointly concave in $(\beta,\gamma,\tau_h)$ (after the minimization over $\tau\ge 0$ has been carried out). 
Note that trivially in an asymptotic regime
\begin{align*}
\frac{\twonorm{\vct{h}}^2}{p} \rightarrow 1\quad\text{and}\quad \frac{\twonorm{\alpha\vct{g}-\vct{\omega}}^2 }{n}\rightarrow \left(\alpha^2+\sigma^2\right)
\end{align*}
Also using concentration of Lipschitz functions of Gaussian we have
\begin{align*}
\frac{1}{\sqrt{p}} \twonorm{\frac{\alpha\beta}{\tau_h}\vct{h}-\vct{\theta_0}}\rightarrow \frac{1}{\sqrt{p}}\sqrt{\E\bigg[\twonorm{\frac{\alpha\beta}{\tau_h}\vct{h}-\vct{\theta_0}}^2\bigg]}\rightarrow\sqrt{\frac{\alpha^2\beta^2}{\tau_h^2}+V^2}.
\end{align*}
Plugging all of the above in \eqref{6some} we arrive at
\begin{align}
\label{eq018}
 \max_{0\le\beta\le K_\beta}\sup_{\gamma,\tau_h\ge0}\;\;\min_{0\le \alpha\le K_\alpha}\;\;\min_{\tau_g\ge 0}\;\;\min_{\tau\ge 0}\;\;  \quad&\frac{\delta\beta}{2(\tau_g+\beta)} \left(\alpha^2+\sigma^2\right) + \frac{1}{p} G\left(\alpha\vct{g}-\vct{\omega};\frac{\tau_g}{\beta},\tau\right)+\frac{\beta\tau_g}{2}\nn\\
 &-\frac{\alpha}{2\tau_h}\gamma^2-\frac{\alpha\beta^2}{2\tau_h}+ \gamma\sqrt{\frac{\alpha^2\beta^2}{\tau_h^2}+V^2} -\frac{\alpha\tau_h}{2}
\end{align}
To simplify further we also need an asymptotic characterization of $\frac{1}{p} G\left(\alpha\vct{g}-\vct{\omega};\frac{\tau_g}{\beta},\tau\right)$. To this aim we prove the following lemma with the proof deferred to Appendix \ref{Glempf}.
\begin{lemma}\label{Glem} Let $\vct{w}\in\R^n$ be a Gaussian random vector distributed as $\mathcal{N}(\vct{0},\omega^2\mtx{I}_n)$. Also assume
\begin{align*}
G(\vct{w};\mu,\tau) := \frac{1}{2\mu(\mu+1)}\twonorm{\vct{w}-\ST(\vct{w};\tau)}^2-\frac{1}{2n}\left(\frac{p}{\eps}\gamma-\frac{1}{1+\mu}\onenorm{\ST(\vct{w};\tau)}\right)_+^2.
\end{align*}
Then
\begin{align*}
\underset{n \rightarrow \infty}{\lim}\text{ }\frac{1}{n}G(\vct{w};\mu,\tau)=&\frac{\omega^2}{2\mu(\mu+1)}\left(\left(1-\sqrt{\frac{2}{\pi}}\frac{\tau}{\omega} e^{-\frac{\tau^2}{2\omega^2}}\right)+\left(\frac{\tau^2}{\omega^2}-1\right)\erfc\left(\frac{1}{\sqrt{2}}\frac{\tau}{\omega}\right)\right)\\
&-\frac{\omega^2}{2(\mu+1)^2}\left(\frac{\gamma(\mu+1)}{\delta\eps\omega}+\frac{\tau}{\omega}\cdot\erfc\left(\frac{1}{\sqrt{2}}\frac{\tau}{\omega}\right)-\sqrt{\frac{2}{\pi}} e^{-\frac{\tau^2}{2\omega^2}}\right)_+^2.
\end{align*}
Furthermore, 
\begin{align*}
\underset{\tau\ge 0}{\min}\text{ }\underset{n \rightarrow \infty}{\lim}\text{ }&\frac{1}{n}G(\vct{w};\mu,\tau)\\
&\quad\quad=
\begin{cases}
0 \quad &\text{ if }\gamma(\mu+1)\le \sqrt{\frac{2}{\pi}}\delta \eps\omega\\
\frac{\omega^2}{2\mu(\mu+1)}\left(\erf\left(\frac{\tau^*\left(\frac{\gamma(\mu+1)}{\delta\eps\omega},\mu\right)}{\sqrt{2}}\right)-\frac{\gamma(\mu+1)}{\delta\eps\omega}\tau^*\left(\frac{\gamma(\mu+1)}{\delta\eps\omega},\mu\right)\right)&\text{ if }\gamma(\mu+1)>\sqrt{\frac{2}{\pi}}\delta \eps \omega
\end{cases}
\end{align*}
where $\tau^*(a,\mu)$ is the unique solution to 
\begin{align*}
a-\frac{1}{\mu}\tau-\tau\cdot\erf\left(\frac{\tau}{\sqrt{2}}\right)-\sqrt{\frac{2}{\pi}} e^{-\frac{\tau^2}{2}}=0
\end{align*}
\end{lemma}
Plugging the above lemma in \eqref{eq018} we arrive at
\begin{align}
\label{eq019}
 \max_{0\le\beta\le K_\beta}\sup_{\gamma,\tau_h\ge0}\;\;\min_{0\le \alpha\le K_\alpha}\;\;\min_{\tau_g\ge 0}\;\;  \quad D(\alpha,\beta, \gamma,\tau_h,\tau_g)\,,
 \end{align}
 where
 \begin{align}\label{eq019-2}
 D(\alpha,\beta, \gamma,\tau_h,\tau_g)=&\frac{\delta\beta}{2(\tau_g+\beta)} \left(\alpha^2+\sigma^2\right)\nn\\
 &+\delta \mathbb{1}_{\big\{\gamma(\tau_g+\beta)>\sqrt{\frac{2}{\pi}}\delta\eps\beta\sqrt{\alpha^2+\sigma^2}\big\}}\frac{\beta^2(\alpha^2+\sigma^2)}{2\tau_g(\tau_g+\beta)}\left(\erf\left(\frac{\tau_*}{\sqrt{2}}\right)-\frac{\gamma(\tau_g+\beta)}{\delta\eps\beta\sqrt{\alpha^2+\sigma^2}}\tau_*\right)\nn\\
 &-\frac{\alpha}{2\tau_h}(\gamma^2 +\beta^2)+ \gamma\sqrt{\frac{\alpha^2\beta^2}{\tau_h^2}+V^2} -\frac{\alpha\tau_h}{2}+\frac{\beta\tau_g}{2}
\end{align}
and $\tau_*$ is the unique solution to 
\begin{align}\label{eq:tau*}
\frac{\gamma(\tau_g+\beta)}{\delta\eps\beta\sqrt{\alpha^2+\sigma^2}}-\frac{\beta}{\tau_g}\tau-\tau\cdot \erf\left(\frac{\tau}{\sqrt{2}}\right)-\sqrt{\frac{2}{\pi}} e^{-\frac{\tau^2}{2}}=0
\end{align}
This completes the scalarization of the AO. 
\begin{remarks} {\bf (Convergence analysis).} In above we showed the point wise convergence of the objective function in~\eqref{6some} to function D given by~\eqref{eq019-2}. However, what is required in this framework, is  (local) uniform convergence so we get that the minimax solution of the objective function in~\eqref{6some} also converges to the minimax solution of the AO problem \eqref{eq019-2}. This can be shown by following similar arguments as in \cite[Lemma A.5]{thrampoulidis2018precise} that is essentially based on a result known as ``convexity lemma'' in the literature (see e.g. \cite[Lemma 7.75]{StatDecision}) by which point wise convergence of convex functions implies uniform convergence in compact subsets.
\end{remarks}
\subsection{Uniqueness of the solution of the AO problem}
As we discussed after Equation~\eqref{eq:f}, the function $f(\bv;\gamma)$ is convex in $\bv$. Furthermore, we wrote \eqref{eq10} (part of the objective that depends on $\bv$) in terms of the Moreau envelope $\frac{1}{p}e_f(\vct{\omega}-\alpha \vct{g}; \tfrac{\tau_g}{\beta})$ and as $n\to \infty$, its limit goes to the \emph{expected Moreau envelope}. Now by using the result of~\cite[Lemma 4.4]{thrampoulidis2018precise} the expected Moreau envelope of a convex function is \emph{strictly} convex ( without requiring any strong or strict convexity assumption
on the function itself). Therefore, the convexity-concavity property discussed after~\eqref{6some} is preserved after taking the limit and the AO objective $D(\alpha,\beta, \gamma,\tau_h,\tau_g)$ is jointly strictly convex in $(\alpha, \tau_g)$ and jointly concave in $(\beta, \gamma, \tau_h)$.

We next note that $\sup_{\beta, \gamma, \tau_h} D(\alpha, \beta,\gamma,\tau_h,\tau_g)$ is strictly convex in $(\alpha,\tau_g)$. This follows from the fact that if $f(\bx,\by)$ is strictly convex in $\bx$, then $\sup_{\by} f(\bx,\by)$ is also strictly convex in $\bx$. We next use \cite[Lemma C.5]{thrampoulidis2018precise} to conclude that $\inf_{\tau_g}\sup_{\beta, \gamma, \tau_h} D(\alpha, \beta,\gamma,\tau_h,\tau_g)$ is strictly convex in $\alpha >0$.Therefore, its minimizer over $\alpha\ge0$ is unique, which completes the proof.  

\subsection{Proofs for fundamental tradeoffs}
\label{funpf}
\subsubsection{Proof of Lemma \ref{lem:SR-AR}}
\label{lem:SR-ARpf}
We have
\begin{align}
\SR(\hth): = \frac{1}{p}\E\Big[(y-\<\bx,\hth\>)^2 \Big] = \frac{1}{p}\E\Big[(w-\<\bx,\hth-\bth_0\>)^2 \Big] = 
\frac{\sigma_0^2}{p} + \frac{1}{p} \twonorm{\hth-\bth_0}^2. 
\end{align}
To characterize $\AR(\hth)$, note that by following a similar argument as in Section~\ref{step1}, the solution of the problem
\[
\max_{\twonorm{\bdelta}\le \etest} (y-\<\x+\bdelta,\hth\>)^2
\]
is given by 
\[\bdelta_i = -\etest \sgn{ y-\<\x,\hth\>}\frac{\hth}{\twonorm{\hth}}\,.\]
Therefore the adversarial risk  can be written as
\begin{align}
\AR(\hth) = \frac{1}{p} \E \left[\left(|y-\<\x,\hth\>| +\etest \twonorm{\hth}\right)^2\right]
\end{align}
By substituting for $y = \<\x,\bth_0\>+ w$ and expanding the terms, we get 
\begin{align}
&\E \left[\left(|y-\<\x,\hth\>| +\etest \twonorm{\hth}\right)^2\right] \nonumber\\
& = \E[\<\x,\bth_0- \hth\>^2] + \E[w^2] + \etest^2\twonorm{\hth}^2 + 2\etest\twonorm{\hth} \E\left[|\<\x,\hth-\bth_0\>+w| \right]\nonumber\\
&= \twonorm{\hth-\bth_0}^2+ \sigma_0^2 +\etest^2\twonorm{\hth}^2+  2\sqrt{\frac{2}{\pi}} \etest\twonorm{\hth} \Big(\sigma_0^2+ \twonorm{\bth_0-\hth}^2 \Big)^{1/2} \,,
\end{align}
where in the first line we used the fact that $\hth$ is independent of $\x$ and $w$ (the test data and the corresponding response) and in the second line we used that $\<\x,\hth-\bth_0\>+w\sim \normal(0,\sigma^2+\twonorm{\hth-\bth_0}^2)$ since $\x\sim\normal(0,\mtx{I}_p)$. This completes the proof.

\subsubsection{Proof of Proposition \ref{pro:hth-Lam}}
\label{pro:hth-Lampf}
By definition,
\[
\bth^\lambda = \arg\min_{\bth} \;\; \lambda \SR(\bth) + \AR(\bth)
\]
Substituting for $\SR(\bth)$ and $\AR(\bth)$ from Lemma~\ref{lem:SR-ARpf} and scaling the objective by a factor $p$, we get
\[
\bth^\lambda = \arg\min_{\bth} \;\; (1+\lambda) \left(\sigma_0^2+\twonorm{\bth-\bth_0}^2\right) +\etest^2\twonorm{\bth}^2+  2\sqrt{\frac{2}{\pi}} \etest\twonorm{\bth} \left(\sigma_0^2+ \twonorm{\bth_0-\bth}^2 \right)^{1/2}
\]
Now by setting the derivative to zero we arrive at the following identity for $\bth^\lambda$:
\begin{align}
(1+\lambda)(\bth^\lambda-\bth_0) + \etest^2 \bth^\lambda +\sqrt{\frac{2}{\pi}}\etest \left(
\frac{\bth^\lambda}{\|\bth^\lambda\|} \left(\sigma_0^2+ \twonorm{\bth_0-\bth}^2 \right)^{1/2} + \frac{\|\bth^\lambda\|}{\left(\sigma_0^2+ \twonorm{\bth_0-\bth}^2 \right)^{1/2}} (\bth^\lambda-\bth_0)\right)=0\,.
\end{align} 
Adopting the shorthand 
$$A^\lambda := \frac{1}{\|\bth^\lambda\|}\left(\sigma_0^2+ \twonorm{\bth_0-\bth^\lambda}^2 \right)^{1/2}$$ 
and rearranging the terms term we get
\[
\left[\left(1+\lambda+\sqrt{\frac{2}{\pi}} \frac{\etest}{A^\lambda}\right)+ \left(\etest^2 + \sqrt{\frac{2}{\pi}} \etest A^\lambda\right) \I\right]\bth^\lambda
= \left(1+\lambda+\sqrt{\frac{2}{\pi}}\frac{\etest}{A} \right) \bth_0\,.
\]
The above equation can be written as 
\[
\bth^\lambda = (1+\gamma_0^\lambda)^{-1} \bth_0\,,
\]
with
\[
\gamma_0:= \frac{\etest^2 + \sqrt{\frac{2}{\pi}} \etest A^\lambda}{1+\lambda+\sqrt{\frac{2}{\pi}}\frac{\etest}{A}}\,,
\]
which is the desired claim. 
The proof is complete by noting that
\[
A^\lambda := \frac{1}{\twonorm{\bth^\lambda}}\left(\sigma_0^2+ \twonorm{\bth_0-\bth}^2 \right)^{1/2} 
= \frac{1}{\twonorm{\bth_0}} \left((1+\gamma_0^\lambda)^2\sigma_0^2+ (\gamma_0^\lambda)^2\twonorm{\bth_0}^2\right)^{1/2}.
\]
\subsection{Proofs for algorithmic tradeoffs}
\label{algpf}
\subsubsection{Proof of Theorem \ref{thm:main}}
  \label{Thm5pf}
We have already prove part (a) in the previous sections. Part (b) is also trivial from \eqref{eq019-2} as 
 \begin{align*}
 \lim_{n\to\infty}\frac{1}{p} \twonorm{\hth^\eps-\bth_0}^2= \lim_{n\to\infty} \twonorm{\widehat{z}^{\eps}}^2=\alpha_*^2.
 \end{align*}
 We thus turn our attention to part (c) and discuss how to calculate $\frac{\twonorm{\hth^\eps}}{\sqrt{p}}$ asymptotically. As discussed earlier using a change of variable of the form $\vct{\theta}=\vct{\theta}_0+\sqrt{p}\vct{z}$ the optimization problem can be written in the form
 \[
 \min_{\z\in\cS_{\z},\vct{v}} \max_{\vct{u}\in\cS_{\vct{u}}}\; \frac{1}{\sqrt{p}}\left(\vct{u}^T\bX \vct{z} - \vct{u}^T \vct{\omega} +\vct{u}^T \vct{v}\right)+\ell(\vct{v};\vct{z})
 \] 
 where 
 \begin{align*}
\ell(\vct{v};\vct{z}):= \frac{1}{2p}\left(\twonorm{\vct{v}}^2+2\frac{\eps}{\sqrt{p}}\onenorm{\vct{v}}\twonorm{\vct{\theta}_0+\sqrt{p}\vct{z}}+\frac{\eps^2}{p}\twonorm{\vct{\theta}_0+\sqrt{p}\vct{z}}^2\right)
 \end{align*}
As in the previous argument on calculating $\twonorm{\hth^\eps-\vct{\theta}_0}$ asymptotically via the AO we proceed by writing $ \ell(\vct{v};\vct{z})$ in terms of its conjugate with respect to $\vct{z}$. That is,
 \begin{align*}
 \ell(\vct{v};\vct{z})=\sup_{\vct{q}} \vct{q}^T\vct{z}-\widetilde{\ell}(\vct{v};\vct{q})
 \end{align*}
As discussed in Section \ref{step2} the conjugate function takes the form
\[
\widetilde{\ell}(\bv;\bq) = -\frac{1}{\sqrt{p}} \bq^T \bth_0 + \frac{1}{2\delta p^2} \left(\frac{p}{\eps} \twonorm{\bq} - \onenorm{\bv}\right)_+^2 -\frac{1}{2p} \twonorm{\bv}^2
\]
and the AO problem can therefore be written as (same as \eqref{usethm4})
\begin{align*}
 \min_{0\le \alpha\le K_\alpha,\vct{v}}\max_{0\le\beta\le K_\beta}\max_{\vct{q}}\;\; \frac{\beta}{\sqrt{p}} \twonorm{\alpha\vct{g}-\vct{\omega}+\vct{v}} -\alpha\twonorm{\frac{\beta}{\sqrt{p}} \vct{h} +\vct{q}}-\widetilde{\ell}(\bv;\bq)
\end{align*}
Our key observation is that the same AO can be used to calculate $\frac{\twonorm{\hth^\eps}}{\sqrt{p}}$. To make this precise we show how to write $\twonorm{\hth^\eps}$ in terms of functions of the $\bq$ and $\bv$ that maximizes the AO. To this aim note that $\widehat{\vct{z}}=\frac{1}{\sqrt{p}}\left(\hth^\eps-\vct{\theta}_0\right)$ obeys
\begin{align*}
\widehat{\vct{z}}&= \underset{\bz}{\arg\max}\text{ } \bq^T \bz - \ell(\bv;\bz)\\
&= \underset{\bz}{\arg\max}\text{ }  \bq^T \bz -\frac{1}{2p} \sum_{i=1}^n \left(|v_i| + \frac{\eps}{\sqrt{p}}\twonorm{\bth_0 +\sqrt{p}\bz}\right)^2\\
&=\underset{\bz}{\arg\max}\text{ }  \bq^T \bz -\frac{1}{2p}\left(\twonorm{\bv}^2 + \frac{2\eps}{\sqrt{p}}\onenorm{\bv} \twonorm{\bth_0 +\sqrt{p}\bz}+ \delta \eps^2 \twonorm{\bth_0 +\sqrt{p}\bz}^2\right)
\end{align*}
Setting derivative w.r.t $\bz$ to zero we arrive at
\begin{align}
\bq -\frac{\eps}{p^{3/2}} \onenorm{\bv} \frac{\bth_0+\sqrt{p}\widehat{\bz}}{\twonorm{\bth_0+\sqrt{p}\widehat{\bz}}} \sqrt{p} - \frac{\delta \eps^2}{p} (\bth_0+\sqrt{p}\widehat{\bz}) \sqrt{p} = 0
\end{align}
Therefore
\[
\bth_0 +\sqrt{p}\widehat{\bz} = \left(\frac{\eps\onenorm{\bv}}{p\twonorm{\bth_0+\sqrt{p}\widehat{\bz}}} + \frac{\delta\eps^2}{\sqrt{p}}\right)^{-1}\bq.
\]
Thus taking Euclidean norm of both sides of the identity we have
\begin{align*}
&\twonorm{\bth_0 +\sqrt{p}\widehat{\bz}}\left(\frac{\eps\onenorm{\bv}}{p\twonorm{\bth_0+\sqrt{p}\widehat{\bz}}} + \frac{\delta\eps^2}{\sqrt{p}}\right) = \twonorm{\bq} \Rightarrow \\
&\quad\quad\quad\quad\quad\quad\quad\quad\quad\quad\quad\quad\quad\quad\quad\twonorm{\bth_0 +\sqrt{p}\widehat{\bz}} = \frac{\twonorm{\bq} - \frac{\eps\onenorm{\bv}}{p}}{\frac{\delta\eps^2}{\sqrt{p}}} = \frac{\sqrt{p}}{\delta \eps^2}\twonorm{\bq} - \frac{1}{\delta\eps\sqrt{p}}\onenorm{\bv}.
\end{align*}
The latter holds as long as $\frac{\sqrt{p}}{\delta \eps^2}\twonorm{\bq} \ge \frac{1}{\delta\eps\sqrt{p}}\onenorm{\bv}$. When $\frac{\sqrt{p}}{\delta \eps^2}\twonorm{\bq} < \frac{1}{\delta\eps\sqrt{p}}\onenorm{\bv}$ it is easy to verify that that the objective value is smaller than or equal to $-\frac{\bq^T\bth_0}{\sqrt{p}} - \frac{1}{2p} \twonorm{\bv}^2$ and therefore $\widehat{\bz} = -\bth_0/\sqrt{p}$ which in turn implies that $\twonorm{\hth^\eps}=\twonorm{\bth_0+\sqrt{p}\widehat{\bz}} = 0$. We thus have
\begin{align}
\label{tmpkr}
\frac{1}{\sqrt{p}}\twonorm{\hth^\eps} =  \frac{1}{\sqrt{p}} \twonorm{\sqrt{p}\widehat{\bz}+\bth_0} = \frac{1}{\sqrt{p}}\left(\frac{\sqrt{p}}{\delta \eps^2}\twonorm{\bq} - \frac{1}{\delta\eps\sqrt{p}}\onenorm{\bv}\right)_+
= \frac{1}{\delta\eps p} \left(\frac{p}{\eps}\twonorm{\bq} - \onenorm{\bv}\right)_+.
\end{align}
So to get the asymptotic value of $\frac{1}{\sqrt{p}}\twonorm{\hth^\eps}$ we can simply look at $\frac{1}{\delta\eps p} \left(\frac{p}{\eps}\gamma - \onenorm{\bv}\right)_+$ with $\vct{v}$ and $\gamma=\tn{\vct{q}}$ the optimal solutions of the AO. Note that based on the argument in Lemma \ref{Glem} for this optimal solution of $\vct{v}$ we have
\begin{align}
\label{tmpkr2}
\underset{n\rightarrow +\infty}{\lim}\frac{1}{n^2}\left(\frac{p}{\eps}\gamma-\onenorm{\vct{v}}\right)^2=\frac{\omega^2}{(\mu+1)^2}\left(\frac{\gamma(\mu+1)}{\delta\eps\omega}+\tau^*\cdot\text{erfc}\left(\frac{1}{\sqrt{2}}\tau^*\right)-\sqrt{\frac{2}{\pi}} e^{-\frac{(\tau^*)^2}{2}}\right)_+^2
\end{align}
with $\omega=\sqrt{\alpha^2+\sigma^2}$, $\mu=\frac{\tau_g}{\beta}$, $\tau^*:=\tau^*\left(\frac{\gamma(\mu+1)}{\delta\eps\omega},\mu\right)$  and $\tau^*(a,\mu)$ is the unique solution to 
\begin{align}\label{eq:tau-star}
a-\frac{\mu+1}{\mu}\tau+\tau\cdot\text{erfc}\left(\frac{\tau}{\sqrt{2}}\right)-\sqrt{\frac{2}{\pi}} e^{-\frac{\tau^2}{2}}=0
\end{align}
Therefore, squaring \eqref{tmpkr} and plugging in \eqref{tmpkr2} we conclude that
\begin{align*}
\lim_{p\to\infty} \frac{1}{p} \twonorm{\hth^\eps}^2&=\lim_{p\to\infty}  \frac{1}{\delta^2\eps^2 p^2} \left(\frac{p}{\eps}\gamma - \onenorm{\bv}\right)_+^2\\
&= \frac{1}{\eps^2}\cdot \lim_{n\to\infty}  \frac{1}{n^2} \left(\frac{p}{\eps}\gamma - \onenorm{\bv}\right)_+^2\\
&=\frac{\omega^2}{\eps^2(\mu+1)^2}\left(\frac{\gamma(\mu+1)}{\delta\eps\omega}+\tau^*\cdot\text{erfc}\left(\frac{1}{\sqrt{2}}\tau^*\right)-\sqrt{\frac{2}{\pi}} e^{-\frac{(\tau^*)^2}{2}}\right)_+^2\\
&=\frac{\omega^2}{\eps^2(\mu+1)^2}\left(\frac{\gamma(\mu+1)}{\delta\eps\omega}+\frac{\mu+1}{\mu}\tau^*-\frac{\gamma(\mu+1)}{\delta \eps\omega}\right)_+^2\\
&=\frac{\omega^2}{\eps^2(\mu+1)^2}\left(\frac{\mu+1}{\mu}\tau^*\right)_+^2\\
&=\frac{\omega^2\tau_*^2}{\eps^2\mu^2}\\
&=\frac{(\alpha_*^2+\sigma^2)\tau_*^2}{\eps^2\mu^2}.
\end{align*}

\subsubsection{Proof of Corollary \ref{cor6}}
\label{cor6pf}
The result follows readily from Lemma~\ref{lem:SR-AR} along with Theorem~\ref{thm:main} (Parts (b) and (c)).

\subsubsection{Proof of Theorem \ref{delta_limit}}
\label{delta_limitpf}
We start by analyzing $\lim_{p\to\infty}\SR(\bth^\lambda)$ and $\lim_{p\to\infty} \AR(\bth^\lambda)$. Using Lemma~\ref{lem:SR-AR}, we have
\begin{align}
\lim_{p\to\infty} \SR(\bth^\lambda)& = \sigma^2+ \lim_{p\to\infty} \frac{1}{p} \twonorm{\bth^\lambda-\bth_0}^2\nonumber\\
&= \sigma^2+ \lim_{p\to\infty} \frac{1}{p} \twonorm{\bth_0}^2 \left(\frac{\gamma_0^\lambda}{1+\gamma_0^\lambda}\right)^2\nonumber\\
&= \sigma^2+ \left(\frac{\gamma_0^\lambda V}{1+\gamma_0^\lambda}\right)^2. \label{eq:SRL}
\end{align}
Likewise,
\begin{align}
\lim_{p\to\infty}\AR(\bth^\lambda) &= \sigma^2+ V^2 \left(\frac{\gamma_0^\lambda}{1+\gamma_0^\lambda}\right)^2 + \etest^2 \frac{V^2}{(1+\gamma_0^\lambda)^2} + 2\sqrt{\frac{2}{\pi}} \frac{\etest V}{1+\gamma_0^\lambda} \left(\sigma^2+  \left(\frac{\gamma_0^\lambda V}{1+\gamma_0^\lambda}\right)^2\right)^{1/2}, \label{eq:ARL}
\end{align}
with $\gamma_0^\lambda$ the fixed point of the following two equations:
\begin{align}\label{eq:my0}
\gamma_0^\lambda = \frac{\etest^2+\sqrt{\frac{2}{\pi}} \etest A^\lambda}{1+\lambda + \sqrt{\frac{2}{\pi}} \frac{\etest}{A^\lambda}}\,,\quad
A^\lambda = \frac{1}{V} \left((1+\gamma_0^\lambda)^2\sigma^2+ (\gamma_0^\lambda)^2 V^2\right)^{1/2}\,.
\end{align}

We next analyze $\lim_{\delta\to\infty} \lim_{n\to\infty}\SR(\hth^\eps)$ and $\lim_{\delta\to\infty} \lim_{n\to\infty}\AR(\hth^\eps)$. By using Corollary~\ref{cor6}, we have
\begin{align}
\lim_{\delta\to\infty}\lim_{n\to\infty} \SR(\hth^\eps) &= \lim_{\delta\to\infty} \left(\sigma^2 + \alpha_*^2\right)\,, \label{eq:SReps}\\
\lim_{\delta\to\infty} \lim_{n\to\infty} \AR(\hth^\eps) &=  \lim_{\delta\to\infty} \left\{\sigma^2+\alpha_*^2 + \etest^2 (\alpha_*^2+\sigma^2) \left(\frac{\beta_*\tau_*}{\eps \tau_{g*}}\right)^2
+2\sqrt{\frac{2}{\pi}} \frac{\etest\beta_*\tau_*}{\eps \tau_{g*}} (\sigma^2+ \alpha_*^2)\right\}\,.\label{eq:AReps}
\end{align}
Therefore, we need to study the solution of the convex-concave minimax optimization~\eqref{eq019-2} at the limits $\delta \to \infty$. It is straightforward to see that as $\delta\to \infty$, the indicator in~\eqref{eq019-2} is active and hence it reduces to 
 \begin{align}
 D(\alpha,\beta, \gamma,\tau_h,\tau_g)=&\frac{\delta\beta}{2(\tau_g+\beta)} \left(\alpha^2+\sigma^2\right)
 +\frac{\delta\beta^2(\alpha^2+\sigma^2)}{2\tau_g(\tau_g+\beta)}\erf\left(\frac{\tau_*}{\sqrt{2}}\right)\nn\\
 &-\frac{\alpha}{2\tau_h} \gamma^2 +\gamma \left(\sqrt{\frac{\alpha^2\beta^2}{\tau_h^2}+ V^2} - \frac{\beta \tau_* \sqrt{\alpha^2+\sigma^2}}{\eps\tau_g} \right) \nn\\
 &-\frac{\alpha}{2\tau_h}\beta^2  -\frac{\alpha\tau_h}{2}+\frac{\beta\tau_g}{2}\,.\label{eq:newD}
\end{align}
Solving for $\gamma$, we obtain
\[
\gamma_* = \frac{\tau_h}{\alpha} \left(\sqrt{\frac{\alpha^2\beta^2}{\tau_h^2} + V^2} - \frac{\beta\tau_* \sqrt{\alpha^2+\sigma^2}}{\eps\tau_g}\right)\,.
\]
Since $\gamma(\tau_g+\beta) > \sqrt{\frac{2}{\pi}} \delta \eps \beta \sqrt{\alpha^2+\sigma^2}$, we have $\gamma\to \infty$ as $\delta \to \infty$, and by the above equation for $\gamma_*$, we obtain that $\tau_h\to \infty$. Therefore,
\begin{align}\label{eq:my-gamma}
\gamma_* \to \frac{\tau_h}{\alpha} \left(V -  \frac{\beta\tau_* \sqrt{\alpha^2+\sigma^2}}{\eps\tau_g}\right)\,.
\end{align}
 In addition, $\tau_*\to 0$ as $\delta \to \infty$. Writing the  Taylor expansion of the characteristic equation of $\tau_*$ as per~\eqref{eq:tau*}, we get
\begin{align}\label{eq:my-gamma2}
\frac{\gamma(\tau_g+\beta)}{\beta\delta\eps\sqrt{\alpha^2+\sigma^2}} = \sqrt{\frac{2}{\pi}} + \frac{\beta\tau_*}{\tau_g} + O(\tau_*^2)\,.
\end{align}
We adopt the shorthands $\omega:=\sqrt{\alpha^2+\sigma^2}$ and $\mu := \frac{\tau_g}{\beta}$. 
Combining \eqref{eq:my-gamma2} with \eqref{eq:my-gamma} yields
 \[
\frac{\tau_h (\mu+1)}{\alpha \delta \eps \omega} \left(V - \frac{\tau_* \omega}{\eps\mu}\right)  = \sqrt{\frac{2}{\pi}} + \frac{\tau_*}{\mu}  + O(\tau_*^2)\,.
 \] 
 Writing the objective $D$ given by~\eqref{eq:newD} in terms of $\omega$, $\mu$, $\eta$ and after substituting for $\gamma_*$ we arrive at
 \begin{align}
 D=&\frac{\delta\omega^2}{2(\mu+1)} 
 +\frac{\delta\omega^2}{2\mu(\mu+1)}\erf\left(\frac{\tau_*}{\sqrt{2}}\right)\nn\\
 &+\frac{\tau_h}{2\alpha} \left(V- \frac{\tau_* \omega}{\eps\mu}\right)^2-\frac{\alpha}{2\tau_h}\beta^2-\frac{\alpha\tau_h}{2}+\frac{\beta^2\mu}{2}\,.
 \end{align}
Since $\delta, \tau_h\to \infty$, keeping only the dominant terms results in
 \begin{align}
 D=&\frac{\delta\omega^2}{2(\mu+1)} 
 +\frac{\delta\omega^2}{2\mu(\mu+1)}\erf\left(\frac{\tau_*}{\sqrt{2}}\right)+\frac{\tau_h}{2\alpha} \left(V- \frac{\tau_* \omega}{\eps\mu}\right)^2-\frac{\alpha\tau_h}{2}\,,
 \end{align}
and by keeping only terms of $O(\tau_*^2)$ we have
  \begin{align}
 D=&\frac{\delta\omega^2}{2(\mu+1)}  \left(1+\sqrt{\frac{2}{\pi}} \frac{\tau_*}{\mu}\right)
+\frac{\tau_h}{2\alpha} \left(V- \frac{\tau_* \omega}{\eps\mu}\right)^2-\frac{\alpha\tau_h}{2}\,.
 \end{align}
 Setting the derivative of $D$, with respect to $\tau_h$, to zero, we get
 \begin{align}\label{eq:myalpha}
 \alpha = V- \frac{\tau_*\omega}{\eps \mu}\,.
 \end{align}
We next set the derivative of $D$, with respect to $\alpha$, to zero, which implies 
\[
\frac{\delta\alpha}{\mu+1}  \left(1+\sqrt{\frac{2}{\pi}} \frac{\tau_*}{\mu}\right) -\frac{\tau_h}{2\alpha^2} \left(V- \frac{\tau_* \omega}{\eps\mu}\right)^2 -\frac{\tau_h}{\alpha} \left(V- \frac{\tau_* \omega}{\eps\mu}\right)\frac{\tau_*}{\eps\mu} \frac{\alpha}{\omega} - \frac{\tau_h}{2} =0\,.
\]
Plugging in for $\alpha$ from~\eqref{eq:myalpha} we obtain
\begin{align}\label{eq:myalpha2}
\alpha = \eps \omega \frac{\sqrt{\frac{2}{\pi}}+\frac{\tau_*}{\mu}}{1+\sqrt{\frac{2}{\pi}}\frac{\tau_*}{\mu}} \left(1+ \frac{\tau_*\alpha}{\eps \mu\omega}\right)\,.
\end{align}
Defining $A^\eps: = \frac{\eps \mu}{\tau_* }$ and $\gamma_0^\eps:= \frac{\eps \mu V}{\tau_*\omega} - 1$, the above two equations~\eqref{eq:myalpha}, \eqref{eq:myalpha2} imply that
\begin{align}
\alpha &=  V - \frac{\omega}{A^\eps} =\frac{\gamma_0^\eps V}{1+\gamma_0^\eps}, \label{eq:my1}\\
\frac{\alpha\eps}{\omega} & = \eps^2 \frac{1+\sqrt{\frac{2}{\pi}}\frac{\mu}{\tau_*} }{\frac{\mu}{\tau_*}+\sqrt{\frac{2}{\pi}} } \left( 1+ \frac{\tau_*\alpha}{\eps\mu\omega}\right) = \frac{\eps^2+\sqrt{\frac{2}{\pi}}\eps A^\eps}{\frac{A^\eps}{\eps}+\sqrt{\frac{2}{\pi}} }\left( 1+ \frac{\tau_*\alpha}{\eps\mu\omega}\right). \label{eq:my2}
\end{align}
From~\eqref{eq:my1} we obtain
\begin{align}\label{eq:my3}
\frac{1}{V}\left((1+\gamma_0^\eps)^2\sigma^2 + (\gamma_0^\eps)^2 V^2\right)^{1/2} = \frac{1+\gamma_0^\eps}{V} \omega = A^\eps.
\end{align}
In addition, from \eqref{eq:my1} and \eqref{eq:my2} we have
\begin{align}\label{eq:my4}
\gamma_0^\eps = \frac{VA^\eps}{\omega}-1 = \frac{A^\eps \alpha}{ \omega}
=\frac{\eps^2+\sqrt{\frac{2}{\pi}}\eps A^\eps}{1+\sqrt{\frac{2}{\pi}} \frac{\eps}{A^\eps}} \left(1+ \frac{\gamma_0^\eps}{(A^\eps)^2}\right). 
\end{align}
Combining equations~\eqref{eq:my3} and \eqref{eq:my4}, we have that $\gamma_0^\eps$ is the fixed point of the following two equations:
\begin{align}
\gamma_0^\eps = \frac{\eps^2+\sqrt{\frac{2}{\pi}} \eps A^\eps}{1 -(\frac{\eps}{A^\eps})^2}  \,,\quad
A^\eps = \frac{1}{V} \left((1+\gamma_0^\eps)^2\sigma^2+ (\gamma_0^\eps)^2 V^2\right)^{1/2}.
\end{align}

Now consider a fixed $\lambda \geq 0$ and let $\gamma_0^\lambda, A^\lambda$ be defined by~\eqref{eq:my0}. Comparing equations~\eqref{eq:SRL} and \eqref{eq:ARL} with \eqref{eq:SReps} and \eqref{eq:AReps}, we see that in order to prove the statement, it suffices to find corresponding $\eps \geq 0$ such that  
$\gamma_0^\eps = \gamma_0^\lambda$ (Note that the statement $\gamma_0^\eps = \gamma_0^\lambda$ implies that $A^\eps = A^\lambda$ as well). Such value of $\eps$ is hence found from the following equation (which equates $\gamma_0^\eps = \gamma_0^\lambda$ and $A^\lambda = A^\eps$): 
$$  \frac{\etest^2+\sqrt{\frac{2}{\pi}} \etest A^\lambda}{1+\lambda + \sqrt{\frac{2}{\pi}} \frac{\etest}{A^\lambda}} = 
 \frac{\eps^2+\sqrt{\frac{2}{\pi}} \eps A^\lambda}{1 -(\frac{\eps}{A^\lambda})^2}
\,. $$
Rearranging terms, we reach to: 
$$ \eps^2 \left( 1+\lambda + \sqrt{\frac{2}{\pi}} \frac{\etest}{A^\lambda} +\left(\frac{\etest}{A^\lambda}\right)^2 + \sqrt{\frac{2}{\pi}} \frac{\etest}{A^\lambda} \right) + \eps \sqrt{\frac{2}{\pi}} \left( A^\lambda (1 + \lambda) +\sqrt{\frac{2}{\pi}} \etest \right) - \left(\etest^2 + \sqrt{\frac{2}{\pi}} \etest A^\lambda \right)  = 0\,.$$
The thesis now follows by noting that the above equation is a quadratic form in $\eps$ and has always a positive solution, which gives the value of $\eps$ in terms of $\lambda$.
\section*{Acknowledgements}
A. Javanmard is partially supported by a Google Faculty Research Award and the NSF CAREER Award DMS-1844481. M. Soltanolkotabi is supported by the Packard Fellowship in Science
and Engineering, a Sloan Research Fellowship in Mathematics, an NSF-CAREER under award
$\#1846369$, the Air Force Office of Scientific Research Young Investigator Program (AFOSR-YIP)
under award $\#$FA$9550-18-1-0078$, Darpa Learning with Less Labels (LwLL) program, an NSF-CIF award $\#1813877$, and a Google faculty research award. This work was done in part while M.S. was visiting the Simons Institute for the Theory of Computing. The research of H. Hassani is supported by NSF HDR TRIPODS award 1934876, NSF award CPS-1837253, NSF award CIF-1910056, and NSF CAREER award CIF-1943064.

\bibliographystyle{amsalpha}
\bibliography{Bibfiles.bib}

\newcommand{\etalchar}[1]{$^{#1}$}
\providecommand{\bysame}{\leavevmode\hbox to3em{\hrulefill}\thinspace}
\providecommand{\MR}{\relax\ifhmode\unskip\space\fi MR }
\providecommand{\MRhref}[2]{%
  \href{http://www.ams.org/mathscinet-getitem?mr=#1}{#2}
}
\providecommand{\href}[2]{#2}
\begin{thebibliography}{BHMM18}

\bibitem[BCM{\etalchar{+}}13]{biggio2013evasion}
Battista Biggio, Igino Corona, Davide Maiorca, Blaine Nelson, Nedim
  {\v{S}}rndi{\'c}, Pavel Laskov, Giorgio Giacinto, and Fabio Roli,
  \emph{Evasion attacks against machine learning at test time}, Joint European
  conference on machine learning and knowledge discovery in databases,
  Springer, 2013, pp.~387--402.

\bibitem[BHMM18]{belkin2018reconciling}
Mikhail Belkin, Daniel Hsu, Siyuan Ma, and Soumik Mandal, \emph{Reconciling
  modern machine learning and the bias-variance trade-off}, arXiv preprint
  arXiv:1812.11118 (2018).

\bibitem[BLPR19]{DBLP:conf/icml/BubeckLPR19}
S{\'{e}}bastien Bubeck, Yin~Tat Lee, Eric Price, and Ilya~P. Razenshteyn,
  \emph{Adversarial examples from computational constraints}, Proceedings of
  the 36th International Conference on Machine Learning, {ICML} 2019, 9-15 June
  2019, Long Beach, California, {USA}, 2019, pp.~831--840.

\bibitem[BMM18]{belkin2018understand}
Mikhail Belkin, Siyuan Ma, and Soumik Mandal, \emph{To understand deep learning
  we need to understand kernel learning}, International Conference on Machine
  Learning, 2018, pp.~541--549.

\bibitem[CBM18]{DBLP:conf/nips/CullinaBM18}
Daniel Cullina, Arjun~Nitin Bhagoji, and Prateek Mittal, \emph{Pac-learning in
  the presence of adversaries}, Advances in Neural Information Processing
  Systems 31: Annual Conference on Neural Information Processing Systems 2018,
  NeurIPS 2018, 3-8 December 2018, Montr{\'{e}}al, Canada, 2018, pp.~228--239.

\bibitem[DKT19]{deng2019model}
Zeyu Deng, Abla Kammoun, and Christos Thrampoulidis, \emph{A model of double
  descent for high-dimensional binary linear classification}, arXiv preprint
  arXiv:1911.05822 (2019).

\bibitem[GCL{\etalchar{+}}19]{DBLP:conf/nips/GaoCLHWL19}
Ruiqi Gao, Tianle Cai, Haochuan Li, Cho{-}Jui Hsieh, Liwei Wang, and Jason~D.
  Lee, \emph{Convergence of adversarial training in overparametrized neural
  networks}, Advances in Neural Information Processing Systems 32: Annual
  Conference on Neural Information Processing Systems 2019, NeurIPS 2019, 8-14
  December 2019, Vancouver, BC, Canada, 2019, pp.~13009--13020.

\bibitem[GMF{\etalchar{+}}18]{gilmer2018adversarial}
Justin Gilmer, Luke Metz, Fartash Faghri, Samuel~S Schoenholz, Maithra Raghu,
  Martin Wattenberg, and Ian Goodfellow, \emph{Adversarial spheres}, arXiv
  preprint arXiv:1801.02774 (2018).

\bibitem[Gor88]{gordon1988milman}
Yehoram Gordon, \emph{On milman's inequality and random subspaces which escape
  through a mesh in $\mathbb{R}^n$}, Geometric aspects of functional analysis,
  Springer, 1988, pp.~84--106.

\bibitem[GSS15]{DBLP:journals/corr/GoodfellowSS14}
Ian~J. Goodfellow, Jonathon Shlens, and Christian Szegedy, \emph{Explaining and
  harnessing adversarial examples}, 3rd International Conference on Learning
  Representations, {ICLR} 2015, San Diego, CA, USA, May 7-9, 2015, Conference
  Track Proceedings, 2015.

\bibitem[HMRT19]{hastie2019surprises}
Trevor Hastie, Andrea Montanari, Saharon Rosset, and Ryan~J Tibshirani,
  \emph{Surprises in high-dimensional ridgeless least squares interpolation},
  arXiv preprint arXiv:1903.08560 (2019).

\bibitem[KGB16]{kurakin2016adversarial}
Alexey Kurakin, Ian Goodfellow, and Samy Bengio, \emph{Adversarial machine
  learning at scale}, arXiv preprint arXiv:1611.01236 (2016).

\bibitem[KL18]{DBLP:journals/corr/abs-1810-09519}
Justin Khim and Po{-}Ling Loh, \emph{Adversarial risk bounds for binary
  classification via function transformation}, CoRR \textbf{abs/1810.09519}
  (2018).

\bibitem[LM08]{StatDecision}
Friedrich Liese and Klaus-J. Miescke, \emph{Statistical decision theory:
  Estimation, testing, and selection}, Springer Science \& Business Media,
  2008.

\bibitem[LS20]{liang2020precise}
Tengyuan Liang and Pragya Sur, \emph{A precise high-dimensional asymptotic
  theory for boosting and min-l1-norm interpolated classifiers}, arXiv preprint
  arXiv:2002.01586 (2020).

\bibitem[MHS19]{DBLP:conf/colt/MontasserHS19}
Omar Montasser, Steve Hanneke, and Nathan Srebro, \emph{{VC} classes are
  adversarially robustly learnable, but only improperly}, Conference on
  Learning Theory, {COLT} 2019, 25-28 June 2019, Phoenix, AZ, {USA}, 2019,
  pp.~2512--2530.

\bibitem[MM19]{mei2019generalization}
Song Mei and Andrea Montanari, \emph{The generalization error of random
  features regression: Precise asymptotics and double descent curve}, arXiv
  preprint arXiv:1908.05355 (2019).

\bibitem[MMS{\etalchar{+}}17]{madry2017towards}
Aleksander Madry, Aleksandar Makelov, Ludwig Schmidt, Dimitris Tsipras, and
  Adrian Vladu, \emph{Towards deep learning models resistant to adversarial
  attacks}, arXiv preprint arXiv:1706.06083 (2017).

\bibitem[MMS{\etalchar{+}}18]{DBLP:conf/iclr/MadryMSTV18}
Aleksander Madry, Aleksandar Makelov, Ludwig Schmidt, Dimitris Tsipras, and
  Adrian Vladu, \emph{Towards deep learning models resistant to adversarial
  attacks}, 6th International Conference on Learning Representations, {ICLR}
  2018, Vancouver, BC, Canada, April 30 - May 3, 2018, Conference Track
  Proceedings, 2018.

\bibitem[MRSY19]{montanari2019generalization}
Andrea Montanari, Feng Ruan, Youngtak Sohn, and Jun Yan, \emph{The
  generalization error of max-margin linear classifiers: High-dimensional
  asymptotics in the overparametrized regime}, arXiv preprint arXiv:1911.01544
  (2019).

\bibitem[Nak19]{nakkiran2019adversarial}
Preetum Nakkiran, \emph{Adversarial robustness may be at odds with simplicity},
  arXiv preprint arXiv:1901.00532 (2019).

\bibitem[PJ19]{pydi2019adversarial}
Muni~Sreenivas Pydi and Varun Jog, \emph{Adversarial risk via optimal transport
  and optimal couplings}, arXiv preprint arXiv:1912.02794 (2019).

\bibitem[RSL18]{DBLP:conf/iclr/RaghunathanSL18}
Aditi Raghunathan, Jacob Steinhardt, and Percy Liang, \emph{Certified defenses
  against adversarial examples}, 6th International Conference on Learning
  Representations, {ICLR} 2018, Vancouver, BC, Canada, April 30 - May 3, 2018,
  Conference Track Proceedings, 2018.

\bibitem[RXY{\etalchar{+}}19]{raghunathan2019adversarial}
Aditi Raghunathan, Sang~Michael Xie, Fanny Yang, John~C Duchi, and Percy Liang,
  \emph{Adversarial training can hurt generalization}, arXiv preprint
  arXiv:1906.06032 (2019).

\bibitem[S{\etalchar{+}}58]{sion1958general}
Maurice Sion et~al., \emph{On general minimax theorems.}, Pacific Journal of
  mathematics \textbf{8} (1958), no.~1, 171--176.

\bibitem[SHS{\etalchar{+}}19]{DBLP:conf/iclr/ShafahiHSFG19}
Ali Shafahi, W.~Ronny Huang, Christoph Studer, Soheil Feizi, and Tom Goldstein,
  \emph{Are adversarial examples inevitable?}, 7th International Conference on
  Learning Representations, {ICLR} 2019, New Orleans, LA, USA, May 6-9, 2019,
  2019.

\bibitem[SST{\etalchar{+}}18]{DBLP:conf/nips/SchmidtSTTM18}
Ludwig Schmidt, Shibani Santurkar, Dimitris Tsipras, Kunal Talwar, and
  Aleksander Madry, \emph{Adversarially robust generalization requires more
  data}, Advances in Neural Information Processing Systems 31: Annual
  Conference on Neural Information Processing Systems 2018, NeurIPS 2018, 3-8
  December 2018, Montr{\'{e}}al, Canada, 2018, pp.~5019--5031.

\bibitem[SZS{\etalchar{+}}14]{szegedy2014intriguing}
Christian Szegedy, Wojciech Zaremba, Ilya Sutskever, Joan Bruna, Dumitru Erhan,
  Ian~J Goodfellow, and Rob Fergus, \emph{Intriguing properties of neural
  networks. iclr, abs/1312.6199, 2014}, 2014.

\bibitem[TAH15]{thrampoulidis2015precise}
Christos Thrampoulidis, Ehsan Abbasi, and Babak Hassibi, \emph{Precise
  high-dimensional error analysis of regularized m-estimators}, 2015 53rd
  Annual Allerton Conference on Communication, Control, and Computing
  (Allerton), IEEE, 2015, pp.~410--417.

\bibitem[TAH18]{thrampoulidis2018precise}
\bysame, \emph{Precise error analysis of regularized $ m $-estimators in high
  dimensions}, IEEE Transactions on Information Theory \textbf{64} (2018),
  no.~8, 5592--5628.

\bibitem[TOH15]{thrampoulidis2015regularized}
Christos Thrampoulidis, Samet Oymak, and Babak Hassibi, \emph{Regularized
  linear regression: A precise analysis of the estimation error}, Conference on
  Learning Theory, 2015, pp.~1683--1709.

\bibitem[TSE{\etalchar{+}}18]{tsipras2018robustness}
Dimitris Tsipras, Shibani Santurkar, Logan Engstrom, Alexander Turner, and
  Aleksander Madry, \emph{Robustness may be at odds with accuracy}, arXiv
  preprint arXiv:1805.12152 (2018).

\bibitem[WK18]{DBLP:conf/icml/WongK18}
Eric Wong and J.~Zico Kolter, \emph{Provable defenses against adversarial
  examples via the convex outer adversarial polytope}, Proceedings of the 35th
  International Conference on Machine Learning, {ICML} 2018,
  Stockholmsm{\"{a}}ssan, Stockholm, Sweden, July 10-15, 2018, 2018,
  pp.~5283--5292.

\bibitem[YRB19]{DBLP:conf/icml/YinRB19}
Dong Yin, Kannan Ramchandran, and Peter~L. Bartlett, \emph{Rademacher
  complexity for adversarially robust generalization}, Proceedings of the 36th
  International Conference on Machine Learning, {ICML} 2019, 9-15 June 2019,
  Long Beach, California, {USA}, 2019, pp.~7085--7094.

\bibitem[ZYJ{\etalchar{+}}19]{DBLP:conf/icml/ZhangYJXGJ19}
Hongyang Zhang, Yaodong Yu, Jiantao Jiao, Eric~P. Xing, Laurent~El Ghaoui, and
  Michael~I. Jordan, \emph{Theoretically principled trade-off between
  robustness and accuracy}, Proceedings of the 36th International Conference on
  Machine Learning, {ICML} 2019, 9-15 June 2019, Long Beach, California, {USA},
  2019, pp.~7472--7482.

\end{thebibliography}

\appendix

\section{Further insights and guarantees into the effect of the size of the training data}
\label{insight}
To provide further insight into the role of the size of the training data on adversarial training we note that we have already shown in our proofs (See Section \ref{step1} and equation \eqref{concstep1}) that the inner maximization in the saddle point problem~\eqref{eq:htheps} has a closed form solution and the estimator $\hth^\eps$ can be equivalently defined by
\begin{align}\label{eq:htheps-2}
\hth^{\eps} \in \arg\min_{\bth\in \reals^p}\, \frac{1}{2p} \sum_{i=1}^n \left(|y_i-\<\x_i,\bth\>|+ \eps \twonorm{\bth}\right)^2\,.
\end{align}         
Therefore for linear regression, adversarial training by the saddle point optimization~\eqref{eq:hth-cS} amounts to a \emph{regularized estimator}. When $\delta<1$, we are in the overparametrized regime and regularization helps with standard accuracy. In particular, when $\delta\to 1$, the condition number of the covariate matrix  diverges (a.k.a interpolation threshold~\cite{belkin2018understand,belkin2018reconciling,hastie2019surprises}) and the role of regularization becomes crucial, without which the standard risk would diverge. This is reflected in Figure~\ref{fig:SR_deltaLess1} in that the standard risk diverges at $\eps= 0$ as $\delta \to 1$, and also the statistical risk plummets quickly with $\eps$ ; See also Proposition~\ref{pro:slope} below.  

Nonetheless, in the $\delta>1$ regime the effect of regularization starts to weaken. To see why, note that as $\delta$ grows, the ratio of sample size $n$ to the dimension $p$ increases, and the reduction in the variance of the estimator due to regularization becomes comparative to the increase in the bias caused by this term. As a result the overall positive effect of regularization on standard risk lessens and we see in Figure~\ref{fig:SR_deltaLarger1}, the negative slope at $\eps=0$ decreases as $\delta$ increases. In addition, at large $\delta$, the standard risk will start to quickly becomes increasing with $\eps$. In other words, for larger $\delta$, the negative effect of adversarial training on standard risk starts to emerge at smaller values of $\eps$. (For example at $\delta=10$, this effect kicks in at $\eps = 0.15$.)

Our next proposition describes the standard risk at small values of $\eps$.
\begin{propo}\label{pro:slope}
Under the assumptions of Theorem~\ref{thm:main} and for $\delta\ge 1$ and $\eps\le 1$, we have
\begin{align}
\lim_{n\to \infty} \SR(\hth^\eps) = \frac{\delta \sigma^2}{\delta-1} 
- 2\sqrt{\frac{2}{\pi}} \frac{\sigma^3\delta^{3/2}}{(\delta-1)^2} \cdot \frac{1}{\sqrt{\sigma^2+V^2(\delta -1)}}\, \eps + O(\eps^2)\,.
\end{align}
\end{propo}
As a result of Proposition~\ref{pro:slope}, for $\eps$ small and $\delta\ge1$: (i) standard risk $\alpha_*$ falls with $\eps$ at vicinity of $\eps=0$ (ii) the risk falls slower at larger $\delta$ (iii) as $\delta \to 1$, the slope diverges and the risk plummets rapidly. These observations corroborates our justification and insights provided above. 

We finish this appendix by the proof of Proposition~\ref{pro:slope}.

\begin{proof}[Proof of Proposition~\ref{pro:slope}]
Define $\x= (\alpha,\beta, \tau_h, \tau_g,\gamma)$.
We can write the objective of the convex-concave minimax problem~\eqref{eq019} as
\[
D(\alpha,\beta, \tau_h, \tau_g,\gamma) = \bar{D}(\alpha,\beta, \tau_h, \tau_g,\gamma) +\mathbb{1}_{\left\{\frac{\gamma(\tau_g+\beta)}{\delta\eps\beta\sqrt{\alpha^2+\sigma^2}}> \sqrt{\frac{2}{\pi}}\right\}} \widetilde{D}(\alpha,\beta, \tau_h, \tau_g,\gamma)\,,
\]
where $\bar{D}$ does not depend on $\eps$.
It is easy to see that when $\eps=0$, then $\gamma=0$. Otherwise $\tau^*=\infty$ and $\widetilde{D} = -\infty$ which implies that the maximum of $D$ over $\gamma$ is achieved at $\gamma=0$.
Therefore at $\eps=0$, we get
\[
D = \bar{D} = \frac{\delta\beta}{2(\tau_g+\beta)}(\alpha^2+\sigma^2) - \frac{\alpha}{2\tau_h} \beta^2 -\frac{\alpha \tau_h}{2} + \frac{\beta\tau_g}{2}\,.
\]
The stationary point is given by $(\tau_g+\beta)^2 = \delta (\alpha^2+\sigma^2)$, $\tau_h = \beta$ and $\delta\alpha = \tau_g+\beta$, $\tau_g = \alpha$ (derivative with respect to $\beta$). Putting things together we have 
\begin{align}\label{eq:zero_eps}
\alpha^2 = \frac{\sigma^2}{\delta -1}, \quad \tau_g = \alpha = \frac{\sigma}{\sqrt{\delta -1}},\quad 
\tau_h = \beta = \sigma \sqrt{\delta -1}\,, \quad \gamma = 0\,.
\end{align}

We next study the behavior of the convex-concave minimax problem~\ref{eq019} at infinitesimal $\eps$. Rewriting the expressions for $\bar{D}$ and $\widetilde{D}$, we have
\begin{align}
\bar{D} &= \frac{\delta\beta}{2(\tau_g+\beta)} \left(\alpha^2+\sigma^2\right) -\frac{\alpha}{2\tau_h}(\gamma^2 +\beta^2)+ \gamma\sqrt{\frac{\alpha^2\beta^2}{\tau_h^2}+ V^2} -\frac{\alpha\tau_h}{2}+\frac{\beta\tau_g}{2}\,,\nonumber\\
\widetilde{D}&= \frac{\delta \beta^2(\alpha^2+\sigma^2)}{2\tau_g(\tau_g+\beta)}\left(\text{erf}\left(\frac{\tau^*}{\sqrt{2}}\right)-\frac{\gamma(\tau_g+\beta)}{\delta\eps\beta\sqrt{\alpha^2+\sigma^2}}\tau^*\right)\,.
\end{align}

Let $\gamma_0 := \sqrt{\frac{2}{\pi}}\frac{\delta\eps\beta\sqrt{\alpha^2+\sigma^2}}{\tau_g+\beta}$.
If $\gamma\ge \gamma_0$, then $D$ is a quadratic function of $\gamma$ with the peak location at 
$$\gamma_1:= \sqrt{\beta^2+\frac{\tau_h^2 }{\alpha^2}V^2} - \frac{\tau_h\beta\sqrt{\alpha^2+\sigma^2}}{2\alpha \eps\tau_g} \tau_*\,.$$

If $\gamma<\gamma_0$, then $D = \bar{D}$ is quadratic in $\gamma$ with the peak location at
$$\gamma_2:= \sqrt{\beta^2+\frac{\tau_h^2}{\alpha^2} V^2}\,.$$

Therefore, to find the optimal $\gamma$ we need to consider three different cases, giving us
\begin{align}\label{eq:gammas}
\gamma_* = \begin{cases}
\gamma_1 & \text{ if }\gamma_0\le \gamma_1\le \gamma_2\,,\\
\gamma_0 &\text{ if }\gamma_1\le \gamma_0 \le \gamma_2\,,\\
\gamma_2 &\text{ if }\gamma_1\le \gamma_2\le \gamma_0\,.
\end{cases}
\end{align}
As $\eps \to 0$, we have $\gamma_0\to 0$. However, using \eqref{eq:zero_eps} we get $\gamma_2\to 
\sqrt{\sigma^2(\delta-1) + (\delta-1)^2 V^2} > 0$. By continuity, at infinitesimal $\eps$ we get $\gamma_0 < \gamma_2$.
Hence, in \eqref{eq:gammas} only the first two cases may happen. Suppose that the first case occurs. Then, $0\le \gamma_0\le \gamma_1$ and by definition of $\gamma_1$ we obtain that $\tau_* = O(\eps)$. Invoking the characterization equation of $\tau_*$ as per~\eqref{eq:tau*}, we get
\begin{align}\label{eq:mycase1}
\frac{\gamma_*(\tau_g +\beta)}{\delta \eps \beta\sqrt{\alpha^2+\sigma^2}} = \sqrt{\frac{2}{\pi}} + O(\eps)\,,\quad \tau_* = O(\eps)\,.
\end{align}
If the second case in~\eqref{eq:gammas} happens, we have $\gamma_* = \gamma_0 =\sqrt{\frac{2}{\pi}}\frac{\delta\eps\beta\sqrt{\alpha^2+\sigma^2}}{\tau_g+\beta}$ and $\tau_* = 0$. So this case is subsumed in~\eqref{eq:mycase1} and henceforth we can proceed with~\eqref{eq:mycase1}. 

By Taylor expansion of the ${{\rm erf}}$ function we have
\[
\text{erf}\left(\frac{\tau^*}{\sqrt{2}}\right) = \sqrt{\frac{2}{\pi}} \tau_* + O(\tau_*^3)\,,
\]
which implies that $\widetilde{D} = O(\tau_*^3) = O(\eps^3)$.  Separating $O(\eps^2)$ terms from the lower order terms we get 
\begin{align}
&D(\alpha, \beta, \tau_g, \tau_h) = D_0(\alpha, \beta, \tau_g, \tau_h) + \eps D_1(\alpha, \beta, \tau_g, \tau_h) + O(\eps^2)\,,\\
&D_0(\alpha, \beta, \tau_g, \tau_h)  = \frac{\delta\beta}{2(\tau_g+\beta)} \left(\alpha^2+\sigma^2\right) -\frac{\alpha}{2\tau_h} \beta^2-\frac{\alpha\tau_h}{2}+\frac{\beta\tau_g}{2}\,, \nonumber\\
&D_1(\alpha, \beta, \tau_g, \tau_h)  =  \sqrt{\frac{2}{\pi}} \frac{\delta\beta\sqrt{\alpha^2+\sigma^2}}{\tau_g+\beta}\sqrt{\frac{\alpha^2\beta^2}{\tau_h^2}+ V^2} \,.
\end{align}
Letting $\x= (\alpha, \beta, \tau_g, \tau_h)$, we then have 
\begin{align}\label{eq:D-eps}
\nabla D(\x) = \nabla D_0(\x) + \eps\nabla D_1(\x) + O(\eps^2)\,.
\end{align}
To get the stationary points, we need to solve for $\nabla D(\x) = 0$. However, to find the solution up to $O(\eps)$ term we can instead solve for $\nabla D_0(\x)+ \eps\nabla D_1(\x) = 0$. To see why, suppose that $\nabla D(\x_*) = 0$ and write 
$\x_*= \x_0 +\eps \x_1+ O(\eps^2)$.
Hence,
\begin{align}
{\bf 0} = \nabla D(\x_*) &= \nabla D_0(\x_*) + \eps \nabla D_1(\x_*) + O(\eps^2) \nonumber\\
&= \nabla D_0(\x_0) + \eps (\nabla^2 D_0(\x_0) \x_1 + \nabla D_1(\x_0)) + O(\eps^2)\,.
\end{align}
This implies that $\x_0$ and $\x_1$ should satisfy 
\begin{align}\label{eq:x0-x1}
\nabla D_0(\x_0)= 0 \quad \text{and} \quad \nabla^2 D_0(\x_0) \x_1 + \nabla D_1(\x_0) = 0\,. 
\end{align}
Likewise, let $\tilde{\x}_*$ be the solution of $\nabla D_0(\x)+ \eps \nabla D_1(\x) = 0$ and write $\tilde{\x}_* = \tilde{\x}_0+\eps \tilde{\x}_1+O(\eps^2)$. Then following similar arguments, we get
\begin{align}\label{eq:tx0-tx1}
\nabla D_0(\tilde{\x}_0)= 0 \quad \text{and} \quad \nabla^2 D_0(\tilde{\x}_0) \tilde{\x}_1 + \nabla D_1(\tilde{\x}_0) = 0\,. 
\end{align}
Comparing equations \eqref{eq:x0-x1} and \eqref{eq:tx0-tx1}, we see that $\x_0=\tilde{\x}_0$ and $\x_1 = \tilde{\x}_1$. Therefore, to find the stationary point $\x_*$ up to $O(\eps)$ terms, we can neglect $O(\eps^2)$ term in~\eqref{eq:D-eps}.  

We proceed by computing the stationary points of $D_0(\x) + \eps D_1(\x)$.
Writing KKT conditions with respect to $\alpha$, $\beta$, $\tau_g$, $\tau_h$ we have
\begin{align*}
&\frac{\sqrt{\frac{2}{\pi}} \alpha \beta \delta \eps \sqrt{\frac{\alpha ^2 \beta ^2}{\tau_h^2}+V^2}}{\sqrt{\alpha ^2+\sigma ^2}(\beta+\tau_g)}+\frac{\sqrt{\frac{2}{\pi}} \alpha \beta ^3 \delta \eps \sqrt{\alpha ^2+\sigma ^2}}{\tau_h^2(\beta+\tau_g) \sqrt{\frac{\alpha ^2 \beta ^2}{\tau_h^2}+V^2}}+\frac{\alpha \beta \delta}{\beta+\tau_g}-\frac{\beta ^2}{2 \tau_h}-\frac{\tau_h}{2}=0\,,\nn\\
&\frac{\sqrt{2} \alpha ^2 \beta ^2 \delta \eps \sqrt{\alpha ^2+\sigma ^2}}{\sqrt{\pi} \tau_h^2(\beta+\tau_g) \sqrt{\frac{\alpha ^2 \beta ^2}{\tau_h^2}+V^2}}-\frac{\sqrt{\frac{2}{\pi}} \beta \delta \eps \sqrt{\alpha ^2+\sigma ^2} \sqrt{\frac{\alpha ^2 \beta ^2}{\tau_h^2}+V^2}}{(\beta+\tau_g)^2}\nn\\
&\quad\quad\quad\quad \quad\quad +\frac{\sqrt{\frac{2}{\pi}} \delta \eps \sqrt{\alpha ^2+\sigma ^2} \sqrt{\frac{\alpha ^2 \beta ^2}{\tau_h^2}+V^2}}{\beta+\tau_g}-\frac{\beta \delta \left(\alpha ^2+\sigma ^2\right)}{2(\beta+\tau_g)^2}+\frac{\delta \left(\alpha ^2+\sigma ^2\right)}{2(\beta+\tau_g)}-\frac{\alpha \beta}{\tau_h}+\frac{\tau_g}{2}=0\,,\\
&-\frac{\sqrt{\frac{2}{\pi}} \beta \delta \eps \sqrt{\alpha ^2+\sigma ^2} \sqrt{\frac{\alpha ^2 \beta ^2}{\tau_h^2}+V^2}}{(\beta+\tau_g)^2}-\frac{\beta \delta \left(\alpha ^2+\sigma ^2\right)}{2(\beta+\tau_g)^2}+\frac{\beta}{2}=0\,,\\
&-\frac{\sqrt{\frac{2}{\pi}} \alpha ^2 \beta ^3 \delta \eps \sqrt{\alpha ^2+\sigma ^2}}{\tau_h^3(\beta+\tau_g) \sqrt{\frac{\alpha ^2 \beta ^2}{\tau_h^2}+V^2}}+\frac{\alpha \beta ^2}{2 \tau_h^2}-\frac{\alpha}{2}=0\,.
\end{align*}
Second equation can be simplified using other equations as
\begin{align*}
&\frac{\alpha\beta}{2\tau_h} -\frac{\alpha \tau_h}{2\beta} - \frac{\beta}{2} -\frac{\delta(\alpha^2+\sigma^2)}{2(\beta+\tau_g)} +\frac{\beta+\tau_g}{2}+\frac{\delta \left(\alpha ^2+\sigma ^2\right)}{2(\beta+\tau_g)}-\frac{\alpha \beta}{\tau_h}+\frac{\tau_g}{2}=0\,.\\
&\to -\frac{\alpha\beta}{2\tau_h}+\tau_g -\frac{\alpha \tau_h}{2\beta} =0\,.
\end{align*}
The first equation also simplifies to
\begin{align*}
&-\frac{\alpha \beta\delta}{2(\beta+\tau_g)} +\frac{\alpha(\beta+\tau_g)\beta}{2(\alpha^2+\sigma^2)}
+\frac{\beta^2}{2\tau_h} - \frac{\tau_h}{2} +\frac{\alpha \beta \delta}{\beta+\tau_g}-\frac{\beta ^2}{2 \tau_h}-\frac{\tau_h}{2}= 0\,,\\
&\to \frac{\alpha \beta\delta}{2(\beta+\tau_g)} +\frac{\alpha(\beta+\tau_g)\beta}{2(\alpha^2+\sigma^2)} -\tau_h = 0\,.
\end{align*}
Define $\eta = \beta/\tau_h> 1$ (since $\eps>0$). The second equation gives $\tau_g = \alpha/2(\eta+1/\eta)$. While this becomes useful in finding optimal $\tau_g$ it does not matter with our goal of finding $\alpha$ as everywhere $\tau_g$ appears in form $\beta+\tau_g$. The first equation though gives
\begin{align}\label{eq:dum0}
\frac{\delta}{\beta+\tau_g} + \frac{\beta+\tau_g}{(\alpha^2+\sigma^2)} = \frac{2}{\alpha \eta}\,.
\end{align}
The third equation gives
\begin{align}\label{eq:dum1}
2\sqrt{\tfrac{2}{\pi}} \delta\eps\sqrt{\alpha^2+\sigma^2} \sqrt{\alpha^2\eta^2+V^2} + \delta (\alpha^2+\sigma^2) = (\beta+\tau_g)^2\,.
\end{align}
The fourth equation gives
\begin{align}
\frac{2\sqrt{\frac{2}{\pi}} \alpha  \eta ^3 \delta \eps \sqrt{\alpha ^2+\sigma ^2}}{(\eta^2-1) \sqrt{{\alpha ^2  \eta^2}+V^2}} = \beta+\tau_g\,.
\end{align}

Continuing from~\eqref{eq:dum0} we get
\begin{align}
\sqrt{\frac{2}{\pi}} \eps \sqrt{\alpha^2\eta^2+V^2} + \sqrt{\alpha^2+\sigma^2} = (\alpha^2+\sigma^2) \frac{2\sqrt{\frac{2}{\pi}}  \eta ^2  \eps}{(\eta^2-1) \sqrt{{\alpha ^2  \eta^2}+V^2}} \,.
\end{align}
Simplifying this equation,
\begin{align}
\sqrt{\frac{2}{\pi}} \eps \sqrt{\alpha^2\eta^2+V^2}\, (\eta^2-1) + \sqrt{\alpha^2+\sigma^2}(\eta^2-1) = (\alpha^2+\sigma^2) \frac{2\sqrt{\frac{2}{\pi}}  \eta ^2  \eps}{ \sqrt{{\alpha ^2  \eta^2}+V^2}} \,.
\end{align}

We now proceed by taking derivatives of both equations implicitly with respect to $\eps$ and evaluate them at 
\begin{align*}
 \tau_g^* = \alpha^* = \frac{\sigma}{\sqrt{\delta -1}},\quad 
\tau_h^* = \beta^* = \sigma \sqrt{\delta -1}\,, \quad \gamma^* = 0, \quad\text{and}\quad \eps=0\,.
\end{align*} 
Note that the derivative of the first equation yields
\begin{align*}
\frac{\de}{\de\eps}\left( \sqrt{\alpha^2+\sigma^2}(\eta^2-1) +\eps\left(\sqrt{\tfrac{2}{\pi}}  \sqrt{\alpha^2\eta^2+V^2}(\eta^2-1) - (\alpha^2+\sigma^2) \tfrac{2\sqrt{\tfrac{2}{\pi}}  \eta ^2  }{ \sqrt{{\alpha ^2  \eta^2}+V^2}}\right) \right)=0\,.
\end{align*}
Thus
\begin{align*}
&\frac{\de}{\de\eps}\left( \sqrt{\alpha^2+\sigma^2}(\eta^2-1) \right) +\left(\sqrt{\tfrac{2}{\pi}}  \sqrt{\alpha^2\eta^2+V^2}(\eta^2-1) - (\alpha^2+\sigma^2) \tfrac{2\sqrt{\tfrac{2}{\pi}}  \eta ^2  }{ \sqrt{{\alpha ^2  \eta^2}+V^2}}\right)+\\
&\eps\frac{\de}{\de\eps}\left(\sqrt{\tfrac{2}{\pi}}  \sqrt{\alpha^2\eta^2+V^2}(\eta^2-1) - (\alpha^2+\sigma^2) \tfrac{2\sqrt{\tfrac{2}{\pi}}  \eta ^2  }{\sqrt{{\alpha ^2  \eta^2}+V^2}}\right)=0\,.
\end{align*}
Setting $\eps=0$ in the above yields
\begin{align*}
\frac{\de}{\de\eps} \left(\sqrt{\alpha^2+\sigma^2}(\eta^2-1)\right) +\left(\sqrt{\tfrac{2}{\pi}}  \sqrt{\alpha^2\eta^2+V^2} (\eta^2-1)- (\alpha^2+\sigma^2) \tfrac{2\sqrt{\tfrac{2}{\pi}}  \eta ^2  }{ \sqrt{{\alpha ^2  \eta^2}+V^2}}\right)=0\,.
\end{align*}
Thus
\begin{align*}
\frac{\alpha^*}{\sqrt{(\alpha^*)^2+\sigma^2}}\frac{\de \alpha}{\de\eps}(\eta_*^2-1)+2\eta_*\sqrt{\alpha_*^2+\sigma^2}\frac{\de \eta}{\de \eps}=(\alpha_*^2+\sigma^2) \tfrac{2\sqrt{\tfrac{2}{\pi}}  \eta_* ^2  }{ \sqrt{{\alpha_* ^2  \eta_*^2}+V^2}} -\sqrt{\tfrac{2}{\pi}}  \sqrt{\alpha_*^2 \eta_*^2+V^2}(\eta_*^2-1)\,.
\end{align*}
Setting $\eta_*=1$ this simplifies to
\begin{align*}
\frac{\de \eta}{\de \eps}=\sqrt{(\alpha^*)^2+\sigma^2} \tfrac{\sqrt{\tfrac{2}{\pi}}    }{ \sqrt{{(\alpha^*) ^2 }+V^2}} = \sigma \sqrt{\frac{2\delta}{\pi}} \frac{1}{\sqrt{\sigma^2+V^2(\delta -1)}}\,.
\end{align*}
In addition, from~\eqref{eq:dum0}
\[
\left(-\frac{\delta}{(\beta_*+\tau_{g*})^2}+ \frac{1}{\alpha_*^2+\sigma^2}\right) \frac{\de}{\de \eps}(\beta+\tau_g)
-\frac{\beta_*+\tau_{g*}}{(\alpha_*^2+\sigma^2)^2} 2\alpha_* \frac{\de \alpha}{\de \eps}
=- \frac{2}{\alpha_* \eta_*^2} \frac{\de \eta}{\de \eps}-\frac{2}{\alpha_*^2 \eta_*} \frac{\de \alpha}{\de \eps}\,.
\]
Plugging in for $\beta_*, \tau_{g*}, \alpha_*$ the coefficient of $\frac{\de}{\de \eps} (\beta+\tau_g)$ vanishes and we arrive at
\[
\frac{\frac{\sigma \delta}{\sqrt{\delta-1}}}{\left(\frac{\sigma^2\delta}{\delta-1}\right)^2} \frac{\sigma}{\sqrt{\delta-1}} \frac{\de \alpha}{\de \eps} = \frac{\sqrt{\delta-1}}{\sigma}\frac{\de \eta}{\de \eps}
+\frac{\delta-1}{\sigma^2} \frac{\de\alpha}{\de\eps}\,.
\]
Rearranging the terms, we obtain
\[
 \frac{\de \alpha}{\de \eps} = -\frac{\sigma\delta}{(\delta-1)^{3/2}} \frac{\de \eta}{\de \eps}= 
- \sqrt{\frac{2}{\pi}}\sigma^2 \left(\frac{\delta}{\delta-1}\right)^{3/2} \frac{1}{\sqrt{\sigma^2+V^2(\delta -1)}}\,.
\]
Now, invoking the definition of statistical risk we have
\begin{align}
{{\sf SR}}(\hth^\eps) &= {{\sf SR}}(\hth^0) + \frac{\de}{\de \eps} {{\sf SR}}(\hth^\eps)\Big|_{\eps=0} \eps + O(\eps^2)\nonumber\\
&=\sigma^2+\alpha_*^2 + 2\alpha_* \frac{\de \alpha}{\de \eps} \Big|_{\eps=0} + O(\eps^2)\nonumber\\
&=\frac{\sigma^2\delta}{\delta -1} - \sqrt{\frac{2}{\pi}}\frac{\sigma^3\delta^{3/2}}{(\delta-1)^2} \cdot\frac{1}{\sqrt{\sigma^2+V^2(\delta -1)}} +O(\eps^2)\,.
\end{align}
The proof is complete.
\end{proof}

\section{Proofs that the minimization and maximization primal problems can be restricted to a compact set}
\label{setres}

In this section we demonstrate how the minimization and maximization problems can be restricted to compacts sets. We start with the restriction on $\vct{z}$. To this aim recall that that one of the main goals of Theorem \ref{thm:main} is to characterize the distance of the optimal solution $\hth^{\eps}$ to $\vct{\theta}_0$ i.e.~$\frac{\twonorm{\hth^{\eps}-\vct{\theta}_0}}{\sqrt{p}}=\twonorm{ \hz^{\eps}}$ asymptotically and in particular to show $\twonorm{\vct{z}}\to \alpha_*$ as $n\to \infty$, in probability, for some $\alpha_*$ to be determined.  Now define the set $\cS_{\z} = \{\z|\;\; \twonorm{\vct{z}}\le K_\alpha\}$ with $K_\alpha = \alpha_*+\zeta$ for a constant $\zeta>0$ and consider the optimization problem
 \begin{align}
 \label{lin11}
\min_{\z\in\mathcal{S}_{\vct{z}},\vct{v}\in\R^n}  \max_{\vct{u}\in\R^n}\; \frac{1}{\sqrt{p}}\left(\vct{u}^T\bX \vct{z} - \vct{u}^T\vct{\omega} +\vct{u}^T \vct{v}\right)+\ell(\vct{v};\vct{z})
 \end{align}
 with $\vct{\omega}=\vct{w}/\sqrt{p}$. Based on the CGMT framework this optimization problem is equivalent to \eqref{lin} in an asymptotic fashion in the sense that if the Euclidean norm of the optimum solution to the above converges asymptotically to a value $\alpha_*$ in probability as $n\rightarrow +\infty$ then $\twonorm{\hz}$ also converges to the same value ($\|\hz\|\to\alpha_*$) in probability. See \cite[Theorem A.1]{thrampoulidis2018precise} for a formal argument.

The optimization problem above is still not in a form where CGMT can be applied as there are no compact restriction on $\vct{u}$. This is the subject of the next lemma.
\begin{lemma}
 The optimal solution $\vct{u}^*$ of \eqref{lin11} satisfies $\twonorm{\vct{u}^*}\le K_\beta$ for a sufficiently large constant $K_\beta>0$ with probability at least $1-2e^{-cn}$.
 \end{lemma}
 \begin{proof}
 Writing the KKT conditions for \eqref{lin11} we have
 \begin{align*}
 &\bX \vct{z} - \frac{1}{\sqrt{p}} \vct{w}+\vct{v} = 0\\
 &u_i = -\sqrt{p} [\nabla_{\vct{v}}\ell(\vct{v};\vct{z})]_i = -\frac{1}{\sqrt{p}} \left(v_i+\frac{\eps}{p}\cdot \sgn{v_i} \twonorm{\vct{\theta}_0+\sqrt{p}\vct{z}}\right)
 \end{align*}
 From the first equation we have that $\vct{v}=\frac{\vct{w}}{\sqrt{p}}-\bX\vct{z}$. Thus,
 \begin{align*}
 \twonorm{\vct{v}}\le& \frac{1}{\sqrt{p}}\twonorm{\vct{w}}+\twonorm{\bX\vct{z}}\\
 \le& \frac{1}{\sqrt{p}}\twonorm{\vct{w}}+\opnorm{\bX}\twonorm{\vct{z}}\\
 \overset{(a)}{\le}&C\sqrt{n}\sigma+C\left(\sqrt{p}+\sqrt{n}\right)\twonorm{\vct{z}}\\
  \overset{(b)}{\le}&C\sqrt{n}\sigma+C\left(\sqrt{p}+\sqrt{n}\right)K_\alpha
 \end{align*}
 holds with probability at least $1-2e^{-cn}$. Here, (a) follows from well known bounds on the Euclidean norm of a Gaussian vector and the spectral norm of a Gaussian matrix and (b) follows from the fact that $\twonorm{\vct{z}}\le K_\alpha$. We thus have $\twonorm{\vct{v}}\le C_2\left(\sqrt{p}+\sqrt{n}\right)$, with high probability. Now using the second equation we have
 \begin{align*}
 \twonorm{\vct{u}}\le& \frac{\twonorm{\vct{v}}}{\sqrt{p}} + \frac{\eps\sqrt{\delta}}{\sqrt{p}} \twonorm{\vct{\theta}_0+\sqrt{p}\vct{z}}\\
  \le& C\sigma\sqrt{\delta}+C(1+\sqrt{\delta})K_\alpha+ \frac{\eps\sqrt{\delta}}{\sqrt{p}} \twonorm{\vct{\theta}_0}+\eps\sqrt{\delta}\twonorm{\vct{z}}\\
  \le& C\sigma\sqrt{\delta}+C(1+\sqrt{\delta})K_\alpha+ \eps\sqrt{\delta}\widetilde{C}+\eps\sqrt{\delta}K_\alpha\\
  \le& K_\beta\,,
 \end{align*}
 for some bounded constant $K_\beta$. In the penultimum step we used the fact that $\frac{\twonorm{\vct{\theta}_0}}{\sqrt{p}}$ is bounded and $\twonorm{\vct{z}}\le K_\alpha$.
 \end{proof}
 
 \section{Proofs for scalarization of Auxilary Optimization (AO)}
 
 \subsection{Proof of Lemma \ref{conjlemma}}
 \label{conjlemmapf}
 We restate the lemma for the convenience of the reader.
 \begin{lemma}\label{conjlemma2}[Restatement of Lemma \ref{conjlemma}] The conjugate of 
\begin{align*}
\ell(\vct{v};\vct{z}):=\frac{1}{2p}\left(\twonorm{\vct{v}}^2+2\frac{\eps}{\sqrt{p}}\onenorm{\vct{v}}\twonorm{\vct{\theta}_0+\sqrt{p}\vct{z}}+\frac{\eps^2}{p}\twonorm{\vct{\theta}_0+\sqrt{p}\vct{z}}^2\right)
 \end{align*}
 with respect to the variable $\vct{z}$ is given by
\begin{align*}
\widetilde{\ell}(\vct{v};\vct{q}):=\sup_{\vct{z}} \vct{q}^T\vct{z}-\ell(\vct{v};\vct{z})=-\frac{1}{\sqrt{p}}\vct{q}^T\vct{\theta}_0+\frac{1}{2\delta p^2}\left(\frac{p}{\eps}\tn{\vct{q}}-\onenorm{\vct{v}}\right)_+^2 -\frac{1}{2p}\twonorm{\bv}^2.
\end{align*}
\end{lemma}
\begin{proof}
 We begin by calculating the conjugate of a slightly simpler function
 \begin{align*}
 \bar{\ell}(\vct{v};\vct{\theta}):=\frac{1}{2p}\sum_{i=1}^n \left(|v_i| + \frac{\eps}{\sqrt{p}}\tn{\vct{\theta}}\right)^2.
 \end{align*}
 We have
\begin{align*}
\bar{\ell}^*(\vct{v};\vct{q}) &= \sup_{\vct{\theta}} \vct{q}^T\vct{\theta} - \bar{\ell}(\vct{v};\vct{\theta})\\
&= \sup_{\vct{\theta}} \vct{q}^T\vct{\theta} - \frac{1}{2p}\sum_{i=1}^n \left(|v_i| + \frac{\eps}{\sqrt{p}}\tn{\vct{\theta}}\right)^2\\
&=\sup_{\vct{\theta}} \sup_{\xi\ge0}  \vct{q}^T\vct{\theta} - \frac{1}{2p} \left(\tn{\bv}^2 + \frac{2\eps}{\sqrt{p}}\onenorm{\bv} \left(\frac{\tn{\vct{\theta}}^2}{2\xi} + \frac{\xi}{2}\right)+
\delta\eps^2\tn{\vct{\theta}}^2\right)\\
&= \sup_{\xi\ge0}  \sup_{\vct{\theta}} \vct{q}^T\vct{\theta} - \frac{1}{2p} \left(\tn{\bv}^2 + \frac{2\eps}{\sqrt{p}}\onenorm{\bv} \left(\frac{\tn{\vct{\theta}}^2}{2\xi} + \frac{\xi}{2}\right)+
\delta\eps^2\tn{\vct{\theta}}^2\right)
\end{align*}
Setting derivative w.r.t $\vct{\theta}$ to zero, we get
\[
\vct{q} - \frac{\eps\onenorm{\bv}}{p^{3/2}\xi} \vct{\theta} - \frac{\delta\eps^2}{p}\vct{\theta} = 0 \quad\Rightarrow\quad
\vct{\theta} = \left(\frac{\eps\onenorm{\bv}}{p^{3/2}\xi}+  \frac{\delta\eps^2}{p} \right)^{-1}\vct{q}
\]
Setting the derivative with respect to $\xi$ to zero we conclude that $\xi=\tn{\vct{\theta}}$. Plugging the latter into above we conclude that
\[
\vct{\theta} = \left(\frac{\eps\onenorm{\bv}}{p^{3/2}\tn{\vct{\theta}}}+  \frac{\delta\eps^2}{p} \right)^{-1}\vct{q}
\]
Taking the Euclidean norm from both sides we conclude that
\begin{align*}
\tn{\vct{\theta}} \left(\frac{\eps\onenorm{\bv}}{p^{3/2}\tn{\vct{\theta}}}+  \frac{\delta\eps^2}{p} \right)=\tn{\vct{q}}\quad\Rightarrow\quad \tn{\vct{\theta}}=\frac{\tn{\vct{q}}-\frac{\eps\onenorm{\vct{v}}}{p^{\frac{3}{2}}}}{\frac{\delta\eps^2}{p}}=\frac{p}{\delta\eps^2}\tn{\vct{q}}-\frac{1}{\delta\eps\sqrt{p}}\onenorm{\vct{v}}
\end{align*}
If $\frac{p}{\delta\eps^2}\tn{\vct{q}}-\frac{1}{\delta\eps\sqrt{p}}\onenorm{\vct{v}}<0$ then it is easy to verify that the objective is less than
$-\frac{1}{2p}\twonorm{\bv}^2$ and hence the optimal is given by $\vct{\theta} = 0$.

Thus
\begin{align*}
\vct{\theta}=\left(\frac{p}{\delta\eps^2}\tn{\vct{q}}-\frac{1}{\delta\eps\sqrt{p}}\onenorm{\vct{v}}\right)\frac{\vct{q}}{\tn{\vct{q}}}=\left(\frac{p}{\delta\eps^2}-\frac{1}{\delta\eps\sqrt{p}}\frac{\onenorm{\vct{v}}}{\tn{\vct{q}}}\right)\vct{q}
\end{align*}
Substituting for $\vct{\theta}$ we have
\begin{align*}
\bar{\ell}^*(\bv;\vct{q}) =& \left(\frac{p}{\delta\eps^2}-\frac{1}{\delta\eps\sqrt{p}}\frac{\onenorm{\vct{v}}}{\tn{\vct{q}}}\right) \tn{\vct{q}}^2\\
&- \frac{1}{2p} \left(\tn{\bv}^2 + \frac{2\eps}{\sqrt{p}}\onenorm{\bv} \left(\frac{p}{\delta\eps^2}\tn{\vct{q}}-\frac{1}{\delta\eps\sqrt{p}}\onenorm{\vct{v}}\right)+
\delta\eps^2\left(\frac{p}{\delta\eps^2}\tn{\vct{q}}-\frac{1}{\delta\eps\sqrt{p}}\onenorm{\vct{v}}\right)^2\right)\nn\\
=&\frac{p}{2\delta\eps^2}\tn{\vct{q}}^2-\frac{1}{\delta\eps\sqrt{p}}\onenorm{\vct{v}}\tn{\vct{q}}+\frac{1}{2\delta p^2}\onenorm{\vct{v}}^2-\frac{1}{2p}\twonorm{\vct{v}}^2\\
=&\frac{1}{2\delta p^2}\left(\frac{p^{\frac{3}{2}}}{\eps}\tn{\vct{q}}-\onenorm{\vct{v}}\right)^2-\frac{1}{2p}\twonorm{\vct{v}}^2
\end{align*}
if $\onenorm{\vct{v}}\le \frac{p^{\frac{3}{2}}}{\eps}\tn{\vct{q}}$. Otherwise,
\begin{align*}
\bar{\ell}^*(\bv;\vct{q}) = -\frac{1}{2p}\twonorm{\bv}^2\,.
\end{align*}
We can put the two cases together using the notation $z_+ = \max(z,0)$.
\begin{align*}
\bar{\ell}^*(\bv;\vct{q}) = \frac{1}{2\delta p^2}\left(\frac{p^{\frac{3}{2}}}{\eps}\tn{\vct{q}}-\onenorm{\vct{v}}\right)_+^2 -\frac{1}{2p}\twonorm{\bv}^2\,.
\end{align*}
Now to calculate the conjugate of $\ell(\vct{v};\vct{z})$ note that
\begin{align*}
\ell(\vct{v};\vct{z})=\bar{\ell}\left(\vct{v};\vct{\theta}_0+\sqrt{p}\vct{z}\right)
\end{align*}
To continue note that if we have $f(\vct{x})=g(\mtx{A}\vct{x}+\vct{x}_0)$ the conjugate is given by
\begin{align*}
f^*(\vct{y})=-\langle \mtx{A}^{-1}\vct{x}_0,\vct{y}\rangle+g^*\left(\mtx{A}^{-T}\vct{y}\right)
\end{align*}
Thus using above with $\vct{x}_0=\vct{\theta}_0$ and $\mtx{A}=\sqrt{p}$ we arrive at
\begin{align*}
\widetilde{\ell}(\vct{v};\vct{q})=&-\frac{1}{\sqrt{p}}\langle\vct{\theta}_0,\vct{q}\rangle+\bar{\ell}^*\left(\bv;\frac{1}{\sqrt{p}}\vct{q}\right) \nonumber\\
=&-\frac{1}{\sqrt{p}}\vct{q}^T\vct{\theta}_0+\frac{1}{2\delta p^2}\left(\frac{p}{\eps}\tn{\vct{q}}-\onenorm{\vct{v}}\right)_+^2 -\frac{1}{2p}\twonorm{\bv}^2,
\end{align*}
concluding the proof.
\end{proof}

\subsection{Proof of Lemma \ref{cvxconcavelem}}
\label{cvxconcavelempf}
\begin{lemma}[Restatement of Lemma \ref{cvxconcavelem}]
\label{cvxconcavelem2}
The function
\begin{align*}
f(\gamma,\beta,\tau_h):=\gamma^2+\frac{\beta^2}{p}\tn{\vct{h}}^2- 2\frac{\gamma}{\sqrt{p}} \twonorm{\beta\vct{h}-\frac{\vct{\theta_0}}{\alpha}}
\end{align*}
is jointly convex in the parameters $(\gamma,\beta,\tau_h)$.
\end{lemma}
\begin{proof}
\begin{align*}
\gamma^2+\frac{\beta^2}{p}\tn{\vct{h}}^2- 2\frac{\gamma}{\sqrt{p}} \twonorm{\beta\vct{h}-\frac{\vct{\theta_0}}{\alpha}}=\gamma^2+\frac{\beta^2}{p}\tn{\vct{h}}^2- 2\frac{\gamma}{\sqrt{p}} \sqrt{\beta^2\twonorm{\vct{h}}^2+\frac{1}{\alpha^2}\tn{\vct{\theta}_0}^2-\frac{2}{\alpha}\beta\vct{h}^T\vct{\theta}_0}
\end{align*}
with the Hessian with respect to $(\gamma,\beta)$ equal to
\begin{align*}
\begin{bmatrix}
2 & -\frac{1}{\sqrt{p}}\frac{2\beta\twonorm{\vct{h}}^2-\frac{2}{\alpha}\vct{h}^T\vct{\theta}_0}{\sqrt{\beta^2\twonorm{\vct{h}}^2+\frac{1}{\alpha^2}\tn{\vct{\theta}_0}^2-\frac{2}{\alpha}\beta\vct{h}^T\vct{\theta}_0}}\\
-\frac{1}{\sqrt{p}}\frac{2\beta\twonorm{\vct{h}}^2-\frac{2}{\alpha}\vct{h}^T\vct{\theta}_0}{\sqrt{\beta^2\twonorm{\vct{h}}^2+\frac{1}{\alpha^2}\tn{\vct{\theta}_0}^2-\frac{2}{\alpha}\beta\vct{h}^T\vct{\theta}_0}}& 2\frac{\tn{\vct{h}}^2}{p}
\end{bmatrix}
\end{align*}
The determinant is equal to
\begin{align*}
&\frac{4}{p}\left(\tn{\vct{h}}^2-\frac{\left(\beta\twonorm{\vct{h}}^2-\frac{\vct{h}^T\vct{\theta}_0}{\alpha}\right)^2}{\beta^2\twonorm{\vct{h}}^2+\frac{1}{\alpha^2}\tn{\vct{\theta}_0}^2-\frac{2}{\alpha}\beta\vct{h}^T\vct{\theta}_0}\right)\\
&\quad\quad\quad\quad\quad\quad\quad=\frac{4}{p\alpha^2}\frac{1}{\beta^2\twonorm{\vct{h}}^2+\frac{1}{\alpha^2}\tn{\vct{\theta}_0}^2-\frac{2}{\alpha}\beta\vct{h}^T\vct{\theta}_0}\left(\tn{\vct{h}}^2\tn{\vct{\theta}_0}^2-\left(\vct{h}^T\vct{\theta}_0\right)^2\right)\\
&\quad\quad\quad\quad\quad\quad\quad\ge 0
\end{align*}
Thus
\begin{align*}
\frac{\alpha}{2}\gamma^2+\frac{\alpha}{2}\frac{\beta^2}{p}\tn{\vct{h}}^2- \frac{\gamma}{\sqrt{p}} \twonorm{\alpha\beta\vct{h}-\vct{\theta}_0}=\frac{\alpha}{2}\left(\gamma^2+\frac{\beta^2}{p}\tn{\vct{h}}^2- 2\frac{\gamma}{\sqrt{p}} \twonorm{\beta\vct{h}-\frac{\vct{\theta_0}}{\alpha}}\right)
\end{align*}
is jointly convex in $(\gamma,\beta)$. Therefore the perspective function
\begin{align*}
\tau_h\left(\frac{\alpha}{2}\left(\frac{\gamma}{\tau_h}\right)^2+\frac{\alpha}{2}\frac{\beta^2}{p\tau_h^2}\tn{\vct{h}}^2- \frac{\gamma}{\tau_h\sqrt{p}} \twonorm{\alpha\frac{\beta}{\tau_h}\vct{h}-\vct{\theta}_0}\right)=\frac{\alpha}{2\tau_h}\gamma^2+\frac{\alpha\beta^2}{2p\tau_h}\tn{\vct{h}}^2- \frac{\gamma}{\sqrt{p}} \twonorm{\frac{\alpha\beta}{\tau_h}\vct{h}-\vct{\theta_0}}
\end{align*}
is jointly convex in $(\gamma,\beta,\tau_h)$.
\end{proof}

\subsection{Proof of Lemma \ref{meenv}}
\label{meenvpf}
We begin by stating and proving the following lemma.
\begin{lemma}\label{lem:prox}
The value of the following problem (with $\lambda>1$)
\begin{align*}
\min_{\vct{v}\in\R^n}\quad \frac{\lambda}{2}\twonorm{\vct{x}-\vct{v}}^2-\frac{1}{2n}\left(\gamma-\onenorm{\vct{v}}\right)_{+}^2
\end{align*}
is given by
\begin{align*}
\min_{\tau \ge0}\quad \frac{\lambda}{2}\twonorm{\vct{x}-\ST(\bx;\tau)}^2-\frac{1}{2n}\left(\gamma-\onenorm{\ST(\bx;\tau)}\right)_{+}^2
\end{align*}
where $\ST(\bx;\tau)$ is the soft-thresholding function.
\end{lemma}
Notably the lemma above transforms the first optimization (on vector $\bv$) to an optimization over scalar $\tau$.

\begin{proof}
We consider two case:

\noindent\textbf{Case I: $\onenorm{\vct{x}}> \gamma$}\\
In this case the optimal value is achieved by $\bv = \bx$ resulting in an objective value of zero. We shall proceed by contradiction and assume $\bv=\bx$ is not an optimal solution. First note that under this contradictory assumption we must have $\onenorm{\bv} <\gamma$ as otherwise the $(\cdot)_{+}$ term would be inactive and the objective value would be greater than or equal to zero in which case $\bv = \bx$ would achieve the optimum value negating the contradictory assumption. We thus focus on the case that $\onenorm{\bv} <\gamma$. To reach a contradiction in this case note that we have
\begin{align*}
&\frac{\lambda}{2}\twonorm{\vct{x}-\vct{v}}^2-\frac{1}{2n}\left(\gamma-\onenorm{\vct{v}}\right)_{+}^2\\
&\ge \frac{\lambda}{2}\twonorm{\vct{x}-\vct{v}}^2-\frac{1}{2n}\left(\onenorm{\bx}-\onenorm{\vct{v}}\right)^2\\
&\ge \frac{\lambda}{2}\twonorm{\vct{x}-\vct{v}}^2-\frac{1}{2n} \onenorm{\bx - \bv}^2\\
&> \frac{1}{2}\twonorm{\vct{x}-\vct{v}}^2-\frac{1}{2n} \onenorm{\bx - \bv}^2\\
 &\ge 0,
\end{align*}
where in the penultimum
Since $\ST(\bx;0)=\bx$ and we showed that it is the optimal $\bv$, the claim holds in this case. Namely, the minimizer is achieved at a point in $\{\ST(\bx;\tau): \; \tau\ge 0\}$.

\noindent\textbf{Case II: $\onenorm{\vct{x}}\le \gamma$}\\
Since $\onenorm{\bv}$ is invariant with respect to the sign of its entries, it is clear that at the solution $\bv$, we must have $\sign(\bv) = \sign(\bx)$. Moreover, without loss of generality we can assume $\onenorm{\bv} \le \gamma$ as otherwise similar to the previous case $\bv=\bx$ would be a solution and the minimizer is achieved at a point in $\{\ST(\bx;\tau): \; \tau\ge 0\}$. At the optimal solution we must have\footnote{We note that since $\lambda>1$ the objective $\frac{\lambda}{2}\twonorm{\vct{x}-\vct{v}}^2-\frac{1}{2n}\left(\gamma-\onenorm{\vct{v}}\right)_{+}^2$ is convex and thus optimality is given by zero being a sub-gradient.}
\[
\vct{0}\in \lambda(\bv-\bx) -\frac{1}{n} (\onenorm{\bv}-\gamma) \partial \onenorm{\bv}
\]
As we argued previously at an optimal solution we must have $\sign(\bv) = \sign(\bx)$ and thus $\partial \onenorm{\bv}=\partial \onenorm{\bx}$ rearranging the terms gives
\[
\bv \in \bx + \frac{1}{\lambda n} (\onenorm{\bv}-\gamma) \partial \onenorm{\bv}=\bx - \frac{1}{\lambda n} (\gamma-\onenorm{\bv}) \partial \onenorm{\bx}
\]
Thus $\bv = \ST(\bx;\tau)$ for $\tau = \frac{1}{\lambda n} (\gamma- \onenorm{\bv})\ge 0$ and the claim follows.
\end{proof}
With the lemma above in place we turn our attention to completing the proof of Lemma \ref{meenv}. To this aim note that since $f(\bv)$ is convex and $\twonorm{\bx-\bv}^2$ is strictly convex, then $\frac{1}{2\mu} \twonorm{\bx-\bv}^2+f(\bv)$ is jointly striclty convex in $(\bx,\bv)$. Since partial minimization preserves convexity, $e_f(\bx;\mu)$ is strictly convex in $\bx$ (also see \cite[Lemma C.5]{thrampoulidis2015precise}).

We write the Moreau envelope as
\begin{align*}
e_f(\bx;\mu) &= \min_{\bv} \frac{1}{2\mu} \twonorm{\vct{x}-\bv}^2+ \frac{1}{2} \twonorm{\bv}^2 -\frac{1}{2\delta p} \left(\frac{p}{\eps}\gamma-\onenorm{\bv}\right)_+^2\\
&=\min_{\bv} \frac{1}{2} \left(\frac{1}{\mu}+1\right) \twonorm{\bv - \frac{\vct{x}}{\mu+1}}^2+ \frac{1}{2(\mu+1)} \twonorm{\bx}^2 -\frac{1}{2\delta p} \left(\frac{p}{\eps}\gamma-\onenorm{\bv}\right)_+^2
\end{align*}
Using Lemma~\ref{lem:prox} with $\lambda = \frac{1+\mu}{\mu}>1$, we arrive at
\begin{align}
e_f(\bx;\mu) &= \frac{1}{2(\mu+1)}\twonorm{\bx}^2 +
\min_{\tau\ge 0} \frac{1}{2\mu(\mu+1)}\twonorm{\vct{x}-\ST(\bx;\tau(\mu+1))}^2\nonumber\\
&\quad\quad\quad\quad\quad\quad\quad\quad\quad\quad\quad\quad\quad\quad-\frac{1}{2n}\left(\frac{p}{\eps}\gamma-\frac{1}{1+\mu}\onenorm{\ST(\bx;\tau(\mu+1))}\right)_+^2.
\end{align}
The result follows by a change of variable $\tau(\mu+1) \to \tau$.

\subsection{Proof of Lemma \ref{Glem}}
\label{Glempf}
We begin by restating the lemma for the convenience of the reader.
\begin{lemma}[Restatement of Lemma \ref{Glem}]\label{Glem2} Let $\vct{w}\in\R^n$ be a Gaussian random vector distributed as $\mathcal{N}(\vct{0},\omega^2\mtx{I}_n)$. Also assume
\begin{align*}
G(\vct{w};\mu,\tau) := \frac{1}{2\mu(\mu+1)}\twonorm{\vct{w}-\ST(\vct{w};\tau)}^2-\frac{1}{2n}\left(\frac{p}{\eps}\gamma-\frac{1}{1+\mu}\onenorm{\ST(\vct{w};\tau)}\right)_+^2.
\end{align*}
Then
\begin{align*}
\underset{n \rightarrow \infty}{\lim}\text{ }\frac{1}{n}G(\vct{w};\mu,\tau)=&\frac{\omega^2}{2\mu(\mu+1)}\left(\left(1-\sqrt{\frac{2}{\pi}}\frac{\tau}{\omega} e^{-\frac{\tau^2}{2\omega^2}}\right)+\left(\frac{\tau^2}{\omega^2}-1\right)\text{erfc}\left(\frac{1}{\sqrt{2}}\frac{\tau}{\omega}\right)\right)\\
&-\frac{\omega^2}{2(\mu+1)^2}\left(\frac{\gamma(\mu+1)}{\delta\eps\omega}+\frac{\tau}{\omega}\cdot\text{erfc}\left(\frac{1}{\sqrt{2}}\frac{\tau}{\omega}\right)-\sqrt{\frac{2}{\pi}} e^{-\frac{\tau^2}{2\omega^2}}\right)_+^2.
\end{align*}
Furthermore, 
\begin{align*}
\underset{\tau\ge 0}{\min}\text{ }\underset{n \rightarrow \infty}{\lim}\text{ }&\frac{1}{n}G(\vct{w};\mu,\tau)\\
&\quad=
\begin{cases}
0 \quad &\text{ if }\gamma(\mu+1)\le \sqrt{\frac{2}{\pi}}\delta \eps\omega\\
\frac{\omega^2}{2\mu(\mu+1)}\left(1-\erfc\left(\frac{\tau^*\left(\frac{\gamma(\mu+1)}{\delta\eps\omega},\mu\right)}{\sqrt{2}}\right)-\frac{\gamma(\mu+1)}{\delta\eps\omega}\tau^*\left(\frac{\gamma(\mu+1)}{\delta\eps\omega},\mu\right)\right)&\text{ if }\gamma(\mu+1)>\sqrt{\frac{2}{\pi}}\delta \eps \omega
\end{cases}
\end{align*}
where $\tau^*(a,\mu)$ is the unique solution to 
\begin{align}\label{eq:tau-star}
a-\frac{\mu+1}{\mu}\tau+\tau\cdot\erfc\left(\frac{\tau}{\sqrt{2}}\right)-\sqrt{\frac{2}{\pi}} e^{-\frac{\tau^2}{2}}=0
\end{align}
Alternatively using the fact that $\erf=1-\erfc$ we can rewrite this in the form
\begin{align*}
\underset{\tau\ge 0}{\min}\text{ }\underset{n \rightarrow \infty}{\lim}\text{ }&\frac{1}{n}G(\vct{w};\mu,\tau)\\
&\quad\quad=
\begin{cases}
0 \quad &\text{ if }\gamma(\mu+1)\le \sqrt{\frac{2}{\pi}}\delta \eps\omega\\
\frac{\omega^2}{2\mu(\mu+1)}\left(\erf \left(\frac{\tau^*\left(\frac{\gamma(\mu+1)}{\delta\eps\omega},\mu\right)}{\sqrt{2}}\right)-\frac{\gamma(\mu+1)}{\delta\eps\omega}\tau^*\left(\frac{\gamma(\mu+1)}{\delta\eps\omega},\mu\right)\right)&\text{ if }\gamma(\mu+1)>\sqrt{\frac{2}{\pi}}\delta \eps \omega
\end{cases}
\end{align*}
where $\tau^*(a,\mu)$ is the unique solution to 
\begin{align*}
a-\frac{1}{\mu}\tau-\tau\cdot\erf\left(\frac{\tau}{\sqrt{2}}\right)-\sqrt{\frac{2}{\pi}} e^{-\frac{\tau^2}{2}}=0
\end{align*}
\end{lemma}
\begin{proof}
First note that by law-of large numbers we have
\begin{align*}
\underset{n \rightarrow \infty}{\lim}\text{ }\frac{1}{n}\twonorm{\vct{w}-\ST(\vct{w};\tau)}^2=&\E_{g\sim\mathcal{N}(0,1)}\Big[\left(\omega g-\ST(\omega g;\tau)\right)^2\Big]\\
=&\omega^2\E_{g\sim\mathcal{N}(0,1)}\Big[\left( g-\ST\left(g;\frac{\tau}{\omega}\right)\right)^2\Big]\\
=&\omega^2\left(\frac{2}{\sqrt{2\pi}}\int_{+\frac{\tau}{\omega}}^{+\infty} \frac{\tau^2}{\omega^2} e^{-\frac{x^2}{2}}dx+\frac{1}{\sqrt{2\pi}}\int_{-\frac{\tau}{\omega}}^{+\frac{\tau}{\omega}} x^2 e^{-\frac{x^2}{2}}dx\right)\\
=&\omega^2\left( \left(1-\sqrt{\frac{2}{\pi}}\frac{\tau}{\omega}e^{-\frac{\tau^2}{2\omega^2}}\right)+\left(\frac{\tau^2}{\omega^2}-1\right)\text{erfc}\left(\frac{1}{\sqrt{2}}\frac{\tau}{\omega}\right)\right)\\
=&\omega\left(\omega-\sqrt{\frac{2}{\pi}}\tau e^{-\frac{\tau^2}{2\omega^2}}\right)+\left(\tau^2-\omega^2\right)\text{erfc}\left(\frac{1}{\sqrt{2}}\frac{\tau}{\omega}\right)
\end{align*}
Next note that
\begin{align*}
\underset{n \rightarrow \infty}{\lim}\text{ }\frac{1}{n}\onenorm{\ST(\vct{w};\tau)}=&\E_{g\sim\mathcal{N}(0,1)}\Big[\abs{\ST(\omega g;\tau)}\Big]\\
=&\omega\E_{g\sim\mathcal{N}(0,1)}\Bigg[\abs{\ST\left( g;\frac{\tau}{\omega}\right)}\Bigg]\\
=&\frac{\omega}{\sqrt{2\pi}}\left(\int_{+\frac{\tau}{\omega}}^{+\infty}\left(x-\frac{\tau}{\omega}\right)e^{-\frac{x^2}{2}}dx-\int_{-\infty}^{-\frac{\tau}{\omega}}\left(x+\frac{\tau}{\omega}\right)e^{-\frac{x^2}{2}}dx \right)\\
=&\sqrt{\frac{2}{\pi}}\omega\left(\int_{+\frac{\tau}{\omega}}^{+\infty}\left(x-\frac{\tau}{\omega}\right)e^{-\frac{x^2}{2}}dx\right)\\
=&\sqrt{\frac{2}{\pi}}\omega\cdot e^{-\frac{\tau^2}{2\omega^2}}-\tau\cdot\text{erfc}\left(\frac{1}{\sqrt{2}}\frac{\tau}{\omega}\right)
\end{align*}
Therefore,
\begin{align*}
\underset{n\rightarrow \infty}{\lim}\quad\frac{1}{2n^2}\left(\frac{p}{\eps}\gamma-\frac{1}{1+\mu}\onenorm{\ST(\vct{w};\tau)}\right)_+^2=&\underset{n\rightarrow \infty}{\lim}\quad\frac{1}{2}\left(\frac{\gamma}{\delta\eps}-\frac{1}{1+\mu}\frac{\onenorm{\ST(\vct{w};\tau)}}{n}\right)_+^2\\
=&\frac{1}{2}\left(\frac{\gamma}{\delta\eps}+\frac{\tau}{1+\mu}\text{erfc}\left(\frac{1}{\sqrt{2}}\frac{\tau}{\omega}\right)-\sqrt{\frac{2}{\pi}}\frac{\omega}{1+\mu} e^{-\frac{\tau^2}{2\omega^2}}\right)_+^2
\end{align*}
The proof of the first identity follows by combining the two summands.

To prove the second identity note that using a change of variable $\tau/\omega\rightarrow \tau$
\begin{align*}
&\underset{\tau\ge 0}{\min}\quad\underset{n \rightarrow \infty}{\lim}\text{ }\frac{1}{n}G(\vct{w};\mu,\tau\omega)\\
&\quad\quad\quad\quad\quad\quad\quad\quad=\frac{\omega^2}{2(\mu+1)^2}\cdot\underset{\tau\ge 0}{\min}\quad \frac{\mu+1}{\mu}\left(\left(1-\sqrt{\frac{2}{\pi}}\tau e^{-\frac{\tau^2}{2}}\right)+\left(\tau^2-1\right)\text{erfc}\left(\frac{\tau}{\sqrt{2}}\right)\right)\\
&\quad\quad\quad\quad\quad\quad\quad\quad\quad-\left(\frac{\gamma(\mu+1)}{\delta\eps\omega}+\tau\cdot\text{erfc}\left(\frac{\tau}{\sqrt{2}}\right)-\sqrt{\frac{2}{\pi}} e^{-\frac{\tau^2}{2}}\right)_+^2
\end{align*}
To continue note that if only the first term is active the derivative is given by
\begin{align*}
2\frac{\mu+1}{\mu}\tau\text{erfc}\left(\frac{\tau}{\sqrt{2}}\right)\ge 0
\end{align*}
and when both terms are active the derivative is given by
\begin{align*}
&2\tau\frac{\mu+1}{\mu}\text{erfc}\left(\frac{\tau}{\sqrt{2}}\right)-2\text{erfc}\left(\frac{\tau}{\sqrt{2}}\right)\left(\frac{\gamma(\mu+1)}{\delta\eps\omega}+\tau\cdot\text{erfc}\left(\frac{\tau}{\sqrt{2}}\right)-\sqrt{\frac{2}{\pi}} e^{-\frac{\tau^2}{2}}\right)\\
&\quad\quad=-2\text{erfc}\left(\frac{\tau}{\sqrt{2}}\right)\left((\mu+1)\left(\frac{\gamma}{\delta\eps\omega}-\frac{\tau}{\mu}\right)+\tau\cdot\text{erfc}\left(\frac{\tau}{\sqrt{2}}\right)-\sqrt{\frac{2}{\pi}} e^{-\frac{\tau^2}{2}}\right)
\end{align*}
We note that the function $(\mu+1)\left(\frac{\gamma}{\delta\eps\omega}-\frac{\tau}{\mu}\right)+\tau\cdot\text{erfc}\left(\frac{\tau}{\sqrt{2}}\right)-\sqrt{\frac{2}{\pi}} e^{-\frac{\tau^2}{2}}$ is always decreasing when $\tau\ge0 $ and its value at $\tau=0$ is given by $\frac{\gamma(\mu+1)}{\delta \eps\omega}-\sqrt{\frac{2}{\pi}}$. To continue further consider two cases.

\noindent\underline{\textbf{ Case I: $\gamma(\mu+1)\le \sqrt{\frac{2}{\pi}} \delta\eps\omega$:}}\\
In this case the function is always increasing in $\tau\in[0,+\infty)$ and thus the minimum is achieved at $\tau=0$ with the corresponding optimal value given by
\begin{align*}
-\frac{\omega^2}{2(\mu+1)^2}\left(\frac{\gamma(\mu+1)}{\delta\eps\omega}-\sqrt{\frac{2}{\pi}}\right)_+^2=0
\end{align*}

\noindent\underline{\textbf{ Case II: $\gamma(\mu+1)> \sqrt{\frac{2}{\pi}}\delta\eps\omega$:}}\\
In this case the function is decreasing at the beginning and then increases. Therefore, the minimum is achieved at a point where
\begin{align*}
(\mu+1)\left(\frac{\gamma}{\delta\eps\omega}-\frac{\tau}{\mu}\right)+\tau\cdot\text{erfc}\left(\frac{\tau}{\sqrt{2}}\right)-\sqrt{\frac{2}{\pi}} e^{-\frac{\tau^2}{2}}=0
\end{align*}
Note that at such a point we have
\begin{align*}
\frac{\gamma(\mu+1)}{\delta\eps\omega}+\tau\cdot\text{erfc}\left(\frac{\tau}{\sqrt{2}}\right)-\sqrt{\frac{2}{\pi}} e^{-\frac{\tau^2}{2}}=\frac{\mu+1}{\mu}\tau
\end{align*}
and
\begin{align*}
\left(1-\sqrt{\frac{2}{\pi}}\tau e^{-\frac{\tau^2}{2}}\right)+\left(\tau^2-1\right)\text{erfc}\left(\frac{\tau}{\sqrt{2}}\right)=&\tau^2\cdot\text{erfc}\left(\frac{\tau}{\sqrt{2}}\right)-\tau\sqrt{\frac{2}{\pi}} e^{-\frac{\tau^2}{2}}+1-\text{erfc}\left(\frac{\tau}{\sqrt{2}}\right)\\
=&\frac{\mu+1}{\mu}\tau^2-\frac{\gamma(\mu+1)}{\delta\eps\omega}\tau+1-\text{erfc}\left(\frac{\tau}{\sqrt{2}}\right)
\end{align*}
Thus
\begin{align*}
&\frac{\mu+1}{\mu}\left(\left(1-\sqrt{\frac{2}{\pi}}\tau e^{-\frac{\tau^2}{2}}\right)+\left(\tau^2-1\right)\text{erfc}\left(\frac{\tau}{\sqrt{2}}\right)\right)-\left(\frac{\gamma(\mu+1)}{\delta\eps\omega}+\tau\cdot\text{erfc}\left(\frac{\tau}{\sqrt{2}}\right)-\sqrt{\frac{2}{\pi}} e^{-\frac{\tau^2}{2}}\right)_+^2\\
&\quad\quad\quad=\frac{(\mu+1)^2}{\mu^2}\tau^2-\frac{\gamma(\mu+1)^2}{\delta\eps\omega\mu}\tau+\frac{\mu+1}{\mu}\left(1-\text{erfc}\left(\frac{\tau}{\sqrt{2}}\right)\right)-\frac{(\mu+1)^2}{\mu^2}\tau^2\\
&\quad\quad\quad=\frac{\mu+1}{\mu}\left(1-\text{erfc}\left(\frac{\tau}{\sqrt{2}}\right)-\frac{\gamma(\mu+1)}{\delta\eps\omega}\tau\right)
\end{align*}
%
\end{proof}

\end{document}